\documentclass[11pt,a4paper]{article}
\usepackage[a4paper,left=2.5cm,right=2.5cm,top=2.5cm,bottom=2.5cm]{geometry} 
 
\usepackage{amsmath,amssymb,amsthm,amsfonts,subcaption,caption}
\usepackage[colorlinks=True,
            linkcolor=blue,
            anchorcolor=green,
            citecolor=blue
            ]{hyperref} 
\usepackage[noabbrev,capitalise,nameinlink]{cleveref}
\usepackage{float}
\usepackage{graphicx}
\usepackage{epstopdf}
\usepackage{multirow} 
\usepackage{enumitem,url}
\usepackage{threeparttable}  

\usepackage[ruled,vlined]{algorithm2e}
\usepackage{etoolbox}

 \newtheorem{theorem}{Theorem}[section]

\newtheorem{lemma}{Lemma}[section]

\newtheorem{remark}{Remark}[section]
\newtheorem{assumption}{Assumption}[section]

\usepackage{hyperref}
\newcommand{\footremember}[2]{%
    \footnote{#2}
    \newcounter{#1}
    \setcounter{#1}{\value{footnote}}%
}

\graphicspath{ {./Figures/} }

\def\B{{\mathcal B}}

\def\N{{\mathbb N}}
\def\D{{\mathcal D}}
\def\L{{\mathcal L}}
\def\E{{\mathbb E}}
\def\R{{\mathbb R}}

\def\I{{\bf  I}}
\def\W{{\bf  W}}
\def\Q{{\bf  Q}}
\def\M{{\bf  M}}

\def\bu{{\bf u}}
\def\bv{{\bf v}} 
\def\bw{{\bf w}}
\def\bz{{\bf z}}
\def\bx{{\bf x}}

\def\bm{{\bf m}}
\def\bg{{\bf g}}

\def\bpi{{\boldsymbol  \pi}} 
\def\P{{\boldsymbol  \Pi}}
\def\bsi{{\boldsymbol  \sigma}}

\def\eqspace{\arraycolsep=1.5pt\def\arraystretch}
\def\dsum{\displaystyle \sum_{i=1}^m}

\author{
  Shenglong Zhou\footremember{bjtu1}{School of Mathematics and Statistics, Beijing Jiaotong University, China (shlzhou@bjtu.edu.cn).}~
  Ouya Wang\footremember{SD}{ITP Lab, Department of EEE, Imperial College London, UK (ouya.wang20@imperial.ac.uk).}~ 
  Ziyan Luo\footremember{bjtu2}{School of Mathematics and Statistics, Beijing Jiaotong University, China (luozy@bjtu.edu.cn).}~
  Yongxu Zhu\footremember{seu}{National Mobile Communications Research Laboratory, Southeast University, China (yongxu.zhu@seu.edu.cn).}~
  Geoffrey Ye Li\footremember{icl1}{ITP Lab, Department of EEE, Imperial College London, UK (geoffrey.li@imperial.ac.uk).}
}

\title{\vspace{-1.25cm}
Preconditioned Inexact Stochastic ADMM for Deep Models\thanks{This work was supported by the National Key R\&D Program of China (2023YFA1011100), the Fundamental Research Funds for the Central Universities, the National Natural Science Foundation of China (No.12271022), and the Talent Fund of Beijing Jiaotong University (\textit{Corresponding authors: Ouya Wang and Shenglong Zhou}).}
\vspace{-0.25cm}}

\date{}

\begin{document}
\flushbottom
 
\maketitle
 
\vspace{-1.0cm}

\begin{abstract}  
\noindent \textbf{Abstract:} Deep learning models are usually trained with stochastic gradient descent-based algorithms,
but these optimizers face inherent limitations, such as slow convergence and stringent
assumptions for convergence.  In particular, data heterogeneity arising from distributed settings poses significant challenges to their theoretical and numerical performance. This paper develops an algorithm, PISA (\textbf{P}reconditioned \textbf{I}nexact \textbf{S}tochastic \textbf{A}lternating Direction Method of Multipliers). Grounded in rigorous theoretical guarantees, the algorithm converges under the sole assumption of Lipschitz continuity of the gradient on a bounded region, thereby removing the need for other conditions commonly imposed by stochastic methods. This capability enables the proposed algorithm to tackle the challenge of data heterogeneity effectively. Moreover, the algorithmic architecture enables scalable parallel computing and supports various preconditions, such as second-order information, second moment, and orthogonalized momentum by Newton-Schulz iterations.  Incorporating the latter two preconditions in PISA yields two computationally efficient variants: SISA and NSISA.
 Comprehensive experimental evaluations for training or fine-tuning diverse deep models, including vision models, large language models, reinforcement learning models, generative adversarial networks, and recurrent neural networks, demonstrate  superior numerical performance of SISA and NSISA compared to various state-of-the-art optimizers. 

\vspace{0.3cm} 
 
\end{abstract}

\numberwithin{equation}{section}


\section{Introduction}

Deep models (DMs) have led to extensive applications across a variety of industries. For instance, fine-tuning large language models (LLMs), such as Gemma \cite{touvron2023llama2openfoundation}, GPT \cite{openai2024gpt4technicalreport}, and Llama \cite{gemmateam2024gemmaopenmodelsbased}, with user-specific data can significantly enhance the model performance. Diffusion models (DMs) are increasingly used to generate personalized content, including images, audio, and video, for entertainment. Vision models (VMs), like ResNet \cite{he2015deepresiduallearningimage} and vision transformers (ViT) \cite{dosovitskiy2020image}, are in high demand for tasks, such as image classification, segmentation, and object detection on large-scale image datasets. However, the growing complexity and scale of these tasks present considerable challenges for widely used stochastic gradient descent (SGD)-based optimizers due to their inherent limitations. In this section, we review a subset of these methods, sufficient to motivate the development of our proposed approach.

\subsection{SGD-based learning algorithms}

It is known that the plain SGD-based algorithms, such as the original stochastic approximation approaches \cite{robbins1951stochastic,chung1954stochastic}, the parallelized SGD \cite{zinkevich2010parallelized} for machine learning, and the newly developed federated averaging (FedAvg \cite{mcmahan2017communication, Li2020On}) and local SGD (LocalSGD \cite{stich2018local}) algorithms for federated learning, are frequently prone to high sensitivity to poor conditioning \cite{novak2018sensitivity} and slow convergence in high-dimensional and non-convex landscapes \cite{chen2023symbolic}. To overcome these drawbacks, SGD with momentum and adaptive learning rates has been proposed to enhance robustness and accelerate convergence.  The former introduced first-order momentum to suppress the oscillation of SGD \cite{qian1999momentum}, and the latter updated the learning rate iteratively based on historical information. 
	
	For instance,  the adaptive gradient (AdaGrad  \cite{duchi2011adaptive}) interpolated the accumulated second-order moment of gradients to achieve the adaptive learning rate. 
	The root mean-squared propagation (RMSProp \cite{tieleman2012lecture}) exploited an exponential weighting technique to balance the distant historical information and the knowledge of the current second-order gradients, avoiding premature termination encountered by AdaGrad. 
	The adaptive moment (Adam  \cite{kingma2014adam}) combined the first-order moment and adaptive learning rates, thereby exhibiting robustness to hyperparameters. 
Adaptive method setup based gradient (AMSGrad \cite{jR2018on}) took smaller learning rates than those of Adam by using a maximum value for normalizing the running average of the gradients, thereby fixing the convergence of Adam.
	Partially Adam (Padam \cite{chen2018closing}) employed a similar strategy of AMSGrad with the difference of the rate power.
 Adam with dynamic bounds on learning rates (AdaBound  \cite{luo2019adaptive}) designed a clipping mechanism on Adam-type learning rates by clipping the gradients larger than
a threshold to avoid gradient explosion. 
Parallel restarted  SGD (PRSGD \cite{yu2019parallel}) simply benefited from a decayed power learning rate. 
A layerwise adaptive large batch optimization technique (Lamb \cite{You2020Large}) performed per dimension normalization with respect to the square root of the second moment used in Adam and set large batch sizes suggested by the layerwise adaptive rate scaling method (Lars \cite{you2017scaling}). 

For some other SGD-based algorithms, one can refer to a per-dimension learning rate method based gradient descent (AdaDelta \cite{Zeiler2012ADADELTAAA}), a variant of Adam based on the infinity norm (Adamax  \cite{kingma2014adam}),  Adam with decoupled weight decay (AdamW \cite{loshchilov2017decoupled}), a structure-aware preconditioning algorithm (Shampoo \cite{gupta2018shampoo}), momentum orthogonalized by Newton-Schulz iterations (Muon \cite{jordan2024muon}), Shampoo with Adam in the preconditioner’s eigenbasis (SOAP \cite{vyas2025soap}), Adam with fewer learning rates (Adam-mini \cite{zhang2025adammini}),
  and those reviewed in nice surveys \cite{bottou2010large, ruder2016overview, bottou2018optimization}.

The aforementioned algorithms primarily leveraged the heavy ball acceleration technique to estimate the first- and second-order moments of the gradient. An alternative approach for accelerating convergence is Nesterov acceleration, which has been theoretically proven to converge faster \cite{nesterov1983method,nesterov1988approach}. This has led to the development of several Nesterov-accelerated SGD algorithms. For example, Nesterov-accelerated Adam (NAdam \cite{dozat2016incorporating}) integrated this technique into Adam to enhance convergence speed. More recently,  an adaptive Nesterov momentum algorithm (Adan \cite{Adan24}) adopted Nesterov momentum estimation to refine the estimation of the first- and second-order moments of the gradient, thereby improving adaptive learning rate adjustments for faster convergence.

\subsection{ADMM-based learning algorithms} 
The alternating direction method of multipliers  (ADMM) \cite{gabay1976dual,boyd2011distributed} is a promising framework for solving large-scale optimization problems because it decomposes complex problems into smaller and more manageable sub-problems. This characteristic makes ADMM particularly well-suited for various challenges in distributed learning. It has demonstrated substantial potential in applications, such as image compressive sensing \cite{yang2018admm}, federated learning \cite{feddp21, zhou2023federated, zhouli23}, reinforcement learning \cite{xu2023federated}, and few-shot learning \cite{wang2023new}.
 It is worth noting that there are many other distributed optimization frameworks, such as dual averaging \cite{agarwal2010distributed, duchi2011dual, hosseini2013online}, push-sum  \cite{tsianos2012push}, and push-pull \cite{pu2020push}. However, in the sequel, our focus is on the family of distributed algorithms, ADMM, as they are most directly relevant to the design and convergence analysis of our algorithm.
 
When applied to deep model training, early studies first relaxed optimization models before developing ADMM-based algorithms. Therefore, these methods can be classified as model-driven/model-specific approaches. They targeted the relaxations rather than the original problems, enabling better handling of the challenges associated with highly non-convex optimization landscapes.  For instance,  an ADMM-based algorithm in \cite{taylor2016training} addressed a penalized formulation, avoiding pitfalls that hinder gradient-based methods in non-convex settings. Similarly, a deep learning ADMM algorithm \cite{wang2019admm} and its variant \cite{ebrahimi2024aa} were designed to enhance convergence by addressing penalization models as well. Moreover, the sigmoid-ADMM algorithm \cite{zeng2021admm} utilized the algorithmic gradient-free property to address saturation issues with sigmoid activation functions. 

Despite their advantages, model-driven ADMM approaches face two critical issues. Firstly, these algorithms exhibit high computational complexity because they require computing full gradients using the entire dataset and performing matrix inversions for weight updates, making them impractical for large-scale tasks. Additionally, the reliance on the specific structure of optimization models limits their general applicability, as different deep models often necessitate distinct formulations.
 

To address the aforementioned limitations, data-driven ADMM algorithms have emerged as a promising alternative, aiming to reduce computational burdens and enhance adaptability to diverse tasks. For instance, the deterministic ADMM \cite{zhou2023federated, zhouli23}, designed for distributed learning problems, achieved high accuracy but low computational efficiency due to the use of full dataset gradients. Stochastic ADMM (SADMM) methods tackled the inefficiency by replacing full gradients with stochastic approximations. Examples like stochastic ADMM  (S-ADMM \cite{ouyang2013stochastic}) and distributed stochastic ADMM (PS-ADMM \cite{ding2019stochastic}) aimed to solve sub-problems exactly and thus still involved high computational costs. 

To further enhance convergence and reduce stochastic gradient variance, advanced techniques of variance reduction and Nesterov acceleration have been incorporated into SADMM. Representatives consist of  stochastic average ADMM \cite{zhong2014fast}, stochastic path-integrated differential estimator based ADMM (SPIDER-ADMM \cite{huang2019faster}), and accelerated stochastic variance reduced gradient based ADMM (ASVRG-ADMM) for convex \cite{liu2017accelerated} and non-convex problems \cite{zeng2024accelerated}. However, these enhanced methods still relied on access to full gradients for the aim of reducing stochastic gradient variance, posing challenges for large-scale applications.


\subsection{Contributions}

In this work, we propose a data-driven preconditioned inexact SADMM algorithm, termed as PISA. It distinguishes itself from prior approaches and aims to reduce computational costs, relax convergence assumptions, and enhance numerical performance. The key contributions are threefold.

{\textit{$a)$ A general algorithmic structure.}}
The algorithm is based on a preconditioned inexact SADMM, combining simplicity with high generality to handle a wide range of deep learning applications.  First, the use of preconditioning matrices allows us to incorporate various forms of useful information, such as first- and second-moment (yielding SISA, a variant of PISA), second-order information (e.g., the Hessian), and orthogonalized momentum by Newton-Schulz iterations (leading to NSISA, another variant of PISA), into the updates, thereby enhancing the performance of the proposed algorithms. Moreover, the proposed algorithmic framework is inherently compatible with parallel computing, making it ideal for large-scale data settings.  

{\textit{$b)$ Strong convergence theory under a sole assumption.}}
PISA is proven to converge under a single assumption: the Lipschitz continuity of the gradient. Despite relying on stochastic gradients, it avoids many of the assumptions typically required by stochastic algorithms. As highlighted in Table \ref{table:compare-conditions}, all algorithms, except for PISA and FedAvg \cite{Li2020On}, have drawn identically and independently distributed (IID) samples to derive stochastic gradients for unbiased gradient estimation. However, real-world data are often heterogeneous (i.e., non-IID), a phenomenon commonly referred to as statistical or data heterogeneity \cite{li2020federated, kairouz2021advances, li2022federated,  ye2023heterogeneous}, which poses significant challenges to the convergence of these algorithms for IID datasets. 

Additionally, the table also presents the algorithmic convexity of two types of convergence.  Type I convergence, ${F(\bw^T)-F^*=B}$, refers to a rate $B$ at which objective values $\{F(\bw^T)\}$ of generated sequence $\{\bw^T\}$ approach $F^*$, where $T$ is the number of iterations and $F^*$ is the limit of sequence $\{F(\bw^T)\}$ or a function value at a stationary point. For instance,  ${B=O(1/T)}$ indicates a sublinear convergence, while  ${B=O(\gamma^T)}$ with ${\gamma\in(0,1)}$ implies a linear rate. Type II convergence describes how quickly the length, $\{\|\nabla F(\bw^T)\|\}$, of gradients diminishes, reflecting the stationarity of the iterates. Therefore, ASVRG-ADMM  \cite{liu2017accelerated} and PISA achieve the best Type I convergence rate. However, the former imposes several restrictive assumptions, while PISA requires a single assumption (i.e., \textcircled{12}).  We emphasize that condition \textcircled{12} is closely related to the locally Lipschitz continuity of the gradient, which is a mild assumption. It is weaker than \textcircled{3}, and any twice continuously differentiable function satisfies \textcircled{12}.  More importantly, we eliminate the need for conditions on the boundedness of the stochastic gradient, the second moment of the stochastic gradient, and variance. This makes PISA well-suited for addressing the challenges associated with data heterogeneity, an open problem in federated learning  \cite{kairouz2021advances, ye2023heterogeneous}.

{\textit{$c)$ High numerical performance for various applications.}} The effectiveness of PISA and its two variants, SISA and NSISA, is demonstrated through comparisons with many state-of-the-art optimizers   using several deep models:  VMs, LLMs, reinforcement learning models (RLMs), generative adversarial networks (GANs), and recurrent neural networks (RNNs), highlighting its great potential for extensive applications.  In particular, numerical experiments on heterogeneous datasets corroborate our claim that PISA effectively addresses this challenge.

 \begin{table}[!h]
\centering
\caption{Assumptions imposed by different stochastic algorithms for convergence: \textcircled{1}: Convexity; 
\textcircled{2}: Strong convexity;  
\textcircled{3}: Lipschitz continuity of the gradient; \textcircled{4}: Bounded (stochastic) gradient or the second-moment of the (stochastic) gradient; 
\textcircled{5}: Bounded variance; 
\textcircled{6}: Bounded sequence generated by the algorithm; \textcircled{7}: Unbiased gradient estimation; 
\textcircled{8}: Compact convex feasible region;
\textcircled{9}: Lipschitz continuity of the objective function;
\textcircled{10}: Lipschitz sub-minimization paths;
\textcircled{11}: Bounded dual variables;
\textcircled{12}: Lipschitz continuity of the gradient on a bounded region.}\label{tab:comp-assumps}
\renewcommand{\arraystretch}{1.1}\addtolength{\tabcolsep}{2pt}
\begin{tabular}{lccccr}
\hline \hline 
  Algorithms&  Refs.  & Class & Data& Convergence rate $B$ & Assumptions \\ \hline\hline 
  \multicolumn{6}{c}{ Type I convergence: $ F(\bw^T)-F^*= B$}\\\hline  
 
AdaGrad& \cite{duchi2011adaptive} &SGD&IID& $O(1/\sqrt{T})$ &\textcircled{1}\textcircled{4}\textcircled{6} \\ 

 Adam  & \cite{kingma2014adam} &SGD&IID& $O(1/\sqrt{T})$&\textcircled{1}\textcircled{4}\textcircled{6} \\
 
 RMSProp& \cite{mukkamala2017variants} &SGD&IID& $O(1/\sqrt{T})$&\textcircled{1}\textcircled{4}\textcircled{6}\\ 
  
   AMSGrad & \cite{jR2018on} &SGD&IID& $O(1/\sqrt{T})$ &\textcircled{1}\textcircled{4}\textcircled{8}  \\   
   
  AdaBound & \cite{luo2019adaptive} &SGD&IID& $O(1/\sqrt{T})$ &  \textcircled{1}\textcircled{4}\textcircled{8}\\
  
    FedAvg & \cite{Li2020On}  &SGD&IID/Non-IID& $O(1/T)$ &  \textcircled{2}\textcircled{3}\textcircled{4}\textcircled{5} \\
      
     LocalSGD & \cite{stich2018local} &SGD&IID& $O(1/T)$&\textcircled{2}\textcircled{3}\textcircled{4}\textcircled{5} \\
   
      S-ADMM & \cite{ouyang2013stochastic} &SADMM&IID&  $O(\log(T)/T)$ &  \textcircled{2}\textcircled{4}\textcircled{8}\\
      
      
      ASVRG-ADMM& \cite{liu2017accelerated}&SADMM&IID&  $O(1/T^2)$ &  \textcircled{1}\textcircled{3}\textcircled{8}\textcircled{11}\\
      ASVRG-ADMM& \cite{liu2017accelerated}&SADMM&IID&  $O(\gamma^T)$ &  \textcircled{2}\textcircled{3}\textcircled{8}\textcircled{11}\\
  PS-ADMM & \cite{ding2019stochastic} &SADMM&IID&  $O(1/T)$ &  \textcircled{2}\textcircled{3}\textcircled{9}\\
    PISA & Ours  &SADMM&IID/Non-IID& $O(\gamma^T)$ &  \textcircled{12} \\
\hline 
  
  \multicolumn{6}{c}{Type II convergence: $\|\nabla F(\bw^T)\|^2 = B$}\\\hline 
 Adam & \cite{zou2019sufficient} &SGD&IID& $O(\log(T)/\sqrt{T})$ &  \textcircled{3}\textcircled{4}\textcircled{7}\\
  Padam & \cite{chen2018closing} &SGD&IID& $O(1/\sqrt{T})$ &  \textcircled{3}\textcircled{4}\textcircled{7}\\
 PRSGD & \cite{yu2019parallel} &SGD&IID& $O(1/\sqrt{T})$ &  \textcircled{3}\textcircled{4}\textcircled{5}\textcircled{7}\\
Lamb & \cite{You2020Large} &SGD&IID& $O(1/\sqrt{T})$ &  \textcircled{3}\textcircled{4}\textcircled{5}\\
  Adan & \cite{Adan24} &SGD&IID& $O(1/\sqrt{T})$ &  \textcircled{3}\textcircled{4}\textcircled{5}\textcircled{7}\\
FedAvg & \cite{reddi2021adaptive}  &SGD&IID/Non-IID& $O(1/\sqrt{T})$ &  \textcircled{3}\textcircled{4}\textcircled{5}\textcircled{7} \\  
SPIDER-ADMM & \cite{huang2019faster} & SADMM &IID& $O(1/T)$ &  \textcircled{3}\textcircled{4}\\
  PISA & Ours  &SADMM&IID/Non-IID& $O(\gamma^T)$ &  \textcircled{12} \\
\hline \hline 
\end{tabular}
\label{table:compare-conditions}
\end{table}
 
\subsection{Organization and Notation}

The paper is organized as follows: In the next section, we introduce the main model and develop the proposed algorithm, PISA. Section \ref{sec:convergence} provides rigorous proofs of the convergence. Section \ref{sec:SISA} specifies the pre-condition by the second moment to derive a variation of PISA, termed as SISA (second-moment-based inexact SADMM). Extensive experimental results validating the effectiveness of SISA are provided in Section \ref{sec:numerical}. Concluding remarks are discussed in the last section.

Throughout the paper, let $[m]:=\{1,2,\ldots,m\}$, where `$:=$' means `define'. The cardinality of a set $\D$ is written as $|\D|$. 
For two vectors $\bw$ and $\bv$, their inner product is denoted  by $\langle\bw,\bv\rangle:=\sum_iw_iv_i$.  Let ${\|\cdot\|}$ be the Euclidean norm for vectors, namely $\|\bw\|=\sqrt{\langle\bw,\bw\rangle}$, and the Spectral norm for matrices. A ball with a positive radius $r$ is written as ${\N(r):=\{\bw:\|\bw\|\leq r\}}$.
A symmetric positive semi-definite matrix $\Q$ is written as $\Q\succeq 0$. Then $\mathbf{P}\succeq \Q$ means that $ \mathbf{P}- \Q\succeq 0$. Denote the identity matrix by $\I$ and let $\textbf{1}$ be the vector with all entries being $1$.   We write 
\begin{eqnarray*} 
&&\P~=(\bpi_1,\bpi_2,\ldots,\bpi_m),\qquad \W~=(\bw_1,\bw_2,\ldots,\bw_m),\\
&&\M=(\bm_1,\bm_2,\ldots,\bm_m),\qquad \bsi=(\sigma_1,\sigma_2,\ldots,\sigma_m).
\end{eqnarray*} 
Similar rules are also employed for the definitions of $\P^\ell, \W^\ell,\M^\ell$, and $\bsi^\ell$.
\section{Preconditioned Inexact SADMM}
We begin this section by introducing the mathematical optimization model for general distributed learning. Then we go through the development of the algorithm. 

\subsection{Model description}
Suppose we are given a set of data as $\D:=\{\bx_t:t=1,2,\ldots,|\D|\}$, where $\bx_t$ is the $t${th} sample. Let $f(\bw; \bx_t)$ be a function (such as neural networks) parameterized by $\bw$ and sampled by $\bx_t$. The total loss function on $\D$ is defined by $ \sum_{\bx_t\in\D} f\left(\bw; \bx_t\right)/{|\D|}$. 
We then divide data $\D$ into $m$ disjoint batches, namely,  $\D= \D_1\cup\D_2\cup\ldots\cup\D_m$ and $\D_i\cap\D_{i'}=\emptyset$ for any two distinct $i$ and $i'$.   Denote
\begin{eqnarray}  \label{def-Fbn}
H_{i}(\bw;\D_i):=  \frac{1 }{|\D_{i}|}  \sum_{\bx_t\in\D_{i}} f\left(\bw; \bx_t\right) \qquad\text{and}\qquad\alpha_i:=\frac{|\D_{i}|}{|\D|}.
\end{eqnarray} 
Clearly, $\sum_{i=1}^m \alpha_i=1$. Now, we can rewrite the total loss as follows,
\begin{eqnarray*}  
\frac{1}{|\D|} \sum_{\bx_t\in\D} f\left(\bw; \bx_t\right)
 = \frac{1}{|\D|} \sum_{i=1}^m \sum_{\bx_t\in\D_{i}} f\left(\bw; \bx_t\right)
 = \sum_{i=1}^m \alpha_i H_{i}(\bw;\D_i).
\end{eqnarray*}
The task is  to learn an optimal parameter to minimize the following regularized loss function,   
\begin{eqnarray}  \label{opt-prob}
 \min\limits_{\bw}~  \sum_{i=1}^m \alpha_i H_{i}(\bw;\D_i) + \frac{\mu}{2}\|\bw\|^2, 
\end{eqnarray}
where $\mu\geq0$ is a penalty constant and $\|\bw\|^2$ is a  regularization. 
\subsection{Main model}
Throughout the paper, we focus on the following  equivalent model of problem (\ref{opt-prob}),
\begin{eqnarray}  \label{opt-prob-distribute}
\begin{aligned}
F^* := \min\limits_{\bw,\W}~\sum_{i=1}^m \alpha_i   F_{i}(\bw_i) + \frac{\lambda}{2}\|\bw\|^2,~~~
{\rm s.t. }~\bw_{i} = \bw,~ i\in[m],
\end{aligned}
\end{eqnarray}
where $ \lambda\in[0,\mu]$ and
\begin{eqnarray}  \label{def-F_i-F}
\begin{aligned}
 F_{i}(\bw)&:= F_{i}(\bw;\D_i):= H_{i}(\bw;\D_i) +\frac{\mu-\lambda}{2}\|\bw\|^2, \\
F (\bw)&:=\sum_{i=1}^m \alpha_i  F_{i}(\bw), \qquad F_\lambda(\bw):=F (\bw)+ \frac{\lambda}{2}\|\bw\|^2.
\end{aligned}
\end{eqnarray}
In problem \eqref{opt-prob-distribute}, $m$ auxiliary variables $\bw_i$ are introduced in addition to the global parameter $\bw$. We emphasize that problems (\ref{opt-prob}) and (\ref{opt-prob-distribute}) are equivalent in terms of their optimal solutions but are expressed in different forms when ${\lambda\in[0,\mu)}$, and they are identical when ${\lambda=\mu}$. Throughout this work, we assume that optimal function value $F^*$ is bounded from below. It is worth noting that any stationary point $\bw$ of problem (\ref{opt-prob-distribute})  satisfies
\begin{eqnarray}\label{stationary-point}
\begin{aligned}
0\in  \nabla F_\lambda(\bw)= \sum_{i=1}^m \alpha_i   \nabla F_{i}(\bw)    + \lambda \bw,
\end{aligned}
\end{eqnarray}
where $ \nabla F_{i}(\bw)$ is the sub-differential \cite[Definition 8.3]{rockafellar2009variational} of $F_{i}(\bw)$. In our convergence analysis, we assume that $F_{i}$ is continuously differentiable, so sub-differential $ \nabla F_{i}(\bw)$ reduces to the gradient of $F_{i}(\bw)$. In this case,  inclusion `$\in$' reduces to `$=$'. Moreover, if each $F_{i}$ is convex (which is unnecessary in this paper), then the stationary points coincide with the optimal solutions to (\ref{opt-prob-distribute}). Since problems (\ref{opt-prob-distribute}) and (\ref{opt-prob}) are equivalent, their stationary points are also identical.

\subsection{The algorithmic design}
When employing ADMM to solve problem (\ref{opt-prob-distribute}), we need its associated augmented Lagrange function, which is defined as follows,
\begin{eqnarray}  \label{opt-prob-distribute-Lag}
\begin{aligned}
\L\left(\bw,\W,\P;\bsi\right)
&:= \sum_{i=1}^m \alpha_i L_{i}(\bw,\bw_i, \bpi_i;\sigma_i) + \frac{\lambda}{2}\|\bw\|^2, \\
L_{i}(\bw,\bw_i, \bpi_i;\sigma_i)&:=F_{i}(\bw_{i})+\langle \bpi_{i}, \bw_{i}- \bw\rangle + \frac{\sigma_i}{2} \| \bw_{i}- \bw\|^2,
\end{aligned}
\end{eqnarray}
where $\sigma_i>0$ and  $\bpi_{i}, i\in[m]$ are the Lagrange multipliers.  Based on the above augmented Lagrange function, the conventional ADMM updates each variable in $(\bw,\W,\P)$ iteratively. However, we modified the framework as follows. 
Given initial point $(\bw^0,\W^0,\P^0;\bsi^0)$, the algorithm performs the following steps iteratively for $\ell=0,1,2,\ldots,$ 
\begin{subequations}\label{frame-ADMM}	
			\begin{alignat}{4}
			 \label{w-update}	
			\bw^{\ell+1} &= {\rm arg}\min_{\bw} \L\left(\bw,\W^\ell,\P^\ell;\bsi^{\ell} \right),\\[1ex]
			\label{wb-update}	
			\bw_{i}^{\ell+1} &= {\rm arg}\min_{\bw_{i}}  L_i\left(\bw^{\ell+1},\bw_i, \bpi_i^{\ell}; \sigma_i^{\ell+1} \right)+\frac{\rho_i}{2}\left\langle\bw_i-\bw^{\ell+1}, \Q_i^{\ell+1}\left(\bw_i-\bw^{\ell+1}\right)\right\rangle,\\[1ex]
			 \label{pib-update}	
			\bpi_{i}^{\ell+1} &= \bpi_{i}^{\ell}  + \sigma_i^{\ell+1} ( \bw_{i}^{\ell+1}- \bw^{\ell+1}),
			 \end{alignat}
		\end{subequations}
 	for each $i\in[m]$, where ${\rho_i>0}$,  both scalar $\sigma_i^{\ell+1}$ and matrix $\Q_i^{\ell+1}\succeq 0$ will be updated properly. Hereafter, superscripts $\ell$ and $\ell+1$ in $\sigma_i^{\ell}$ and $\sigma_i^{\ell+1}$ stand for the iteration number rather than the power.  Here, $\Q_i^{\ell+1}$ is commonly referred to as an (adaptively) preconditioning matrix in preconditioned gradient methods \cite{li2017preconditioned, gupta2018shampoo,agarwal2019efficient,yong2023general}. 
\begin{remark}The primary distinction between algorithmic framework (\ref{frame-ADMM}) and conventional ADMM lies in the inclusion of a term $\frac{\rho_i}{2}\langle\bw_i-\bw^{\ell+1}, \Q_i^{\ell+1}(\bw_i-\bw^{\ell+1})\rangle$. This term enables the incorporation of various forms of useful information, such as second-moment, second-order information (e.g., Hessian), and orthogonalized momentum by Newton-Schulz iterations, thereby enhancing the performance of the proposed algorithms; see Section~\ref{sec:SISA} for more details.
\end{remark}
One can check that sub-problem \eqref{w-update} admits a closed-form solution outlined in \eqref{sub-w-mini}. 
 For sub-problem (\ref{wb-update}), to accelerate the computational speed, we solve it inexactly by
  \begin{eqnarray}\label{sub-wbn}
\begin{aligned}
\bw_{i}^{\ell+1} &= {\rm arg}\min\limits_{\bw_{i}} ~ \langle    \bpi_{i}^{\ell}, \bw_{i} -\bw^{\ell+1}\rangle + \frac{ \sigma_i^{\ell+1} }{2} \|\bw_{i} -\bw^{\ell+1}\|^2 \\ 
&+  F_{i}( \bw^{\ell+1}) + \langle    \nabla F_{i}( \bw^{\ell+1};\B_i^{\ell+1})  , \bw_{i} - \bw^{\ell+1}\rangle   +\frac{\rho_i}{2}\left\langle\bw_i-\bw^{\ell+1}, \Q_i^{\ell+1}\left(\bw_i-\bw^{\ell+1}\right)\right\rangle  \\ 
&=\bw^{\ell+1} - \Big(\sigma_i^{\ell+1}  \I + \rho_i \Q_i^{\ell+1} \Big)^{-1}  \Big(  \bpi_{i}^{\ell} +  \nabla F_{i}( \bw^{\ell+1};\B_i^{\ell+1})  \Big).
\end{aligned}
\end{eqnarray}

\begin{algorithm}[!t]
    \SetAlgoLined
Divide  $\D$ into $m$ disjoint  batches {$\{\D_1,\D_2,\ldots,\D_m\}$} and calculate $\alpha_i$ by (\ref{def-Fbn}). 

{{Initialize $\bw^0=\bw_i^0=\bpi_i^0=0$,}}  $\gamma_i\in[3/4,1),$  and $(\sigma_i^0, \eta_i, \rho_i)>0$  for each $i\in[m]$.  

\For{$\ell=0,1,2,\ldots$}{
 \begin{eqnarray}  \label{sub-w-mini}
\bw^{\ell+1}  =  \frac{\sum_{i=1}^m\alpha_{i} \left(\sigma_i^\ell\bw_{i}^{\ell}+\bpi_{i}^{\ell}  \right) }{\sum_{i=1}^m\alpha_{i}  \sigma_i^\ell+\lambda}.
 \end{eqnarray}   
\For{$i=1,2,\ldots,m$}{  
 \begin{align}
 \label{sub-B-g}  &\text{Randomly draw a mini-batc  $ {\B}_{i}^{\ell+1}\in\D_{i}$ and calculate $\bg_i^{\ell+1}= \nabla F_{i}( \bw^{\ell+1}; {\B}_{i}^{\ell+1})$}.\\[1ex]
 \label{sub-Q-eta}  &\text{Choose $\Q_i^{\ell+1} $ to satisfy $\eta_i\I\succeq \Q_i^{\ell+1} \succeq0$.}\\[1ex]
 \label{sub-sigma-mini} &\sigma_i^{\ell+1}  = \sigma_i^{\ell}/\gamma_i. \\[1ex]
 \label{sub-wbn-mini} &\bw_{i}^{\ell+1}  = \bw^{\ell+1} - \Big(\sigma_i^{\ell+1} \I + \rho_i \Q_i^{\ell+1} \Big)^{-1}  \Big(  \bpi_{i}^{\ell} + \bg_i^{\ell+1} \Big). \\[1ex]
\label{sub-pin} &\bpi_{i}^{\ell+1} =  \bpi_{i}^{\ell} +  \sigma_i^{\ell+1}(\bw_{i}^{\ell+1}- \bw^{\ell+1}). 
 \end{align} 
}
} 
\caption{\textbf{P}reconditioned \textbf{I}nexact \textbf{S}tochastic \textbf{A}DMM (PISA).  }\label{algorithm-ADMM-mini}
\end{algorithm}

This update admits three advantages. First, it solves problem (\ref{wb-update}) by a closed-form solution, namely, the second equation in \eqref{sub-wbn}, reducing the computational complexity.  Second,  we approximate $ F_{i}(\bw) $ using its first-order approximation at $\bw^{\ell+1}$ rather than $\bw_{i}^{\ell}$, which facilitates each batch parameter $\bw_{i}^{\ell+1}$ to tend to $\bw^{\ell+1}$ quickly, thereby accelerating the overall convergence. Finally, $\nabla F_{i}( \bw^{\ell+1};\B_i^{\ell+1})$ serves as a stochastic approximation of true gradient $\nabla F_{i}( \bw^{\ell+1})=\nabla F_{i}( \bw^{\ell+1};\D_i)$, as defined by \eqref{def-F_i-F}, where $\B_i^{\ell+1}$ is a random sample from $\D_i$. By using sub-batch datasets $\{\B_1^{\ell+1},\ldots,\B_m^{\ell+1}\}$ in every iteration,  rather than full data $\D=\{\D_1,\ldots,\D_m\}$, the computational cost is significantly reduced. Overall, based on these observations, we name our algorithm PISA, which stands for Preconditioned Inexact SADMM, as described in Algorithm \ref{algorithm-ADMM-mini}.

Another advantageous property of PISA  is its ability to perform parallel computation, which stems from the parallelism used in solving sub-problems in ADMM. At each iteration,  $m$ nodes (i.e., $i=1,2,\cdots,m$) update their parameters by \eqref{sub-B-g}-\eqref{sub-pin} in parallel, thereby enabling the processing of large-scale datasets. Moreover, when specifying the preconditioning matrix, $\Q_i^{\ell+1}$, as a diagonal matrix (as outlined in Section \ref{sec:SISA}) and sampling ${\B}_{i}^{\ell+1}$ with small batch sizes, each node exhibits significantly low-computational complexity, facilitating fast computation.

\section{Convergence of PISA}\label{sec:convergence}
In this subsection, we aim to establish the convergence property of Algorithm \ref{algorithm-ADMM-mini}. To proceed with that,  we first  define a critical bound by  
 \begin{eqnarray} \label{def-eps}
\begin{aligned}
\varepsilon_i(r)&:=  \sup_{\B_i,\B'_i\subseteq\D_i,\bw\in\N(r)} 64\left\|\nabla F_i(\bw;\B_i
)- \nabla F_i(\bw;\B'_i)\right\|^2,~~\forall~i\in[m].
\end{aligned}
\end{eqnarray} 
\begin{lemma}\label{bound-varepsilon} $\varepsilon_i(r)<\infty$ for any given $r\in(0,\infty)$ and any $i\in[m]$. 
\end{lemma}
\begin{proof} As $F_i$ is continuously differentiable, $\nabla F_i$ is continuous. We note that $\N(r)$ is bounded due to $r\in(0,\infty)$, thereby gap $\|\nabla F_i(\bw;\B_i
)- \nabla F_i(\bw;\B_i')\|^2$ is bounded for fixed $\B_i$ and $\B'_i$. Then there are finitely many $\B_i$ and $\B'_i$ in $\D_i$, indicating $\varepsilon_i(r)$ is bounded.
\end{proof}
\noindent One can observe that ${\varepsilon_i(r)=0}$ for any ${r>0}$ if we take the full batch data in each step, namely, choosing ${\B_i^{\ell}=(\B_i^{\ell})'=\D_i}$ for every ${i\in[m]}$ and all ${\ell\geq1}$. However for  min-batch dataset ${\B_i^{\ell}\subset\D_i}$, this parameter is related to the bound of variance $\E\left\|\nabla F_i(\bw;\B_i
)- \nabla F_i(\bw;\D_i)\right\|^2$, which is commonly assumed to be bounded for any $\bw$ \cite{stich2018local,yu2019parallel,Li2020On,CooperativeSGD21,Adan24}. However, in the subsequent analysis, we can verify that both generated sequences $\{\bw^\ell\}$ and $\{\bw_i^\ell\}$ fall into a bounded region $\N(\delta)$ for any $i\in[m]$ with $\delta$ defined as \eqref{define-delta}, thereby leading to a finitely bounded $\varepsilon_i(\delta)$ naturally, see Lemma \ref{descent-lemma-L}. In other words, we no longer need to assume the boundedness of the variance, $\E\left\|\nabla F_i(\bw;\B_i) - \nabla F_i(\bw;\D_i)\right\|^2$ for any $\bw$. This assumption is known to be somewhat restrictive, particularly for non-IID or heterogeneous datasets.  Therefore, the theorems we establish in the sequel effectively address this critical challenge  \cite{kairouz2021advances,  ye2023heterogeneous}. Therefore, our algorithm demonstrates robust performance in settings with heterogeneous data.


\subsection{Convergence analysis}
To establish convergence, we assume function $f$ has a  Lipschitz continuous gradient on a bounded region, namely, the gradient is locally Lipschitz continuous. This is a relatively mild condition. Functions with (global) Lipschitz continuity and twice continuously differentiable functions satisfy this condition. It is known that the Lipschitz continuity of the gradient is commonly referred to as L-smoothness. Therefore, our assumption can be regarded as L-smoothness on a bounded region, which is weaker than L-smoothness.
\begin{assumption}\label{assumption} For each $t\in[|\D|]$,  gradient $\nabla f(\cdot;\bx_t)$ is Lipschitz continuous with a constant $c(\bx_t)>0$ on $\N(2\delta)$. Denote $c_i:=\max_{\bx_t\in\D_i}c(\bx_t)$ and $ r_i:=c_i+\mu-\lambda$ for each $i\in[m]$. 
\end{assumption}
First, given a constant  $\sigma>0$, we define a set
\begin{eqnarray} \label{compact-set} 
\Omega :=   \left\{(\bw,\W):~ \dsum  \alpha_i \left(  F_{i}(\bw )  +  \dfrac{\lambda}{2}\|\bw \|^2 + \dfrac{\sigma}{2}  \| {\bw}_i-{\bw} \|^2  \right)   \leq F(\bw^0) + \frac{1}{1-\gamma}  \right\},
\end{eqnarray}
where $\gamma:=\max_{i\in[m]}\gamma_i$, based on which  we further define
\begin{eqnarray} \label{define-delta} 
\delta :=  \sup_{(\bw,\W)\in\Omega} \Big\{ \|\bw \|, \| {\bw}_1 \|,  \| {\bw}_2 \|,\ldots,\| {\bw}_m \|  \Big\}.
\end{eqnarray}
This indicates that any point $(\bw,\W)\in\Omega$ satisfies $ \{\bw,  \bw_1,\ldots,\bw_m\} \subseteq \N(\delta)$. Using this $\delta$, we initialize $\bsi^0:=(\sigma_1^0,\sigma_2^0,\cdots,\sigma_m^0)$ by
\begin{eqnarray}  \label{choice-of-sigma}
\sigma^0:=\min \{\sigma_1^0,\sigma_2^0,\cdots,\sigma_m^0\} \geq~ 8\max_{i\in[m]}\Big\{\sigma,~\rho_i\eta_i,~ r_i,~\delta^{-2},~  \varepsilon_i(2\delta)\Big\}.
\end{eqnarray}
It is easy to see that $\Omega$ is a bounded set due to $F_i$ being bounded from below. Therefore, $\delta$ is bounded and so is $\varepsilon_i(2\delta)$ due to Lemma \ref{bound-varepsilon}. Hence, $\sigma^0$ in \eqref{choice-of-sigma} is a well-defined constant, namely, $\sigma^0$ can be set as a finite positive number. For notational simplicity, hereafter, we denote
\begin{eqnarray}\label{def-notation}
\begin{aligned}
  &\triangle\bw^{\ell}:= \bw^{\ell}-\bw^{\ell-1},~~~&&\triangle\bw_{i}^{\ell}:= \bw_{i}^{\ell}-\bw_{i}^{\ell-1}\\
  & \triangle\bpi_{i}^{\ell}:= \bpi_{i}^{\ell}-\bpi_{i}^{\ell-1},&&\triangle\overline{\bw}^{\ell}_i:= \bw^{\ell}_i-\bw^{\ell},\\ 
   & \triangle\bg_i^{\ell}:= \bg_i^{\ell}- \nabla F_i(\bw^{\ell}),&&\L^{\ell}:=\L(\bw^\ell,\W^\ell,\P^\ell;\bsi^\ell).
\end{aligned}\end{eqnarray}
Our first result shows the descent property of a merit function associated with  $\L^{\ell}$.  
\begin{lemma}\label{descent-lemma-L} Let $\{(\bw^\ell,\W^\ell,\P^\ell)\}$ be the sequence generated by Algorithm \ref{algorithm-ADMM-mini} with $\bsi^0$  chosen as (\ref{choice-of-sigma}). Then the following statements are valid under Assumption \ref{assumption}.
\begin{itemize}
\item[1)] For any $\ell\geq0$, sequence $\{\bw^\ell, \bw_1^\ell,\ldots,\bw_m^\ell\} \subseteq \N(\delta)$.
\item[2)] For any $\ell\geq0$,
\begin{eqnarray} \label{descent-L}
\widetilde{\L}^{\ell}- \widetilde{\L}^{\ell+1} \geq \sum_{i=1}^m  \alpha_i\left[ \frac{\sigma_i^{\ell}+2\lambda}{4}\left\|  \triangle\bw^{\ell+1} \right\|^2 +\frac{ \sigma_i^{\ell}}{4} \left\|\triangle \bw_{i}^{\ell+1}\right\|^2 \right],  
\end{eqnarray}
 where $\widetilde{\L}^{\ell}$ is defined by
{{\begin{eqnarray}\label{def-tilde-L}
\begin{aligned}
\widetilde{\L}^{\ell} :=\L^\ell+\sum_{i=1}^m \alpha_i \left[ \dfrac{8}{\sigma_i^{\ell}}  \left\|\rho_i \Q_i^{\ell} \triangle\overline{\bw}^{\ell}_i\right\|^2 + \frac{\gamma_i^\ell}{16(1- \gamma_i)}    \right].
\end{aligned}\end{eqnarray}}}
\end{itemize}
\end{lemma}
The proof of the above lemma is given in Appendix \ref{app:proof-lemma32}. This lemma is derived from a deterministic perspective. Such a success lies in considering the worst case of bound $\varepsilon_i(2\delta)$ (i.e., taking all possible selections of $\{\B_1^\ell,\ldots,\B_m^\ell\}$ into account). Based on the above key lemma, the following theorem shows the sequence convergence of the algorithm.
\begin{theorem}\label{main-convergence} Let $\{(\bw^\ell,\W^\ell,\P^\ell)\}$ be the sequence generated by Algorithm \ref{algorithm-ADMM-mini} with  $\bsi^0$ chosen as (\ref{choice-of-sigma}). Then the following statements are valid under Assumption \ref{assumption}.
\begin{itemize}[leftmargin=17pt]
\item[1)] Sequences $\{ {\L}^\ell\}$ and $\{\widetilde{\L}^\ell\}$  converge and for any $i\in[m]$, 
\begin{eqnarray} \label{gap-all}
0=\lim_{\ell\to\infty}\left\|  \triangle\bw^{\ell} \right\| =\lim_{\ell\to\infty} \left\|\triangle \bw_{i}^{\ell}\right\|= \lim_{\ell\to\infty} \left\| \triangle \overline{\bw}^{\ell}_i\right\|=\lim_{\ell\to\infty} \left(\widetilde{\L}^\ell-\L^\ell\right). 
\end{eqnarray}
\item[2)] Sequence  $\{(\bw^\ell,\W^\ell,\E  \P^\ell)\}$ converges.
\end{itemize}
\end{theorem} 
The proof of Theorem \ref{main-convergence} is given in Appendix \ref{app:proof-theorem31}. To ensure the convergence results,  initial value $\bsi^0$ is selected according to (\ref{choice-of-sigma}), which involves a hyperparameter $\delta$. If a lower bound $\underline{F}$ of $\min_{\bw}\sum_{i=1}^m \alpha_i F_i(\bw)$ is known, then an upper bound $\overline{\delta}$ for $\delta$ can be estimated from \eqref{compact-set} and \eqref{define-delta} by substituting $F_i(\bw)$ with $\underline{F}$. In practice, particularly in deep learning, many widely used loss functions, such as mean squared error and cross-entropy, yield non-negative values.  This observation allows us to set the lower bound as $\underline{F} = 0$. Once   $\overline{\delta}$ is estimated, it can be used in \eqref{choice-of-sigma} to select $\bsi^0$, without affecting the convergence guarantees. However, it is worth emphasizing that (\ref{choice-of-sigma}) is a sufficient but not necessary condition. Therefore, in practice, it is not essential to enforce this condition strictly when initializing $\bsi^0$ in numerical experiments.


\subsection{Complexity analysis}
Besides the convergence established above, the algorithm exhibits the following rate of convergence under the same assumption and parameter setup.
\begin{theorem}\label{main-convergence-rate-eps} Let $\{(\bw^\ell,\W^\ell,\P^\ell)\}$ be the sequence generated by Algorithm \ref{algorithm-ADMM-mini} with $\bsi^0$  chosen as (\ref{choice-of-sigma}). Let $\bw^\infty$ be the limit of sequence  $\{\bw^\ell\}$. Then there is a constant $C_1>0$ such that
\begin{eqnarray} \label{rate-w} 
\max\left\{ \left\|  {\bw}^{\ell} - \bw^\infty  \right\|, ~\left\|  {\bw}_i^{\ell} - \bw^\infty  \right\|, ~\left\|\E \bpi_{i}^{\ell} +    \nabla F_i(\bw^\infty)  \right\|,~\forall~i\in[m] \right\} \leq C_1\gamma^\ell
\end{eqnarray} 
and a constant $C_2>0$ such that
\begin{eqnarray} \label{rate-F} 
\max\left\{   F_\lambda( {\bw}^{\ell}),~  \L^{\ell}  ,~ \widetilde{\L}^{\ell}  \right\} - F_\lambda( {\bw}^{\infty})\leq C_2\gamma^\ell. 
\end{eqnarray} 
\end{theorem}
The proof of Theorem \ref{main-convergence-rate-eps} is given in Appendix \ref{app:proof-theorem32}.  This theorem means that  $\{(\bw^\ell,\W^\ell,\E\P^\ell)\}$ converges to its limit in a linear rate.   To achieve such a result, we only assume Assumption \ref{assumption} without imposing any other commonly used assumptions, such as those presented in Table \ref{tab:comp-assumps}. 
 To further see the optimality of limiting point $\bw^{\infty}$ for problem \eqref{opt-prob-distribute}, we need to specify particular $\gamma_i$. In particular, without loss of generality, choose $ \gamma_i\equiv \gamma\in[3/4,1)$ for all $i\in[m]$ to satisfy that
\begin{equation}\label{condition-gamma-K}
\lim_{T\to\infty} \frac{1-\gamma}{1-\gamma^T} = 0.
\end{equation}
This is a mild condition because given $T\geq2$,  many choices of $\gamma$ satisfy this condition, such as
\begin{align*}
\gamma&=1 - T^{-a}, \qquad\forall~ a>{\rm log}_T(4),\\
\gamma&=1 - c^{-T},\qquad~\forall~ c>4^{1/T}.
\end{align*}
\begin{theorem}\label{main-convergence-rate-gradient} Let $\{(\bw^\ell,\W^\ell,\P^\ell)\}$ be the sequence generated by Algorithm \ref{algorithm-ADMM-mini} and $T\geq2$ be the total number of iterations. Choose $\bsi^0$ to satisfy (\ref{choice-of-sigma}) and $\gamma$ to satisfy  (\ref{condition-gamma-K}). Let $\bw^\infty$ be the limit of sequence  $\{\bw^\ell\}$. Then  
\begin{eqnarray} \label{rate-F} 
\E \|\nabla F_\lambda( {\bw}^{\infty})\|^2 =0,\quad \E \|\nabla F_\lambda( {\bw}^{T})\|^2 =  O(\gamma^T). 
\end{eqnarray} 
\end{theorem}
The proof of Theorem \ref{main-convergence-rate-gradient} is given in Appendix \ref{app:proof-theorem33}.   It complements Theorems~\ref{main-convergence} and~\ref{main-convergence-rate-eps}  by providing a first-order stationarity characterization.
Specifically, the first equality shows that the limit point \( {\bw}^{\infty}\) is a stationary point in the \(L^2\)-sense for problem \eqref{opt-prob-distribute}, while the second equality further shows that the global gradient norm at \({\bw}^{T}\) decays linearly at the rate \(O(\gamma^{T})\).
Thus, the convergence of PISA is not only to a consensus limit, but also to a first-order stationary point of the global objective.

\section{Precondition Specification}\label{sec:pre-cond}
In this section, we explore the preconditioning matrix, namely, matrix $\Q_i^\ell$. A simple and computationally efficient choice is to set $\Q_i^\ell={\bf I}$,   which enables fast computation of updating  $\bw_{i}^{\ell+1}$ via \eqref{sub-wbn-mini}. However, this choice is too simple to extract useful information about $F_i$. Therefore, several alternatives can be adopted to set $\Q_i^\ell$.

\subsection{Second-order 	Information}
Second-order optimization methods, such as Newton-type and trust region methods, are known to enhance numerical performance by leveraging second-order information, the (generalized) Hessian. For instance, if each function $F_{i}$ is twice continuously differentiable, then one can set  \begin{eqnarray}\label{Q-newton}{\Q_i^{\ell+1}=\nabla^2 F_{i}(\bw^{\ell+1};\B_i^{\ell+1}),}\end{eqnarray}
 where $\nabla^2 F_{i}(\bw^{\ell+1};\B_i^{\ell+1})$ represents the Hessian of $F_{i}(\cdot;\B_i^{\ell+1})$ at $\bw^{\ell+1}$. With this choice, subproblem \eqref{sub-wbn} becomes closely related to second-order methods, and the update takes the form 
  \begin{eqnarray*}\label{sub-wbn-newton}
\begin{aligned}
\bw_{i}^{\ell+1} =\bw^{\ell+1} - \Big(\sigma_i^{\ell+1}  \I + \rho_i \nabla^2 F_{i}(\bw^{\ell+1};\B_i^{\ell+1}) \Big)^{-1}  \Big(  \bpi_{i}^{\ell} +   \bg_{i}^{\ell+1}  \Big).
\end{aligned}
\end{eqnarray*}
This update corresponds to a Levenberg-Marquardt step \cite{levenberg1944method, marquardt1963algorithm}  or a regularized Newton step \cite{li2004regularized, polyak2009regularized}  when  ${\sigma_i^{\ell+1}>0}$, and reduces to the classical Newton step if ${\sigma_i^{\ell+1}=0}$. While incorporating the Hessian can improve performance in terms of iteration count and solution quality, it often leads to significantly high computational complexity. To mitigate this, some other effective approaches exploit second-moment derived from historical updates to construct the preconditioning matrices.

\subsection{Second Moment}\label{sec:SISA}
We note that the second moment to determine an adaptive learning rate enables the improvements of the learning performance of several popular algorithms, such as RMSProp \cite{tieleman2012lecture} and   
Adam  \cite{kingma2014adam}. Motivated by this, we specify preconditioning matrix by using the second moment as follows,
\begin{equation}
 \label{fact-4} 
\Q_i^{\ell+1} = {\rm Diag}\left(\sqrt{\bm_i^{\ell+1}}\right),
\end{equation} 
where ${\rm Diag}(\bm)$ is the diagonal matrix with the diagonal entries formed by $\bm$ and $\bm_i^{\ell+1}$ can be chosen flexibly as long as it satisfies that $\|\bm_i^{\ell+1}\|_\infty\leq \eta_i^2$. Here, $\|\bm\|_\infty$ is the infinity norm of $\bm$. We can set $\bm_i^{\ell+1}$ as follows 
\begin{equation}
 \label{set-second-moment-m} \bm_i^{\ell+1} = \min\left\{  \widetilde{\bm}_i^{\ell+1}, ~ \eta_i^2  \textbf{1} \right\},
 \end{equation} 
where $\widetilde{\bm}_i^{\ell+1}$ can be updated by
\begin{equation}
\label{second_moment_update}
\begin{aligned}
&\text{Scheme I:}\qquad&&\widetilde{\bm}_i^{\ell+1}=  \widetilde{\bm}_i^{\ell} + \left( \bpi_{i}^{\ell} +   \bg_i^{\ell+1}   \right) \odot \left(  \bpi_{i}^{\ell} +   \bg_i^{\ell+1}  \right),\\
&\text{Scheme II:}\qquad&&\widetilde{\bm}_i^{\ell+1}=  \beta_i \widetilde{\bm}_i^{\ell} +  (1-\beta_i)  \left( \bpi_{i}^{\ell} +   \bg_i^{\ell+1}   \Big) \odot \Big(  \bpi_{i}^{\ell} +   \bg_i^{\ell+1}  \right),\\
&\text{Scheme III:}\qquad &&\mathbf{n}_i^{\ell+1}=  \beta_i \mathbf{n}_i^{\ell} +  (1-\beta_i)  \left( \bpi_{i}^{\ell} +   \bg_i^{\ell+1}   \Big) \odot \Big(  \bpi_{i}^{\ell} +   \bg_i^{\ell+1}  \right),\\
&&&\widetilde{\bm}_i^{\ell+1}=\mathbf{n}_i^{\ell+1}/(1- \beta_i^{\ell+1}),
\end{aligned}
\end{equation}
where  $\widetilde{\bm}_i^{0}$ and $\mathbf{n}_i^{0}$ are given,  $\beta_i\in(0,1)$,  and $\beta_i^\ell$ stands for power $\ell$ of $\beta_i$.
These three schemes resemble the ones used by  AdaGrad   \cite{duchi2011adaptive}, RMSProp   \cite{tieleman2012lecture}, and   
Adam  \cite{kingma2014adam}, respectively.    Employing \eqref{fact-4} into  Algorithm \ref{algorithm-ADMM-mini} gives rise to Algorithm \ref{algorithm-ADMM-SM}. We term it SISA, an abbreviation for the second moment-based inexact SADMM.   Compared to PISA in Algorithm \ref{algorithm-ADMM-mini}, SISA admits three advantages. 
\begin{itemize}[leftmargin=17pt]
\item[i)]
It is capable of incorporating various schemes of the second moment, which may enhance the numerical performance of SISA significantly. 
\item[ii)] One can easily check that $\eta_i \I \succeq \Q_i^{\ell+1} \succeq 0 $ for each batch $i\in[m]$ and all $\ell\geq1$. Therefore,  \eqref{fact-4} enables us to preserve the convergence property as follows.
\begin{theorem}\label{main-convergence-sm} Let $\{(\bw^\ell,\W^\ell,\P^\ell)\}$ be the sequence generated by Algorithm \ref{algorithm-ADMM-SM} with $\bsi^0$ chosen as (\ref{choice-of-sigma}). Then under Assumption \ref{assumption}, all statements in Theorems \ref{main-convergence} and \ref{main-convergence-rate-eps} are valid.
\end{theorem}
\item[iii)] Such a choice of $\Q_i^{\ell+1}$ enables the fast computation compared to update $\bw_i^{\ell+1}$ by \eqref{sub-wbn}. In fact, since operation ${\bf u}/{\bf v}$ denotes element-wise division,  the complexity of computing \eqref{wi-update-PISA-sm} is $O(p)$, where $p$ is the dimension of $\bw_i^{\ell+1}$, whereas the complexity of computing \eqref{sub-wbn-mini} is   $O(p^3)$. 
\end{itemize}

\begin{algorithm}[!t]
    \SetAlgoLined
  
Divide  $\D$ into $m$ disjoint  batches {$\{\D_1,\D_2,\ldots,\D_m\}$} and calculate $\alpha_i$ by (\ref{def-Fbn}). 

Initialize $\bw^0=\bw_i^0=\bpi_i^0=0$,  $\gamma_i\in[3/4,1),$  and $(\sigma_i^0,\eta_i,\rho_i)>0$  for each $i\in[m]$.  


\For{$\ell=0,1,2,\ldots$}{
 \begin{eqnarray*}  
\bw^{\ell+1}  =  \frac{\sum_{i=1}^m\alpha_{i} \left(\sigma_i^\ell\bw_{i}^{\ell}+\bpi_{i}^{\ell}  \right) }{\sum_{i=1}^m\alpha_{i}  \sigma_i^\ell+\lambda}.
 \end{eqnarray*}     
\For{$i=1,2,\ldots,m$}{  
 \begin{align}
 &\text{Randomly draw a mini-batch  $ {\B}_{i}^{\ell+1}\in\D_{i}$ and calculate $\bg_i^{\ell+1}= \nabla F_{i}( \bw^{\ell+1}; {\B}_{i}^{\ell+1})$}.\nonumber\\[1ex] 
  &\text{Choose $\bm_i^{\ell+1}$ to satisfy $\|\bm_i^{\ell+1}\|_\infty\leq \eta_i^2$}.~~~~~\nonumber\\[1ex] 
  &  \sigma_i^{\ell+1}  = \sigma_i^{\ell}/\gamma_i. \nonumber\\[1ex]
 \label{wi-update-PISA-sm}
 & \bw_{i}^{\ell+1} = \bw^{\ell+1} -   \frac{\bpi_{i}^{\ell} +  \bg_i^{\ell+1}}{\sigma_i^{\ell+1}  + \rho_i \sqrt{\bm_i^{\ell+1}}  }. \\ 
&\bpi_i^{\ell+1} =  \bpi_{i}^{\ell} +\sigma_i^{\ell+1}(\bw_{i}^{\ell+1}- \bw^{\ell+1}).\nonumber
\end{align}
 } 
} 
\caption{\textbf{S}econd moment-based \textbf{I}nexact \textbf{S}tochastic \textbf{A}DMM (SISA). }\label{algorithm-ADMM-SM}
\end{algorithm} 

\begin{algorithm}[!t]
    \SetAlgoLined
  
Divide  $\D$ into $m$ disjoint  batches {$\{\D_1,\D_2,\ldots,\D_m\}$} and calculate $\alpha_i$ by (\ref{def-Fbn}). 

Initialize $\bw^0=\bw_i^0=\bpi_i^0={\bf b}_i^0=0$,  $\gamma_i\in[3/4,1),  \epsilon_i\in(0,1)$, and $ (\sigma_i^0, \rho_i,\mu_i)>0$  for each $i\in[m]$.  


\For{$\ell=0,1,2,\ldots$}{
 \begin{eqnarray*}  
\bw^{\ell+1}  =  \frac{\sum_{i=1}^m\alpha_{i} \left(\sigma_i^\ell\bw_{i}^{\ell}+\bpi_{i}^{\ell}  \right) }{\sum_{i=1}^m\alpha_{i}  \sigma_i^\ell+\lambda}.
 \end{eqnarray*}     
\For{$i=1,2,\ldots,m$}{  
 \begin{align}
 &\text{Randomly draw a mini-batch  $ {\B}_{i}^{\ell+1}\in\D_{i}$ and calculate $\bg_i^{\ell+1}= \nabla F_{i}( \bw^{\ell+1}; {\B}_{i}^{\ell+1})$}.\nonumber\\[1ex] 
   &{\bf b}_i^{\ell+1}=\mu_i{\bf b}^\ell_i + \bg_i^{\ell+1}.~~~~~\nonumber\\[1ex] 
  &{\bf o}_i^{\ell+1}= \texttt{NewtonSchulz}({\bf b}^{\ell+1}_i).~~~~~\nonumber\\[1ex] 
  &  \sigma_i^{\ell+1}  = \sigma_i^{\ell}/\gamma_i. \nonumber\\[1ex]
 \label{sub-wbn-NS-m}
 & \bw_{i}^{\ell+1} =\bw^{\ell+1} - \dfrac{  \bpi_{i}^{\ell}+{\bf o}_i^{\ell+1} + \epsilon_i^{\ell+1} {\bf v}_i^{\ell+1} }{\sigma_i^{\ell+1}  + \rho_i \sqrt{ {\bf m}_i^{\ell+1}}} . \\ 
&\bpi_i^{\ell+1} =  \bpi_{i}^{\ell} +\sigma_i^{\ell+1}(\bw_{i}^{\ell+1}- \bw^{\ell+1}).\nonumber
\end{align}
 } 
} 
\caption{\textbf{N}ewton-\textbf{S}chulz-based \textbf{I}nexact \textbf{S}tochastic \textbf{A}DMM (NSISA). }\label{algorithm-ADMM-NS}
\end{algorithm} 

\subsection{Orthogonalized Momentum by Newton-Schulz Iterations}
Recently, the authors of \cite{jordan2024muon} proposed an algorithm called Muon, which orthogonalizes momentum using Newton-Schulz iterations. This approach has shown promising results in fine-tuning LLMs, outperforming many established optimizers. The underlying philosophy of Muon can also inform the design of the preconditioning matrix. Specifically, we consider the two-dimensional case, namely, the trainable variable $\bw$ is a matrix. Then subproblem \eqref{sub-wbn} in a vector form turns to  
  \begin{eqnarray}\label{sub-wbn-NS}
\begin{aligned}
{\rm vec}(\bw_{i}^{\ell+1}) ={\rm vec}(\bw^{\ell+1}) - \Big(\sigma_i^{\ell+1}  \I + \rho_i \Q_i^{\ell+1} \Big)^{-1}  \Big(  {\rm vec}(\bpi_{i}^{\ell}) + {\rm vec}(\bg_i^{\ell+1})  \Big),
\end{aligned}
\end{eqnarray}
where ${\rm vec}(\bw)$ denotes the column-wise vectorization of  matrix $\bw$. Now, initialize ${\bf b}^0_i$ for all ${i\in[m]}$ and ${\mu>0}$, update momentum by ${{\bf b}_i^{\ell+1}=\mu{\bf b}^\ell_i + \bg_i^{\ell+1}}$. Let ${\bf b}_i^{\ell+1} ={\bf U}_i^{\ell+1}{\boldsymbol \Lambda}_i^{\ell+1}({\bf V}_i^{\ell+1})^\top$ be the singular value decomposition of ${\bf b}_i^{\ell+1}$, where ${\boldsymbol \Lambda}_i^{\ell+1}$ is an diagonal matrix and ${\bf U}_i^{\ell+1}$ and ${\bf V}_i^{\ell+1}$ are two orthogonal matrices. Compute ${\bf o}_i^{\ell+1} = {\bf U}_i^{\ell+1} ({\bf V}_i^{\ell+1})^\top$ and ${\bf p}_i^{\ell+1}$  by
  \begin{eqnarray*}\label{sub-wbn-NS-P}
  \begin{aligned}
{\bf p}_i^\ell =
 \dfrac{\Big(\sigma_i^{\ell+1} + \rho_i \sqrt{{\bf m}_i^{\ell+1}}\Big)\odot \Big(\bpi_{i}^{\ell}+{\bf g}_i^{\ell+1} \Big) }{\rho_i\Big( \bpi_{i}^{\ell}+{\bf o}_i^{\ell+1} + \epsilon_i^{\ell+1}{\bf v}_i^{\ell+1} \Big) }-\dfrac{\sigma_i^{\ell+1}}{\rho_i}, 
\end{aligned}
\end{eqnarray*}
where ${\bf m}_i^{\ell+1}$ can be the second moment (e.g., ${\bf m}_i^{\ell+1}=(\bpi_{i}^{\ell}+{\bf o}_i^{\ell+1})\odot(\bpi_{i}^{\ell}+{\bf o}_i^{\ell+1})$ is used in the numerical experiment), ${\epsilon_i\in(0,1)}$ (here, $\epsilon_i^\ell$ stands for power $\ell$ of $\epsilon_i$), and ${\bf v}_i^{\ell+1}$ is a matrix with $(k,j)$th element computed by  $({\bf v}_i^{\ell+1})_{kj}=1$ if $(\bpi_{i}^{\ell}+{\bf o}_i^{\ell+1})_{kj}= 0$ and $({\bf v}_i^{\ell+1})_{kj}=0$ otherwise. 
Then we set the preconditioning matrix by $$\Q_i^{\ell+1}={\rm Diag}\left({\rm vec}({\bf p}_i^{\ell+1})\right).$$
Substituting above choice into \eqref{sub-wbn-NS} derives \eqref{sub-wbn-NS-m}. 
The idea of using \eqref{sub-wbn-NS-m} is inspired by \cite{jordan2024muon}, where Newton–Schulz orthogonalization \cite{kovarik1970some, bjorck1971iterative} is employed to efficiently approximate ${\bf o}_i^{\ell+1}$. Incorporating these steps into Algorithm \ref{algorithm-ADMM-mini} leads to Algorithm \ref{algorithm-ADMM-NS}, which we refer to as NSISA, short for Newton-Schulz-based Inexact SADMM. The implementation of \texttt{NewtonSchulz}(${\bf b}$) is provided in \cite{jordan2024muon}. Below is the convergence result of NSISA.
\begin{theorem}\label{main-convergence-ns} Let $\{(\bw^\ell,\W^\ell,\P^\ell)\}$ be the sequence generated by Algorithm \ref{algorithm-ADMM-NS} with $\bsi^0$ chosen as (\ref{choice-of-sigma}). If Assumption \ref{assumption} holds and $({\bf p}_i^\ell)_{kj}\in(0,\eta_i)$ for any $(k,j)$ and $\ell\geq0$, all statements in Theorems \ref{main-convergence} and \ref{main-convergence-rate-eps} are valid.
\end{theorem}

 \section{Numerical Experiments}
\label{sec:numerical}


This section evaluates the performance of SISA and NSISA in comparison with various established optimizers to process several deep models: VMs, LLMs, RLMs, GANs, and RNNs.  All LLM-targeted experiments are conducted on four NVIDIA H100-80GB GPUs, while the remaining experiments are run on a single such GPU. The source codes are available at \url{https://github.com/Tracy-Wang7/PISA}. Due to space limitations, details on hyperparameter setup,  effects of key hyperparameters, and experiments on RLMs and RNNs are provided in Appendix  \ref{appedix-A}.  

\begin{table}[!b]
\renewcommand{\arraystretch}{1.1}\addtolength{\tabcolsep}{-1.5pt}
\begin{center}
\caption{Testing accuracy under data heterogeneity with skewed label distributions.} 
\label{table:heterogeneity}
\begin{tabular}{lcccccc}
\hline
Dataset& Partition      & Fedavg& Fedprox&Fednova&Scaffold& SISA \\ \hline
\multirow{3}{*}{MNIST}&1-Label & $53.19 \pm 0.21$ & $54.33 \pm 0.67$& $49.00 \pm 0.42$& $47.04 \pm 0.02$ & $94.97 \pm 0.01$\\ 
&2-Label                & $90.70 \pm 0.02$& $90.75 \pm 0.03$& $90.40 \pm 0.05$& $87.43 \pm 0.08$ & $96.14 \pm 0.00$ \\ 
&3-Label                & $95.11 \pm 0.01$& $94.73 \pm 0.02$& $94.37 \pm 0.01$& $93.19 \pm 0.03$ & $96.21 \pm 0.01$ \\\hline
\multirow{3}{*}{CIFAR10}&1-Label & $10.02 \pm 0.00$ & $10.00 \pm 0.00$& $10.31 \pm 0.03$& $11.41 \pm 0.03$ & $30.16 \pm 0.00$\\ 
&2-Label                & $13.76 \pm 0.00$& $13.77 \pm 0.00$& $18.02 \pm 0.02$& $12.73 \pm 0.00$ & $30.53 \pm 0.00$ \\ 
&3-Label                & $20.06 \pm 0.01$& $20.03 \pm 0.01$& $21.20 \pm 0.01$& $17.00 \pm 0.00$ & $30.96 \pm 0.00$ \\\hline
\multirow{3}{*}{FMNIST}&1-Label & $48.81 \pm 0.34$ & $47.14 \pm 0.32$& $49.08 \pm 0.34$& $43.93 \pm 0.20$ & $72.85 \pm 0.02$\\ 
&2-Label                & $69.33 \pm 0.01$& $69.29 \pm 0.02$& $69.10 \pm 0.04$& $65.99 \pm 0.06$ & $70.24 \pm 0.05$ \\ 
&3-Label                & $70.48 \pm 0.02$& $69.96 \pm 0.08$& $69.70 \pm 0.08$& $63.64 \pm 0.15$ & $72.46 \pm 0.01$ \\\hline
Adult &1-Label & $76.38^* \pm 0.00$ & $78.30^* \pm 0.03$ & $85.15^* \pm 0.08$ & $80.00^* \pm 0.02$ & $82.21 \pm 0.00$ \\\hline
\end{tabular}
\begin{tablenotes}
$*$ indicates mini-batch training in each local client. {Higher} values indicate better performance.
\end{tablenotes}  
\end{center}
\end{table}
\subsection{Data Heterogeneity}
To evaluate the effectiveness of the proposed algorithms in addressing the challenge of data heterogeneity, we design some experiments within a centralized federated learning (CFL) framework. 
In this setting, $m$ clients collaboratively learn a shared parameter under a central server, with each client $i$ holding a local dataset $\D_i$. Note that our three algorithms are well-suited for CFL tasks.  As shown in Algorithm \ref{algorithm-ADMM-SM}, clients update their local parameters $(\bw_{i}^{\ell+1}, \bpi_i^{\ell+1})$ in parallel, and the server aggregates them to compute global parameter $\bw^{\ell+2}$.

In the experiment, we focus on heterogeneous datasets $\D_1, \D_2,\ldots,\D_m$. The heterogeneity may arise from label distribution skew, feature distribution skew, or quantity skew  \cite{li2022federated}. Empirical evidence has shown that the label distribution skew presents a significantly greater challenge to current FL methods compared to the other two types. Therefore, we adopt this skew as the primary setting for evaluating data heterogeneity in our experiments.

To reflect the fact that most benchmark datasets contain $10$ image classes, we set ${m=10}$.  Four datasets: MNIST, CIFAR10, FMNIST, and Adult are used for the experiment. As shown in Table \ref{table:heterogeneity}, the configuration `$s$-Label' indicates that each client holds data containing $s$ distinct label classes, where $s=1,2,3$. For example, in the `2-Label' setting,  $\D_1$ may include two labels 1 and 2, while $\D_2$ contains two labels 2 and 3.    To ensure better performance, we employed a mini-batch training strategy with multiple local updates before each aggregation step for the five baseline algorithms. In contrast, SISA keeps using one local update per aggregation step and still achieves competitive testing accuracy with far fewer local updates. Notably, on the MNIST dataset under the 1-Label skew setting, the best accuracy achieved by other algorithms is 54.33\%, whereas SISA reaches 94.97\%, demonstrating a significant improvement.

 
%

\begin{figure}[!b]
\centering
\begin{subfigure}{.495 \textwidth}
	\centering
	\includegraphics[width=.975\linewidth]{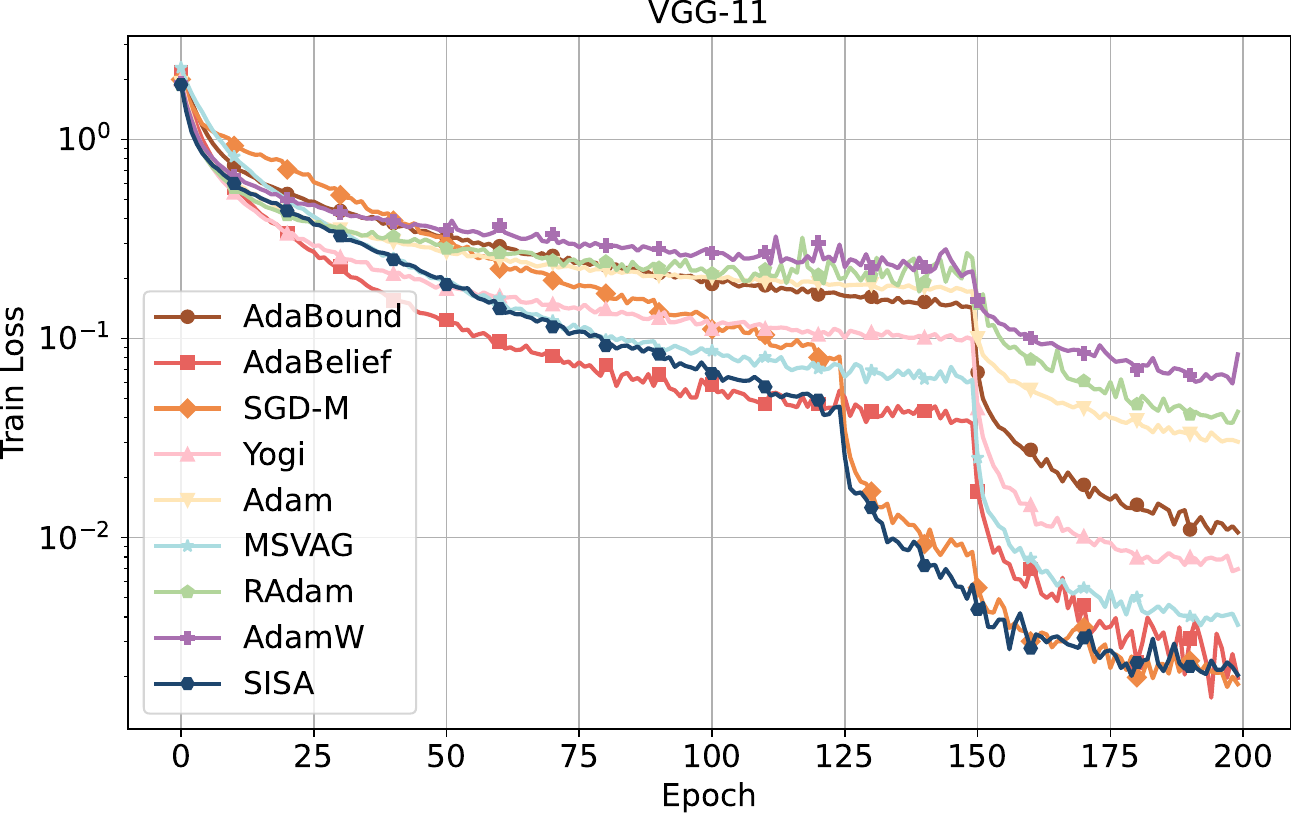} 
\end{subfigure}	 
\begin{subfigure}{.495 \textwidth}
	\centering
	\includegraphics[width=.95\linewidth]{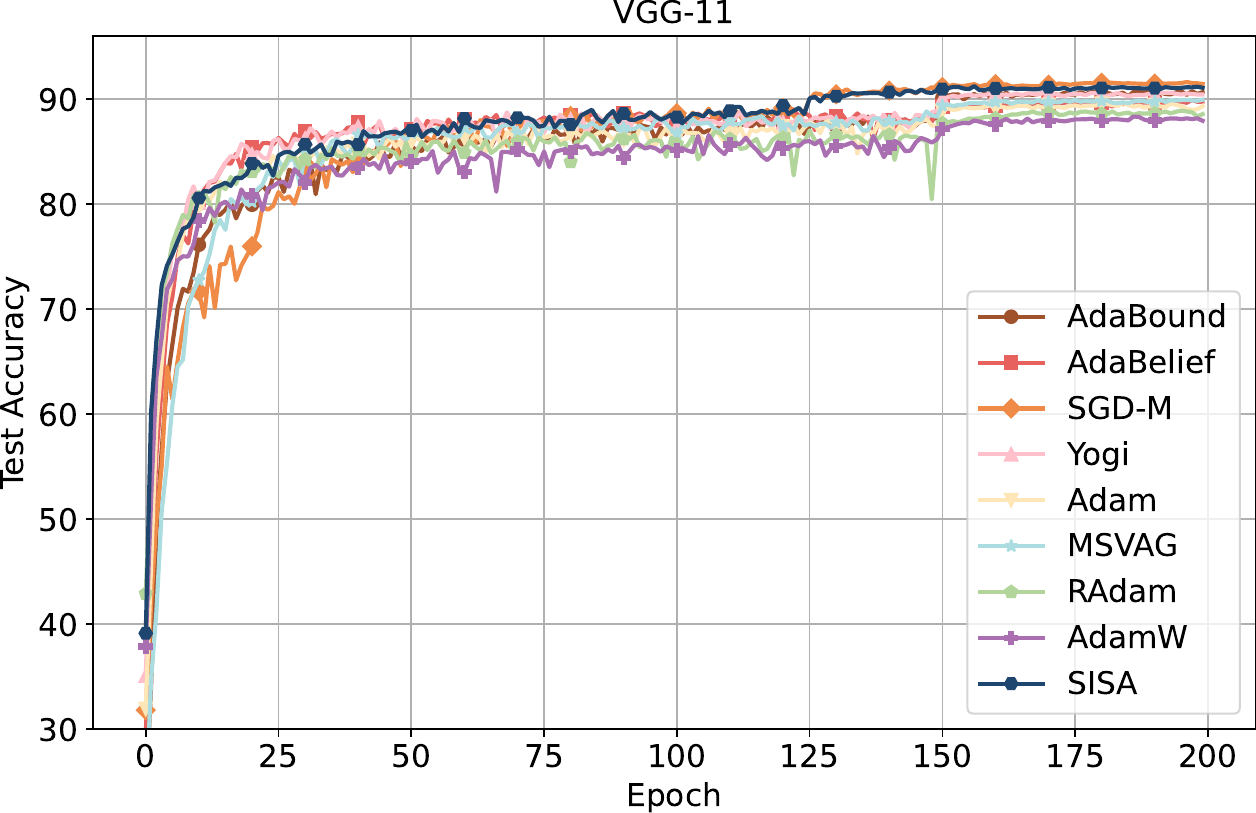} 
\end{subfigure}   \\[2ex]
\begin{subfigure}{.495 \textwidth}
	\centering
	\includegraphics[width=.975\linewidth]{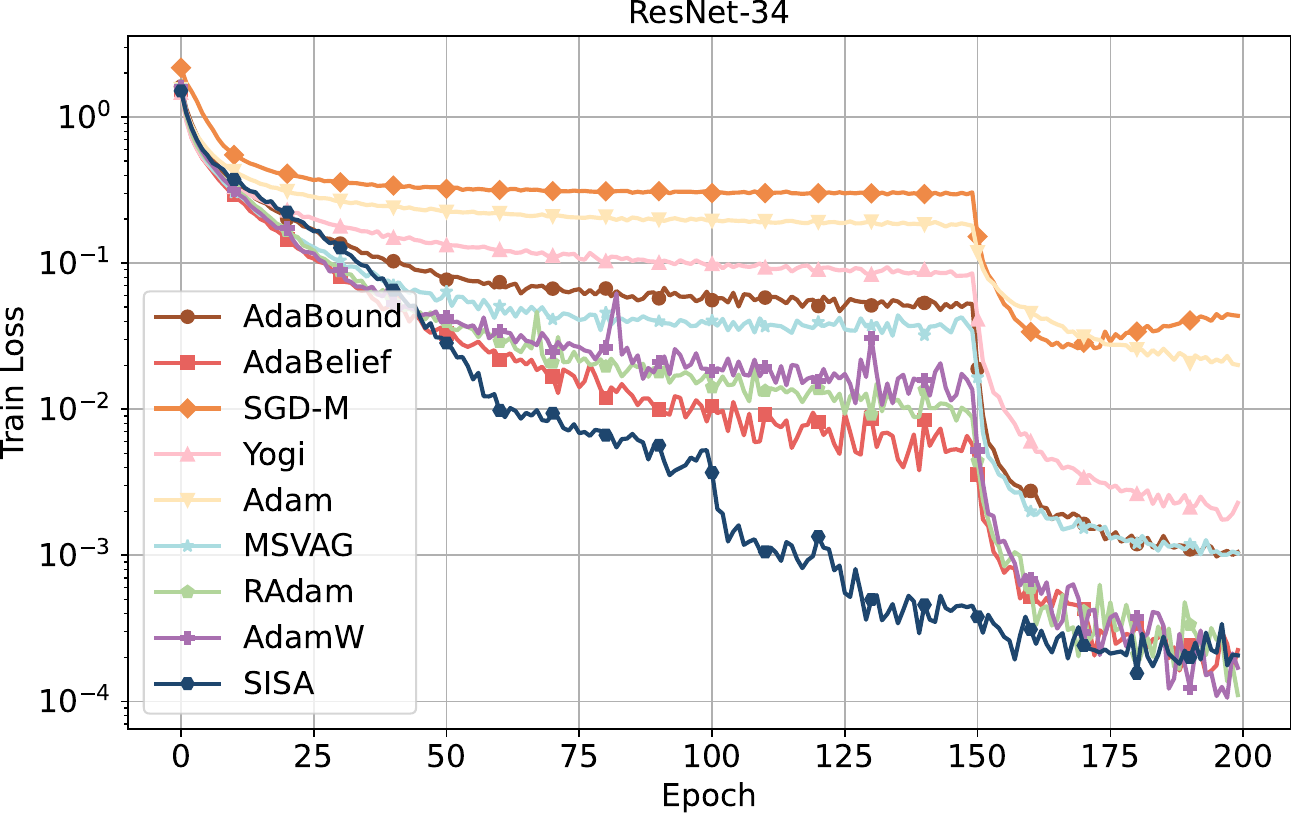} 
\end{subfigure}   
\begin{subfigure}{.495 \textwidth}
	\centering
	\includegraphics[width=.95\linewidth]{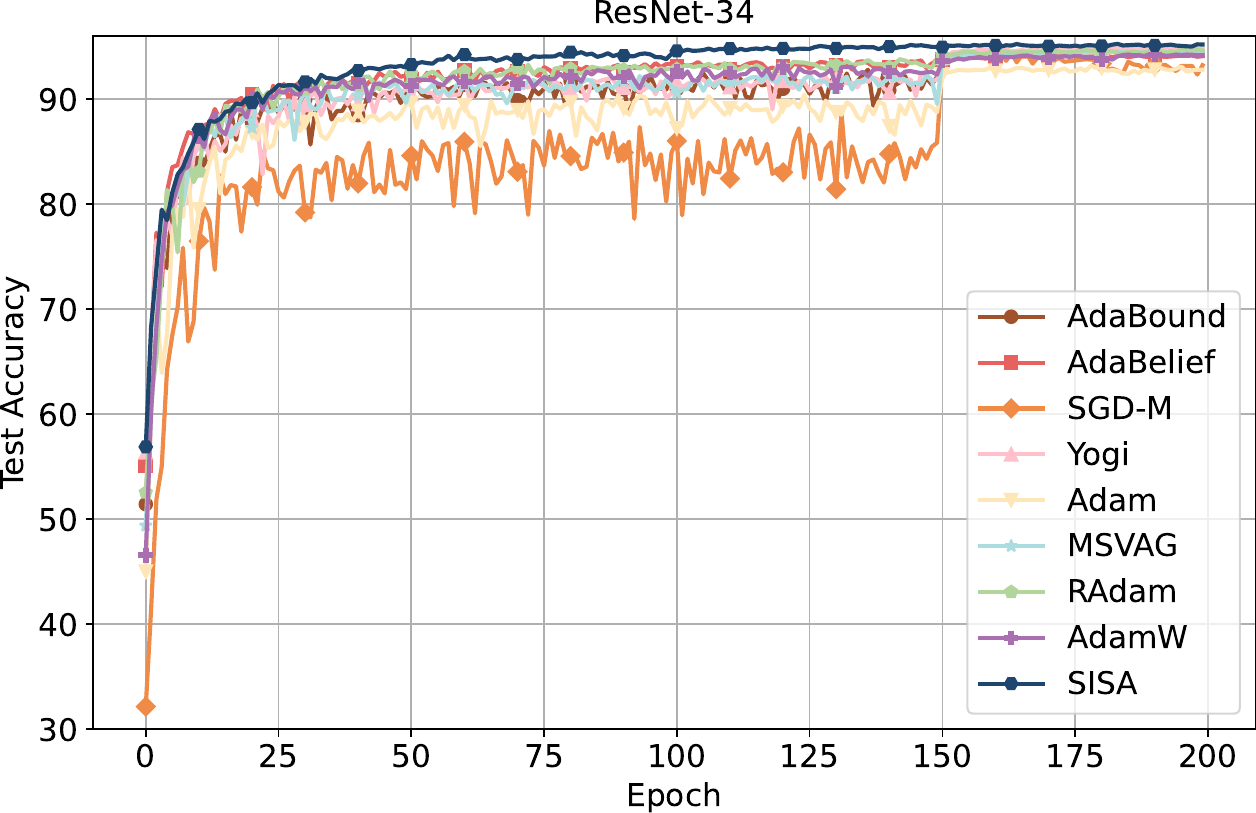} 
\end{subfigure} 
\caption{
Training loss and testing accuracy over epochs for VMs on Cifar-10.\label{fig:cifar}}
\end{figure}

\subsection{Classification by VMs}
We perform classification tasks on two well-known datasets: ImageNet \cite{krizhevsky2012imagenet} and CIFAR-10 \cite{krizhevsky2010cifar}, and evaluate performance based on convergence speed and testing accuracy. Convergence results are presented in Appendix \ref{app:class-VM}. We now focus on testing accuracy, following the experimental setup and hyperparameter configurations from \cite{zhuang2020adabelief}. In this  experiment, a small batch size (e.g., $|\D_i|=128$ for every $i\in[m]$) is used during training.  As SISA updates its hyperparameters frequently,   it uses a different learning rate scheduler from that in \cite{zhuang2020adabelief}. For fair comparison, each baseline optimizer is tested with both its default scheduler and the one employed by SISA, and the best result is reported. On CIFAR-10, we train three models: VGG-11 \cite{simonyan2014very}, ResNet-34 \cite{he2016identity}, and DenseNet-121 \cite{huang2017densely}. As shown in Table \ref{table:cifar_sota}, SISA achieves the highest testing accuracy on ResNet-34 and DenseNet-121, and performs slightly worse than SGD-M on VGG-11.
Additionally, Figure \ref{fig:cifar} illustrates the training loss and testing accuracy over training epochs for VGG-11 and ResNet-34, demonstrating the fast convergence and high accuracy of SISA. 
For ImageNet, we train a ResNet-18 model and report the testing accuracy in Table \ref{table:cifar_sota}. SISA outperforms most adaptive methods but is slightly behind SGD-M and AdaBelief.

 \begin{table}[!th]
 \renewcommand{\arraystretch}{1.0}\addtolength{\tabcolsep}{-4.7pt}
\begin{center}
\caption{Top-1 Acc. (\%) obtained by different algorithms.} 
\label{table:cifar_sota}
\begin{tabular}{lcccccccccccc}
\hline 
Algs.  & SISA & SGD-M & AdaBd & Adam & Radam & MSVAG & AdamW & AdaBf & Yogi &{LAMB}&Nadam\\ 
 Refs. &   & \cite{robbins1951stochastic} & \cite{luo2019adaptive} &  \cite{kingma2014adam} &  \cite{liu2019variance} &  \cite{balles2018dissecting} &  \cite{loshchilov2017decoupled} &  \cite{zhuang2020adabelief} & \cite{zaheer2018adaptive} &\cite{You2020Large} &\cite{dozat2016incorporating} \\
\hline
ResNet-34 & {95.04}& $94.65^*$ & $94.33^\star$ &$94.69^*$&$94.20$ &$94.44^\star$ &$94.28^*$ &$94.11^*$ & $94.52^*$ &$94.01^*$ &--\\
VGG-11 &91.25&  {91.51} & $90.62^\star$ &$88.40^\star$&$89.30^\star$ &$90.24^\star$ &$89.39^\star$ &90.07 & $90.67^\star$& 87.48&--\\
DenseNet-121 & {95.77}& 95.47 & $94.58^\star$ &$93.35^\star$&$94.91^\star$ &$94.81^\star$ &$94.55^\star$ &94.78 & 95.15 &94.53&--\\
ResNet-18&70.03& ${70.23}^\star$ & $68.13^\star$ &$63.79^\star$&$67.62^\star$ &$65.99^\star$ &$67.93^\star$ &$70.08^\star$ &--& $68.46^*$ & $68.82^*$\\
\hline
\end{tabular}
\begin{tablenotes}
$*$ and $\star$ are reported in \cite{Adan24} and \cite{zhuang2020adabelief}, and the remaining are derived by our implementation. {Higher} values indicate better performance. `--' means no results were reported. AdaBd and AdaBf stand for AdaBound and AdaBelief, respectively.
\end{tablenotes} \vspace{-5mm}
\end{center}
\end{table}

\subsection{Training LLMs}
In this subsection, we apply SISA and NSISA to train several GPT2 models \cite{radford2019language} and compare it with advanced optimizers like  AdamW, Muon, Shampoo, SOAP, and Adam-mini that have been proven to be effective in training LLM. The comparison between SISA and AdamW is provided in Appendix \ref{app:LLMs}.   We now compare NSISA with the other four optimizers, with the aim of tuning three  GPT2 models: GPT2-Nano (125M), GPT2-Medium (330M), and GPT2-XL (1.5B). The first model follows the setup described \cite{jordan2024muon}, while the latter two follow the configuration from \cite{zhang2025adammini}. We use the FineWeb dataset, a large collection of high-quality web text, to fine-tune GPT2 models, enabling them to generate more coherent and contextually relevant web-style content. The hyperparameters of NSISA remain consistent across the experiments for the three GPT2 models.

The comparison of different algorithms in terms of the memory overhead is given in Appendix \ref{app:LLMs}. The experimental results are presented in Figure \ref{fig:gpt_new}. We evaluate each optimizer's training performance from two perspectives: the number of training tokens consumed and the wall-clock time required. To reduce wall-clock time, we implemented NSISA using parallel computation. From Figure \ref{fig:gpt_new}, as the number of GPT2 parameters increases from Nano to Medium to XL, the validation loss gap between NSISA and other baseline optimizers widens. In particular, for GPT2-XL, the last figure demonstrated evident advantage of NSISA in terms of the wall-clock time.

\begin{figure}[!t]
\centering
\begin{subfigure}{.495 \textwidth}
	\centering
	\includegraphics[width=.95\linewidth]{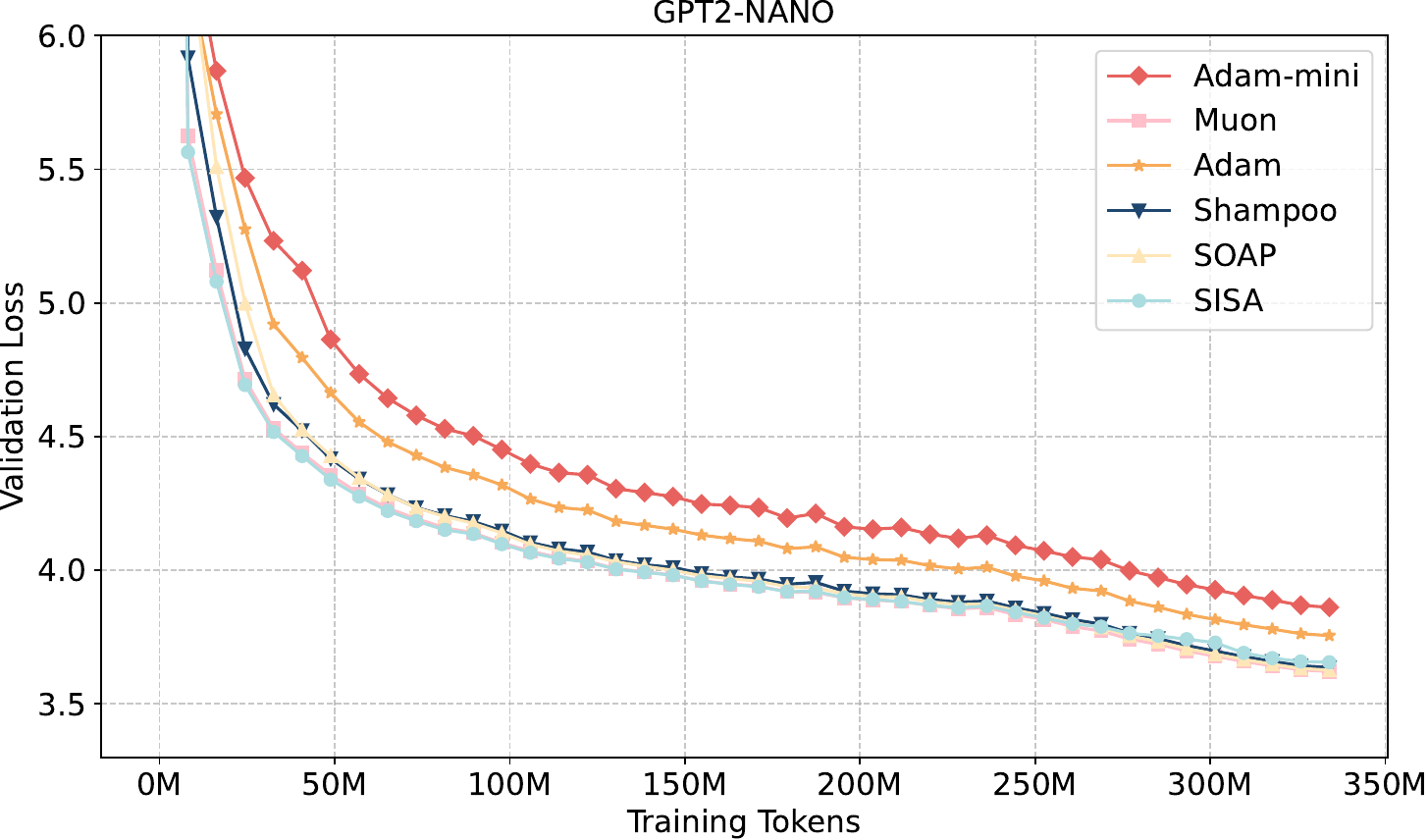} 
\end{subfigure}	 
\begin{subfigure}{.495 \textwidth}
	\centering
	\includegraphics[width=.95\linewidth]{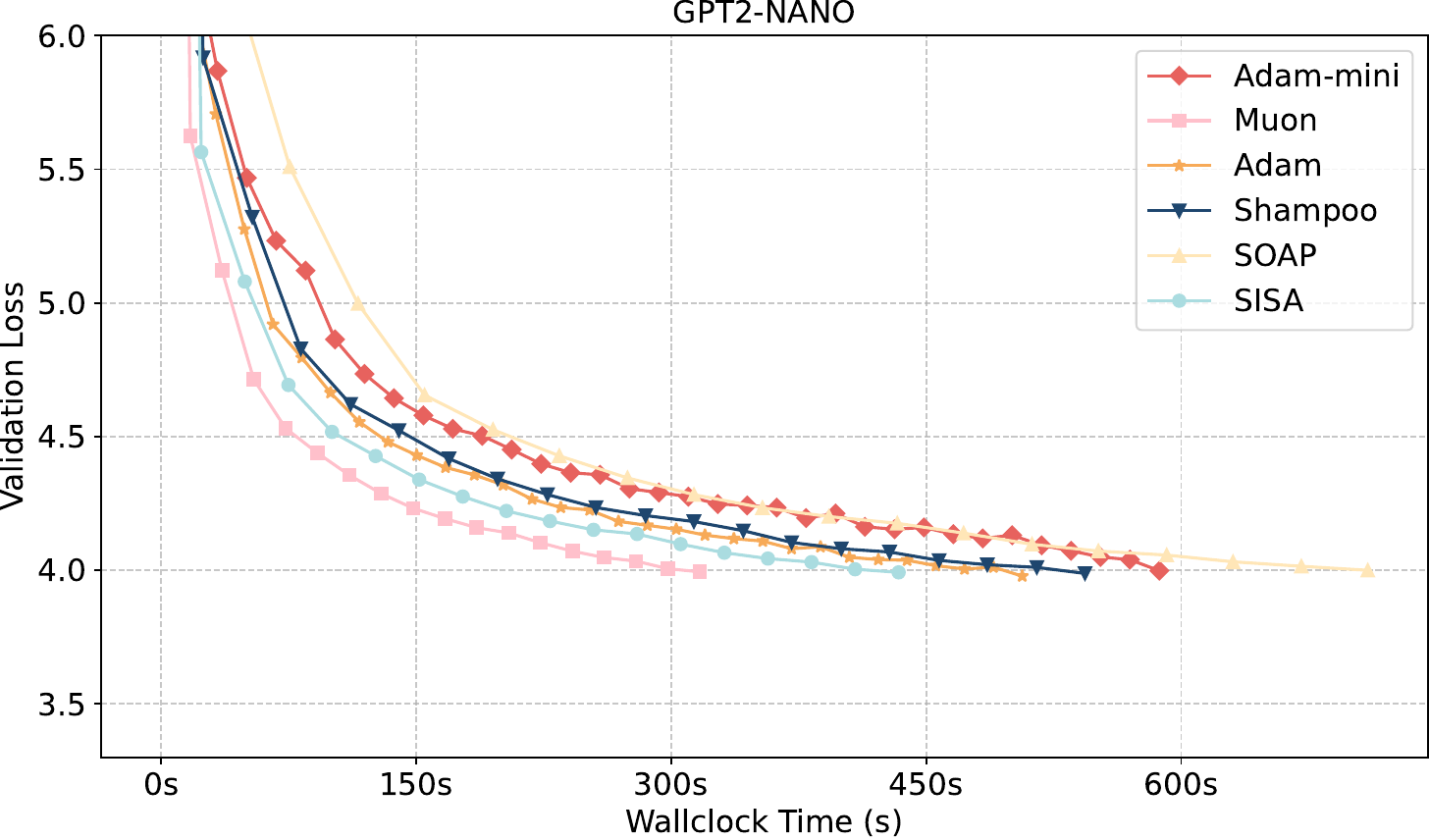} 
\end{subfigure}   \\[2ex]

\begin{subfigure}{.495 \textwidth}
	\centering
	\includegraphics[width=.95\linewidth]{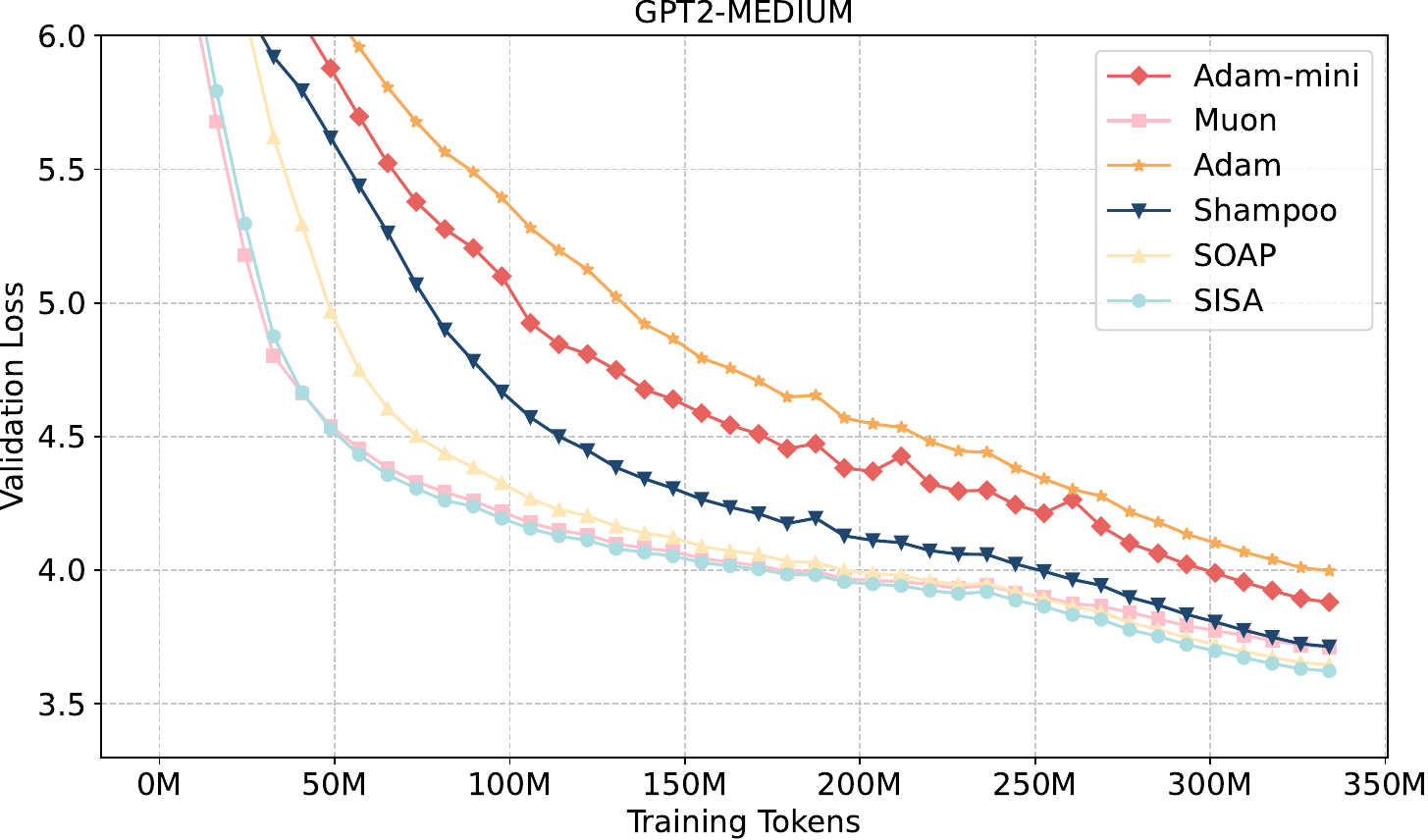} 
\end{subfigure}	 
\begin{subfigure}{.495 \textwidth}
	\centering
	\includegraphics[width=.95\linewidth]{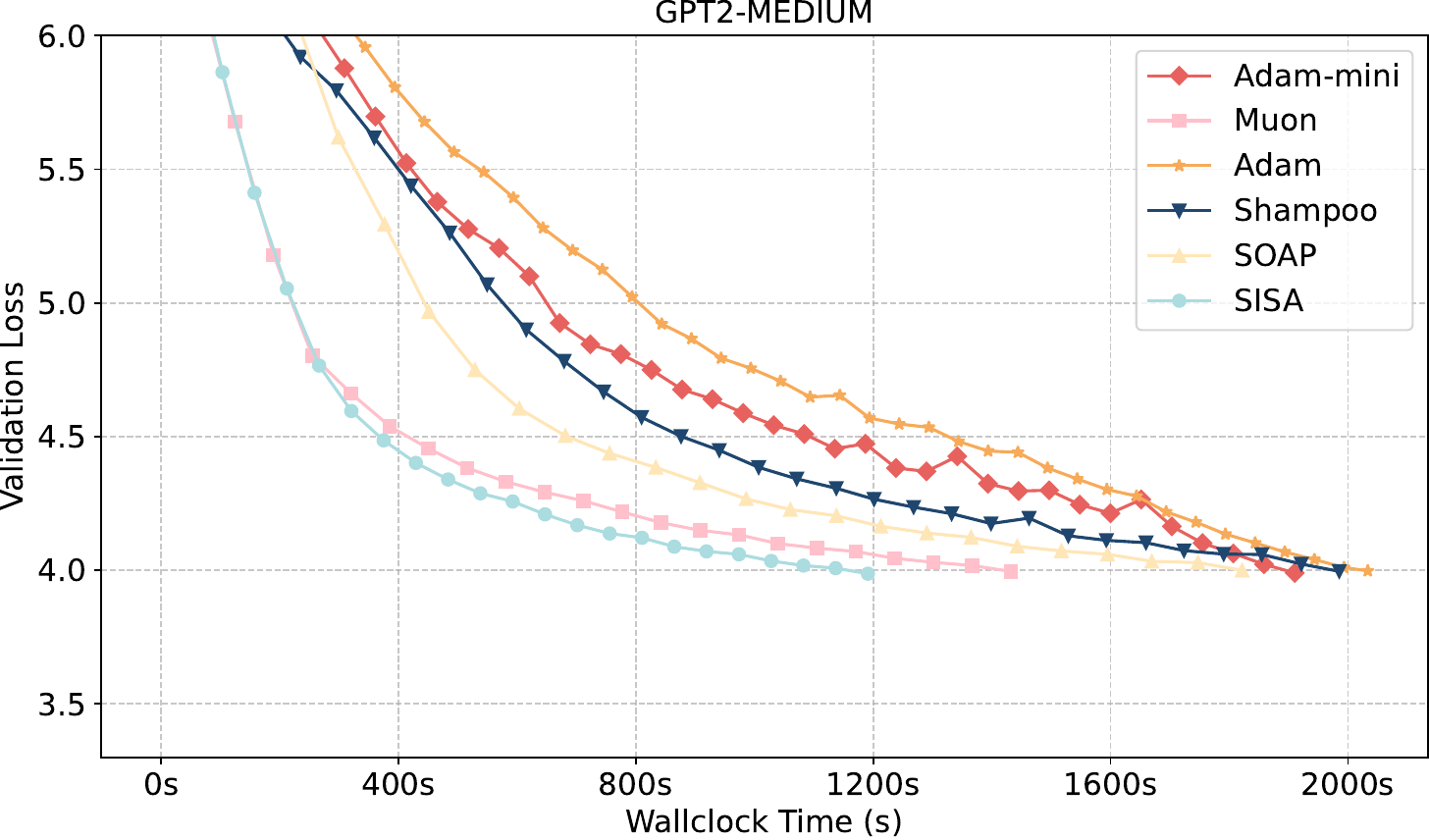} 
\end{subfigure}   \\[2ex]

\begin{subfigure}{.495 \textwidth}
	\centering
	\includegraphics[width=.95\linewidth]{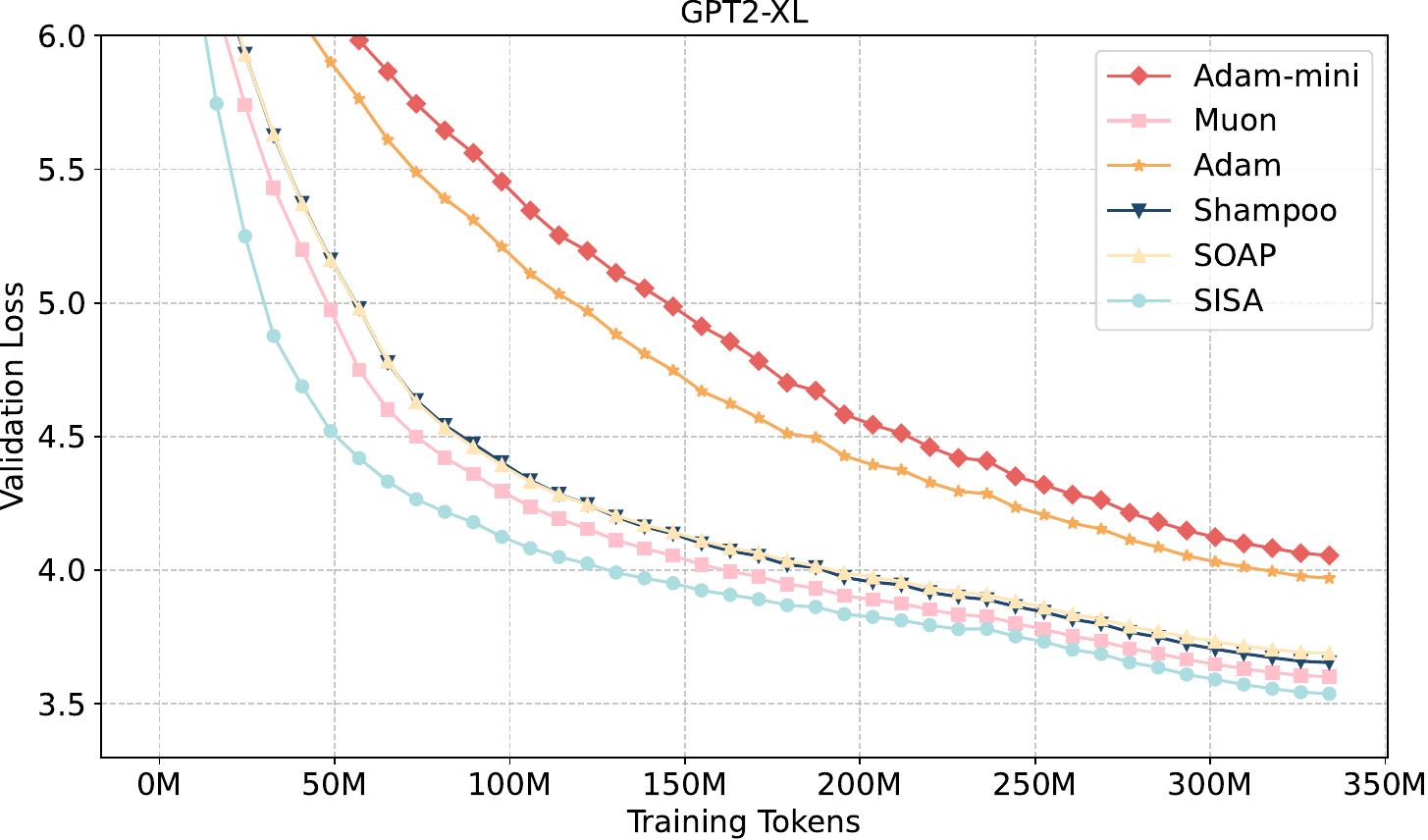} 
\end{subfigure}   
\begin{subfigure}{.495 \textwidth}
	\centering
	\includegraphics[width=.95\linewidth]{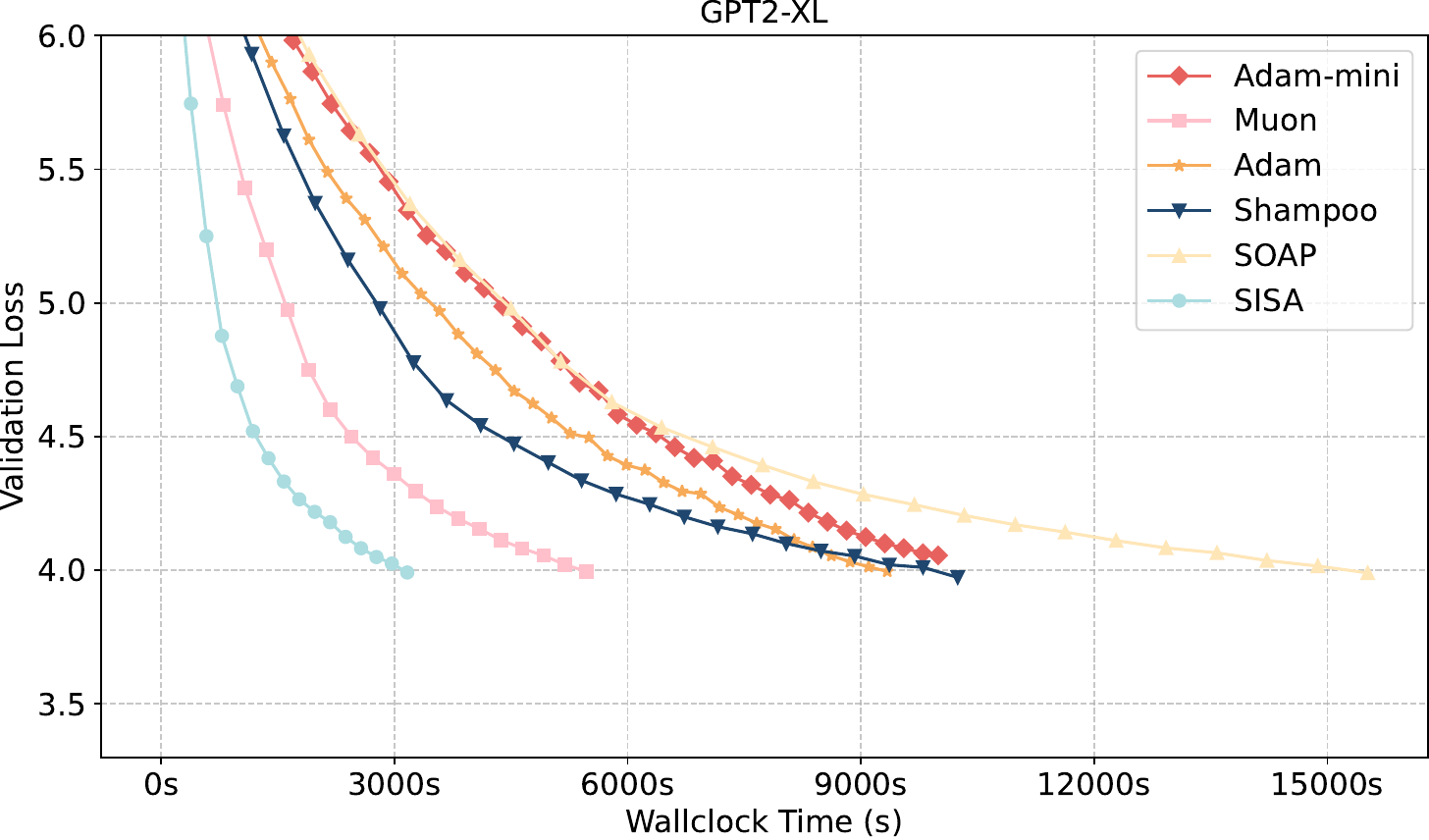} 
\end{subfigure} 
\caption{
Validation loss over tokens and wall-clock time (in seconds) for three GPT2 models. \label{fig:gpt_new}}
\end{figure}

\subsection{Generation tasks by GANs}
 GANs involve alternating updates between the generator and discriminator in a minimax game, making training notoriously unstable \cite{goodfellow2014generative}. To assess optimizer stability, we compare SISA with adaptive methods such as Adam, RMSProp, and AdaBelief, which are commonly recommended to  enhance both stability and efficiency in training GANs \cite{salimans2016improved, zhuang2020adabelief}. We use two popular architectures, the Wasserstein GAN (WGAN \cite{arjovsky2017wasserstein}) and  WGAN with gradient penalty (WGAN-GP \cite{salimans2016improved}), with a small model and vanilla CNN generator on the CIFAR-10 dataset.

 To evaluate performance, we compute the Fréchet Inception Distance (FID) \cite{heusel2017gans} every 10 training epochs, measuring the distance between 6,400 generated images and 60,000 real images. The FID score is a widely used metric for generative models, assessing both image quality and diversity. The lower FID values indicate better generated images. We follow the experimental setup from \cite{zhuang2020adabelief} and run each optimizer five times independently. The corresponding mean and standard deviation of training FID over epochs are presented in Figure \ref{fig:gan}. Clearly, SISA achieves the fastest convergence and the lowest FID. After training,  64,000 fake images are produced to compute the testing FID in Table \ref{table:wgan_sota}. Once again, SISA outperforms the other three optimizers, demonstrating its superior stability and effectiveness in training GANs.

\begin{figure}[!t]
\centering
\begin{subfigure}{.495 \textwidth}
	\centering
	\includegraphics[width=.95\linewidth]{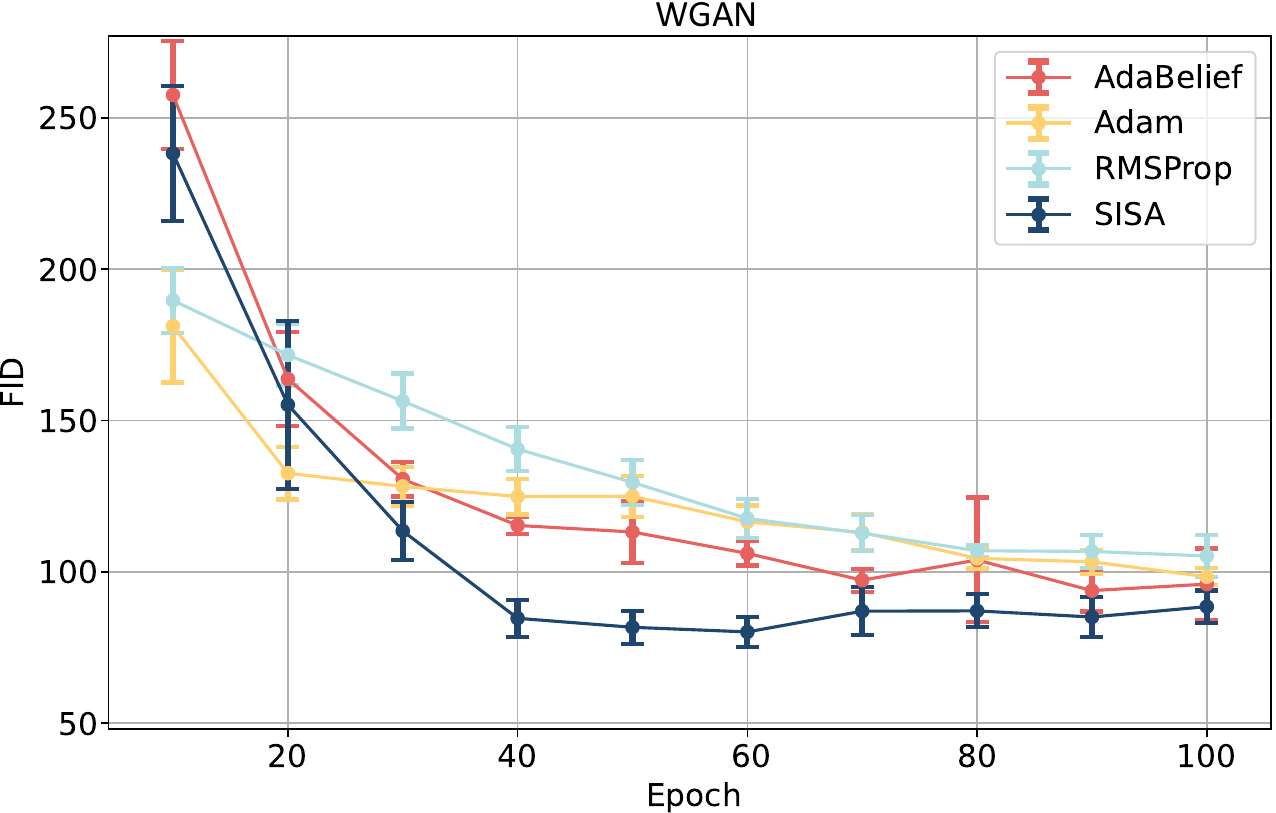} 
\end{subfigure}	 
\begin{subfigure}{.495 \textwidth}
	\centering
	\includegraphics[width=.95\linewidth]{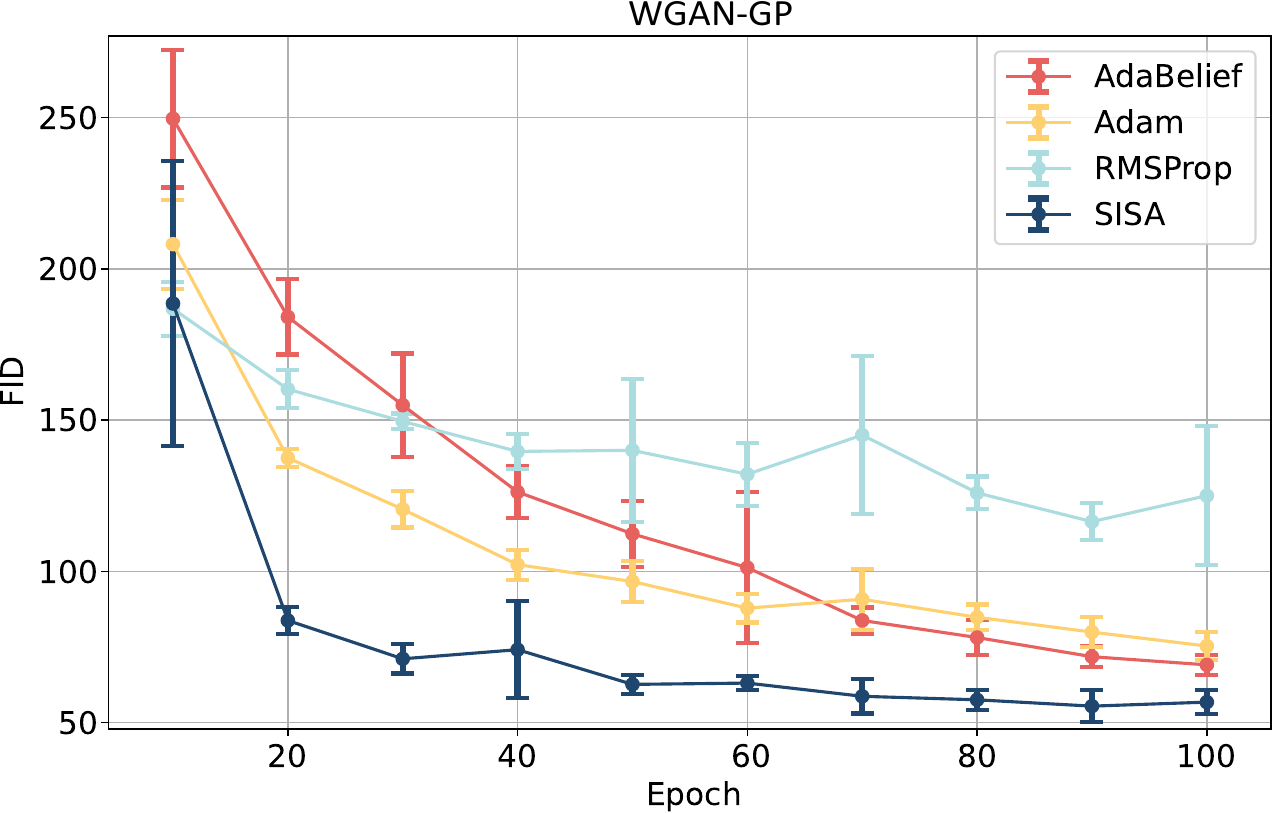} 
\end{subfigure}   
\caption{Training FID over epochs for GANs on Cifar-10. \label{fig:gan}}
\end{figure}

\begin{table}[!t]
\renewcommand{\arraystretch}{1.0}\addtolength{\tabcolsep}{10.0pt}
\begin{center}
\caption{Testing FID after training GANs on Cifar-10. }
\label{table:wgan_sota}
\begin{tabular}{ccccc}
\hline 
 & SISA & Adam & RMSProp& AdaBelief\\\hline 
WGAN &$ {85.07}\pm {5.49}$ & $95.06\pm2.97$ & $102.39\pm7.05$ &$92.37\pm11.13$\\
WGAN\_GP &$ {53.34}\pm {3.87}$ & $71.38\pm5.04$ & $121.62\pm22.79$ &$65.71\pm3.36$\\
\hline
\end{tabular}
\end{center} 
\end{table}

\section{Conclusion}
This paper introduces an ADMM-based algorithm that leverages stochastic gradients and solves sub-problems inexactly, significantly accelerating computation. The preconditioning framework allows it to incorporate popular second-moment schemes, enhancing training performance. Theoretical guarantees, based solely on the Lipschitz continuity of the gradients, make the algorithm suitable for heterogeneous datasets, effectively addressing an open problem in stochastic optimization for distributed learning. The algorithm demonstrates superior generalization across a wide range of architectures, datasets, and tasks, outperforming well-known deep learning optimizers.

\bibliographystyle{abbrv}
\bibliography{references}

\newpage
\appendix

\section{Appendix: Extended Data}\label{appedix-A}

\subsection{Hyperparameters of All Experiments}
\begin{table}[!th]
\caption{Hyperparameters of SISA and NSISA in the numerical experiments.}\vspace{-5mm}
\begin{center}
\renewcommand{\arraystretch}{1.1}\addtolength{\tabcolsep}{-0.25pt}
\begin{tabular}{llccccccccc}
\hline
\multirow{2}{*}{} & model &$\sigma_i^0$ & $\rho_i$ & $\gamma_i$ & $\mu$ & $\lambda$ & $E$ &$m$ &$n$ &$|\D_{ij}|$\\ \hline
&ResNet-18& 0.005 & 2e3/3 & 1 & 0 & 0 & 303 &4&4&8e$4$ \\ 
             VMs        &ViT-B-16& 0.005 & 2e4 & 1 & 0 & 0 & 380 &4&4&8e$4$ \\ 
             Training        &AlexNet& 0.050 & 2e4 & 1 & 0 & 0 & 380 &4&4&8e$4$ \\
                     &DenseNet-121& 0.005 & 2e3/3 & 1 & 0 & 0 & 303 &4&4&8e$4$ \\\hline

&ResNet-34& 0.10 & 5e3 & $\gamma_\ell(0.80,15)^*$ & 5e-5 & 0 & 200&4&390&32 \\ 
             VMs        &VGG-11& 4.50 & 3e4 & $\gamma_\ell(0.90,10)$ & 2.5e-4 & 4e-4 & 200&4&390&32 \\ 
              Testing       &DenseNet-121& 0.8 & 2e3 & $\gamma_\ell(0.25,50)$ & 5e-4 & 5e-5 & 200&4&390&64 \\
                     &ResNet-18& 1.00 & 2e3 & 1 & 1e-2 & 5e-6 & 100 &4&5000&64\\\hline
\multirow{3}{*}{LLMs} &GPT(124M)& 0.01 & 2.5e4 & $\gamma_\ell^{\star}$ & 0 & 0 & 80 &4&4&$62$ \\
                      &GPT(335M)& 0.01 & 1e4 & $\gamma_\ell$ & 0 & 0 & 150 &4&4&62 \\ 
                      &GPT2$^{\dagger}$& 8e1 & 1e2 & $\gamma_\ell$ & 0 & 0 & - &4&-&16 \\ \hline
\multirow{4}{*}{RLMs} &Ant& 5.0 & 2e4 & $\gamma_\ell$ & 0 & 0 & 100 &4&112&64 \\
                      &Humanoid& 0.1 & 1e5 & $\gamma_\ell$ & 0 & 0 & 100 &4&112&64 \\ 
                      &Half-Cheetah& 0.5 & 2e2 & $\gamma_\ell$ & 0 & 0 & 100 &4&117&64 \\
                      &Pendulum& 0.5 & 2e2 & $\gamma_\ell$ & 0 & 0 & 100 &4&117&64\\\hline
\multirow{2}{*}{GANs} &WGAN& 1e-1 & 1e3 & 1 & 0 & 5e-5 & 100 &4&937&16 \\  
                     &WGAN-GP& 5e-5 & 2e2 & 1 & 0 & 1e-4 & 100 &2&937&32\\ \hline
\multirow{3}{*}{RNNs} &1 layer& 0.04 & 4e2 & 1 & 1.2e-6 & 0 & 200 &4&663&5 \\ 
                     &2 layer& 0.01 & 2e2/3 & 1 &  1.2e-6 & 0 & 200&4&663&5 \\ 
                     &3 layer& 0.01 & 2e2/3 & 1 &  1.2e-6 & 0 & 200&4&663&5 \\\hline
\end{tabular}
\begin{tablenotes}
      $^*~ \gamma_\ell(t,s)=t$ if $\ell$ is a multiple of $sn$ and $=1$ otherwise.
      
      $^{\star} \gamma_\ell=1-\frac{1}{E-\lfloor {\ell}/{n}\rfloor}$, where $E$ is the number of total epochs and $\lfloor a\rfloor$ is the floor of $a$. 
      
       $^{\dagger}$ Hyperparameters are set for NSISA. The remaining is set for SISA.
    \end{tablenotes}
\label{table:hyperparameter}
\end{center}\vspace{-5mm}
\end{table}

\subsection{Effect of Key Parameters}
We consider a simple example where the function  $f$ is a least squares loss, namely, $$f\left(\bw; \bx\right)=\frac{1}{2}(\langle\bw,{\bf a}\rangle-b)^2,$$ where ${\bx=({\bf a},b)}$. The data generation process follows that in \cite{zhou2023federated}. We use a setup with ${m=32}$ batches, ${\mu=\lambda=0}$, and ${\bw\in\R^{100}}$.  Our experiments reveal that parameters ${(\sigma_i^0,\gamma_i)}$ significantly influence the performance of the proposed algorithm. For simplicity, we set ${\sigma_i^0=\sigma^0}$ and ${\gamma_i^0=\gamma}$ for all ${i\in[m]}$. According to Algorithms \ref{algorithm-ADMM-mini}-\ref{algorithm-ADMM-NS},  update rule ${\sigma_i^\ell=\sigma^0\gamma^{-\ell}}$ may result in rapid growth of $\sigma_i^\ell$ if $\gamma$ is far from $1$ (e.g., ${\gamma\in(0,0.9]}$).  However, empirical evidence suggests that updating $\sigma_i^\ell$ at every iteration is not always ideal. Instead, updating it every few iterations (e.g., when $\ell$ is a multiple of $k_0$) can lead to improved performance. We will demonstrate this in the sequel.

\begin{figure}[!t]
\centering
\begin{subfigure}{.495 \textwidth}
	\centering
	\includegraphics[width=.99\linewidth]{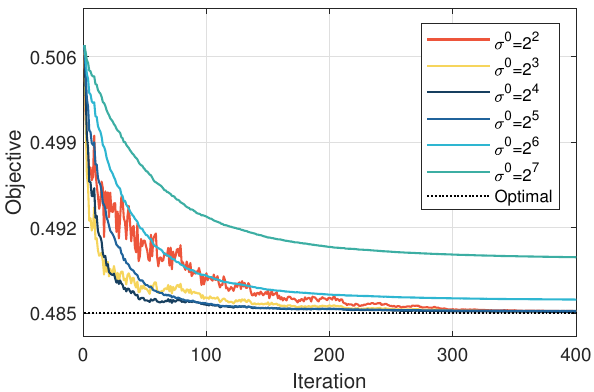} 
	\caption{
Effect of $\sigma^0$ using second-order information.\label{fig:tuning-para-a}}
\end{subfigure}	 
\begin{subfigure}{.495 \textwidth}
	\centering
	\includegraphics[width=.99\linewidth]{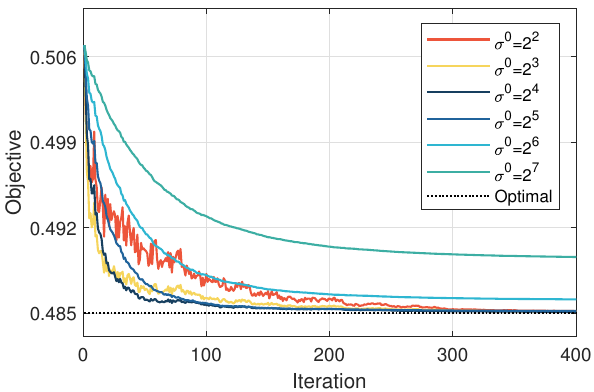} 
	\caption{
Effect of $\sigma^0$ using second moment.\label{fig:tuning-para-b}}
\end{subfigure}  \\[3ex]
 
\begin{subfigure}{.495 \textwidth}
	\centering
	\includegraphics[width=.99\linewidth]{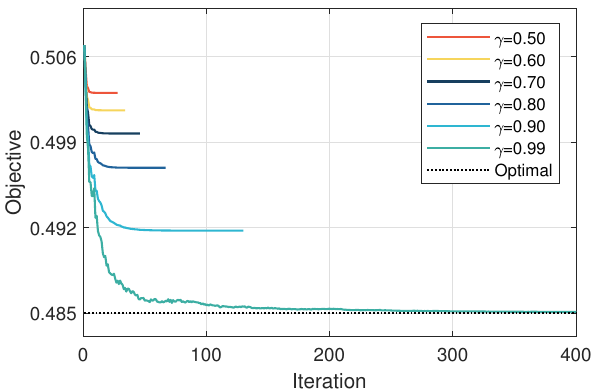} 
	\caption{
Effect of $\gamma$ with $k_0=1$.\label{fig:tuning-para-c}}
\end{subfigure}  
\begin{subfigure}{.495 \textwidth}
	\centering
	\includegraphics[width=.99\linewidth]{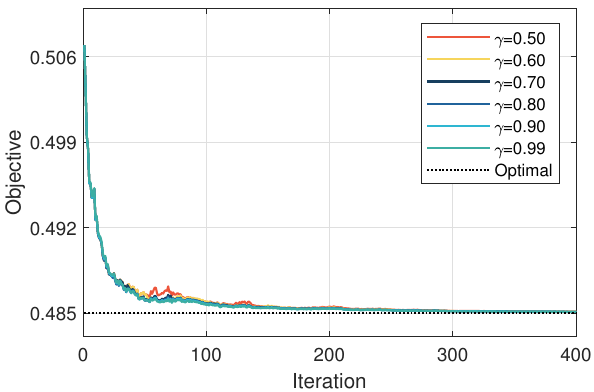} 
		\caption{
Effect of $\gamma$ with $k_0=\lceil\ln(\gamma)/\ln(0.99)\rceil$.\label{fig:tuning-para-d}}
\end{subfigure} 
\caption{
Effect of $\sigma^0$ and $\gamma$.\label{fig:tuning-para}} 
\end{figure}

{\bf a) Effect of $\sigma^0$.} To examine the effect of $\sigma^0$, we fix ${\gamma=0.99}$, vary ${\sigma^0\in\{2^2,2^3,\ldots,2^7\}}$, and set ${k_0=1}$, meaning $\sigma_i^\ell$ is updated at every iteration, namely ${k_0=1}$. We implement PISA using both second-order information (i.e., \eqref{Q-newton}) and second moment estimation (i.e., Scheme III in \eqref{second_moment_update}) to construct $\Q_i^{\ell}$. Note that when second moment is used, PISA reduces to SISA (i.e., Algorithm \ref{algorithm-ADMM-SM}).  The results are shown in Figures \ref{fig:tuning-para-a} and \ref{fig:tuning-para-b}. When ${\sigma^0=2^2}$, the objective curves fluctuate significantly during the first 100 iterations but eventually converge to the optimal value. With  ${\sigma^0=2^4}$ or  $2^5$, the convergence is more stable and reaches the optimum efficiently. However, for ${\sigma^0>2^6}$, the curves are very smooth but fail to reach the optimal value.
  
These findings provide practical guidance for tuning $\sigma^0$. Specifically, one can initialize the algorithm with a moderately small value of $\sigma^0$, depending on the task. If the algorithm exhibits persistent oscillations or divergence within the first few hundred iterations, it may be beneficial to restart with a larger $\sigma^0$. This tuning process can be repeated until a steadily decreasing objective is observed. 

{\bf b) Effect of $\gamma$.} To study the effect of $\gamma$, we fix ${\sigma^0=2^4}$ and alter $\gamma\in\{0.5,0.6,0.7,0.8,0.9,0.99\}$. We employ SISA to solve the problem with updating $\sigma_i^\ell$ at every iteration. As shown in Figure \ref{fig:tuning-para-c}, the algorithm achieves the optimal value only when ${\gamma = 0.99}$. This result is expected since the update rule, ${\sigma_i^\ell=\sigma^0\gamma^{-\ell}}$, causes  $\sigma_i^\ell$   to grow rapidly for ${\gamma\leq0.99}$, forcing premature convergence to suboptimal values. Interestingly, Theorem \ref{main-convergence-rate-eps} (specifically, Equation \eqref{rate-F}) suggests that smaller  $\gamma$ leads to faster convergence in theory, which is indeed supported by the results in Figure \ref{fig:tuning-para-c}. However, this theoretical speedup comes at the cost of solution quality.

To moderate the effect of fast-growing $\sigma_i^\ell$, we recommend updating it less frequently, such as only when $\ell$ is a multiple of $k_0$. However, determining a suitable $k_0$ for a given $\gamma$ is nontrivial. Nevertheless, a practical strategy is to choose $k_0$  such that the cumulative decay rate over $K$ iterations matches that of 
${\gamma=0.99}$, namely,   
\begin{equation}\label{k0-gamma}
0.99^K=\gamma^{K/k_0}\qquad \Longrightarrow\qquad k_0=\left\lceil\frac{\ln(\gamma)}{\ln(0.99)}\right\rceil.
\end{equation}
For example, ${k_0=69,36,11,1}$ when ${\gamma=0.5,0.7,0.9,0.99}$, respectively. Using this strategy, we run SISA with updating $\sigma_i^\ell$ only when $\ell$ is a multiple of ${k_0=\lceil\ln(\gamma)/\ln(0.99)\rceil}$. As shown in Figure \ref{fig:tuning-para-d}, all curves closely follow the same trajectory and converge to the optimal value, confirming the effectiveness of the proposed approach for choosing $k_0$  based on \eqref{k0-gamma}.

Overall, according to Table \ref{table:hyperparameter}, a recommended range for $(\sigma_i^0,\gamma_i^0,k_0)$ for each $i\in[m]$ can be $\sigma_i^0\in\{0.005,0.01,0.05,0.1,0.5\}$, $\gamma_i^0\in\{0.99,0.995,0.999\}$, and $k_0$ chosen as \eqref{k0-gamma}.

\subsection{Classification by VMs}\label{app:class-VM}
We perform classification tasks using  ImageNet \cite{krizhevsky2012imagenet}  and evaluate the performance by reporting the convergence speed.   To validate this, we  adopt a large-batch training strategy, which provides more stable gradient estimates and reduces variance in parameter updates. All experimental settings follow the official implementations of widely used model architectures on the ImageNet dataset\footnote{https://github.com/pytorch/examples/tree/main/imagenet}. 
We compare SISA with five optimizers: Adam, AdamW, AdaBelief, Lamb, and SGD-M \cite{robbins1951stochastic} on training four VMs:  ViT \cite{dosovitskiy2020image}, AlexNet \cite{krizhevsky2012imagenet}, ResNet \cite{he2016identity}, and DenseNet \cite{huang2017densely}. 
To ensure fair comparison, all models are trained using identical training strategies and hyperparameters, such as batch size and weight decay. The only exception is the learning rate (which plays a role analogous to $\sigma$ and $\rho$ in SISA), which is tuned for each optimizer to achieve comparable performance. Figure \ref{fig:imagnet} illustrates the training loss over epochs, and SISA consistently achieves the fastest convergence and maintains the lowest training loss.

\begin{figure}[!t]
\centering
\begin{subfigure}{.495 \textwidth}
	\centering
	\includegraphics[width=.95\linewidth]{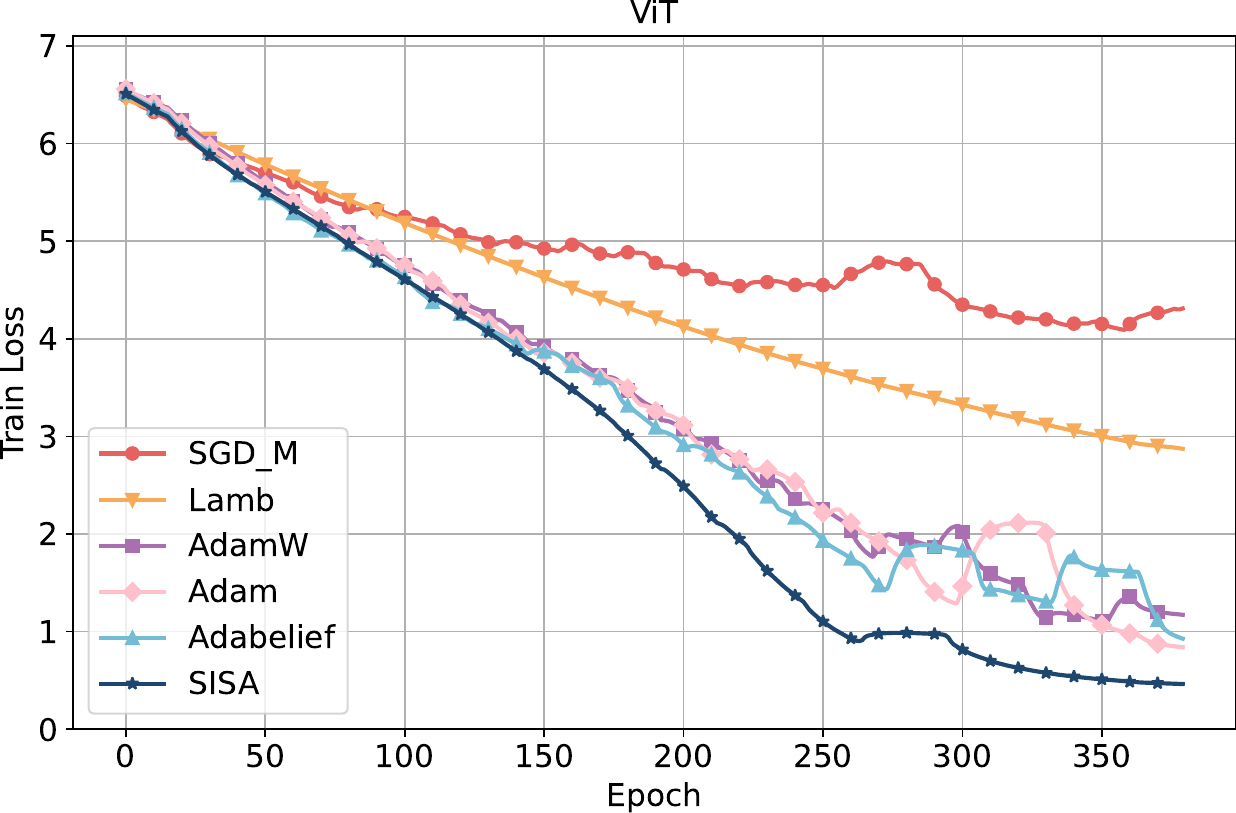} 
\end{subfigure}	 
\begin{subfigure}{.495 \textwidth}
	\centering
	\includegraphics[width=.95\linewidth]{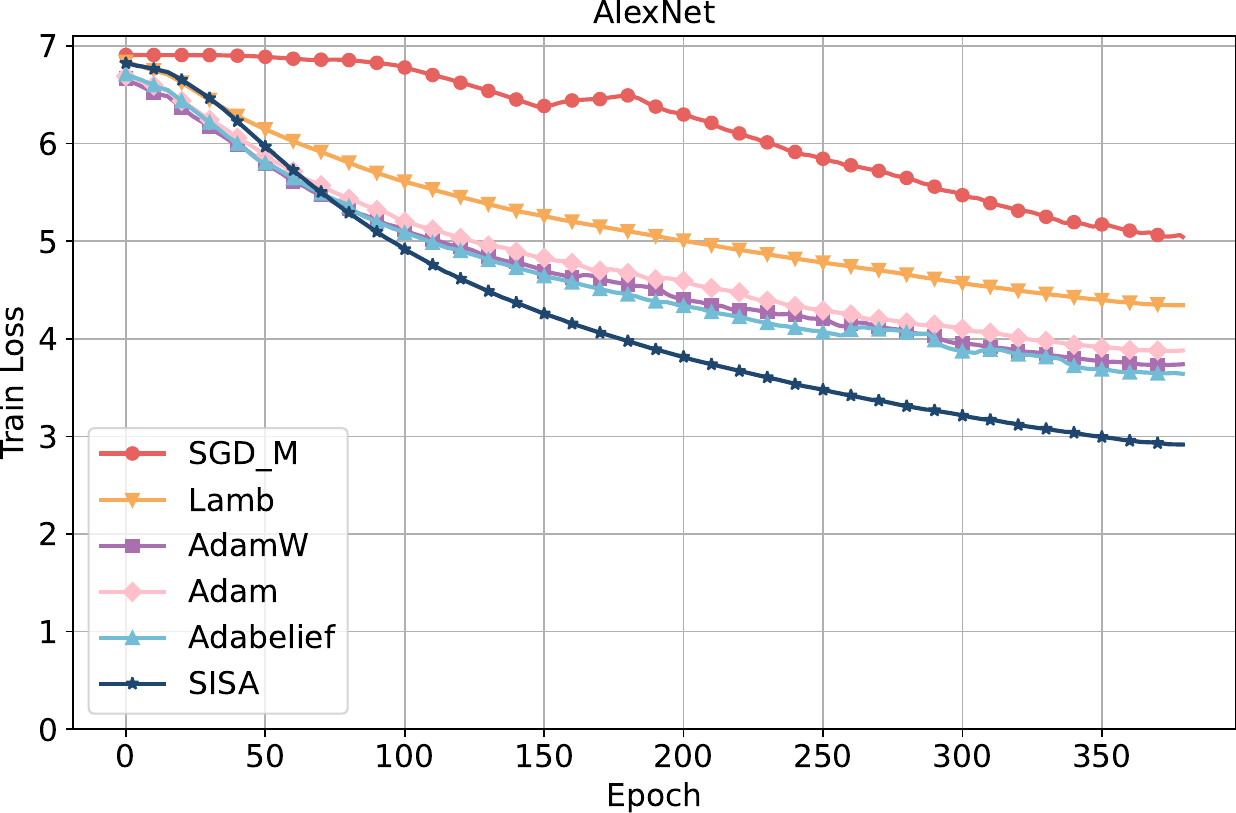} 
\end{subfigure}   
\\[2ex]
\begin{subfigure}{.495 \textwidth}
	\centering
	\includegraphics[width=.95\linewidth]{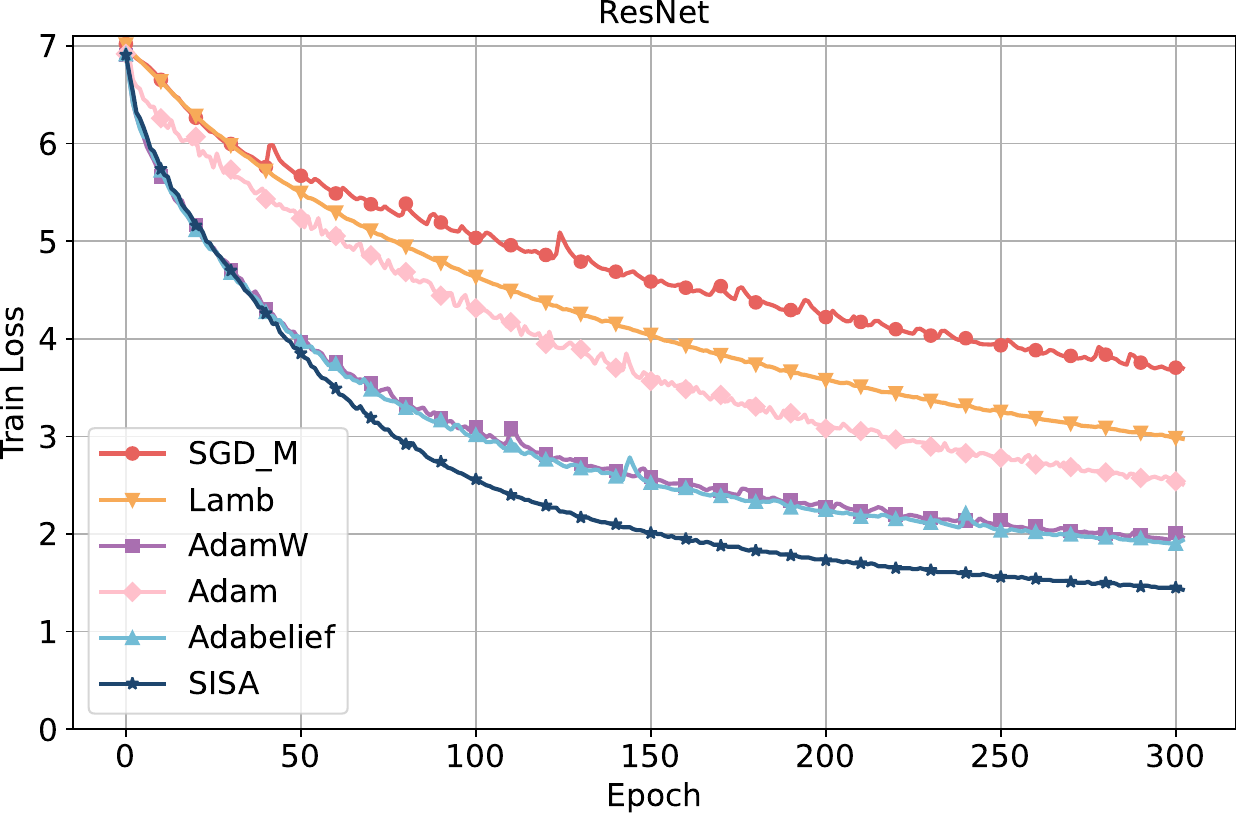} 
\end{subfigure}	 
\begin{subfigure}{.495 \textwidth}
	\centering
	\includegraphics[width=.95\linewidth]{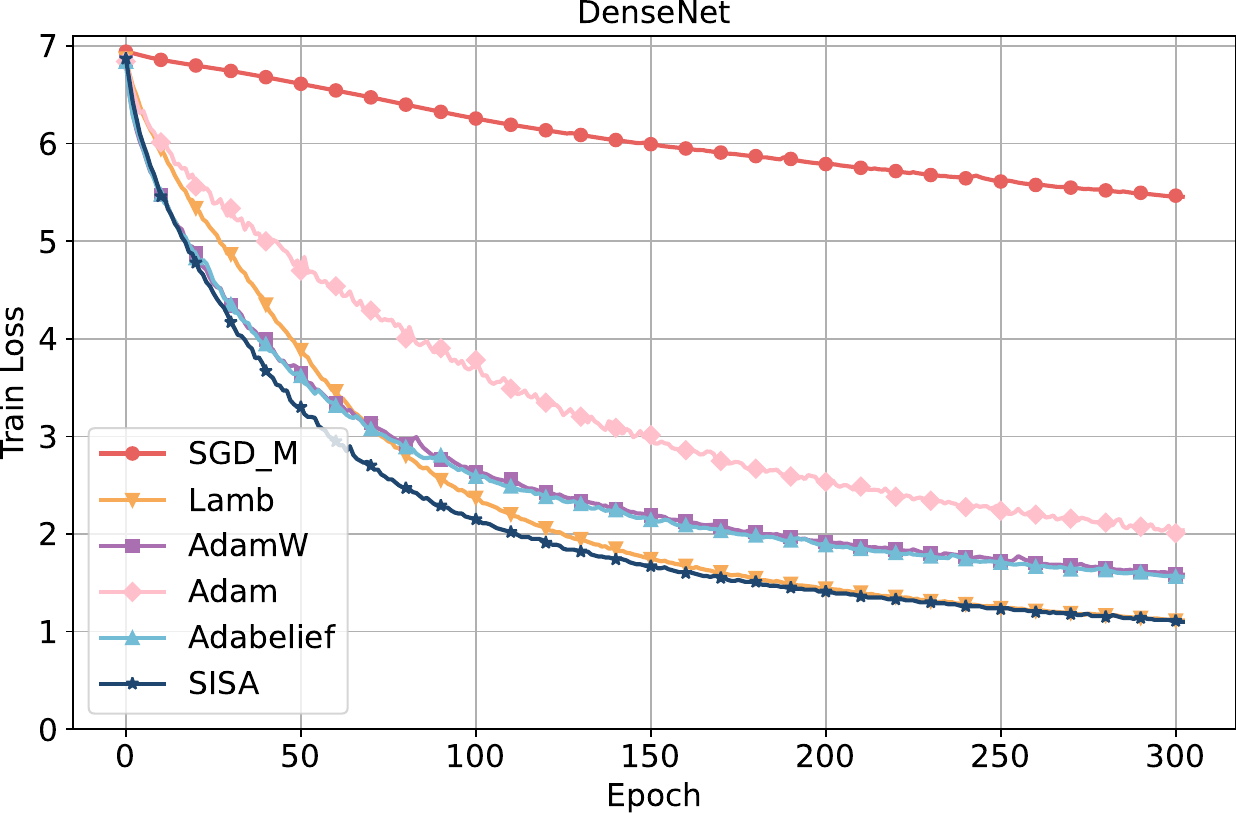} 
\end{subfigure} 
\caption{Training loss over epochs for four VMs on ImageNet.\label{fig:imagnet}}
\end{figure}

\begin{figure}[!t]
\centering
\begin{subfigure}{.495 \textwidth}
	\centering
	\includegraphics[width=.95\linewidth]{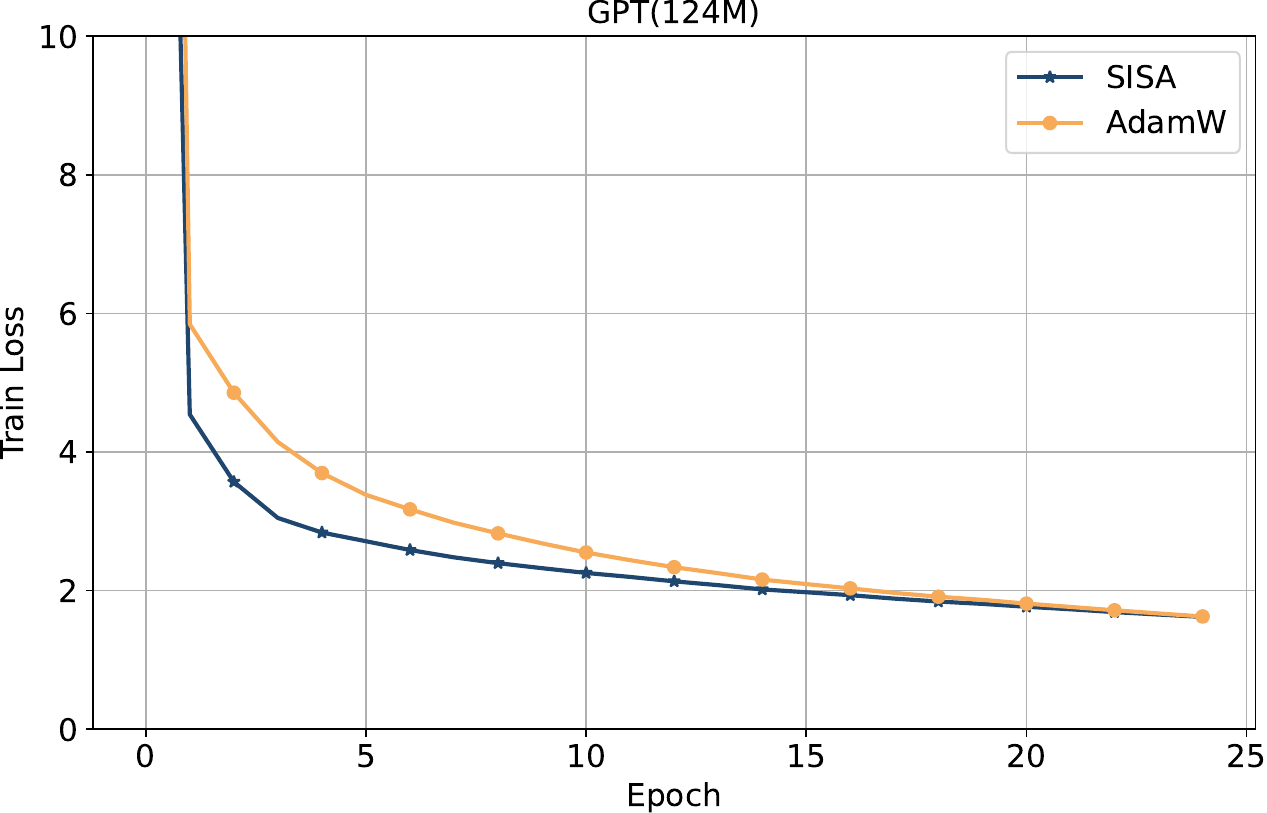} 
\end{subfigure}	 
\begin{subfigure}{.495 \textwidth}
	\centering
	\includegraphics[width=.95\linewidth]{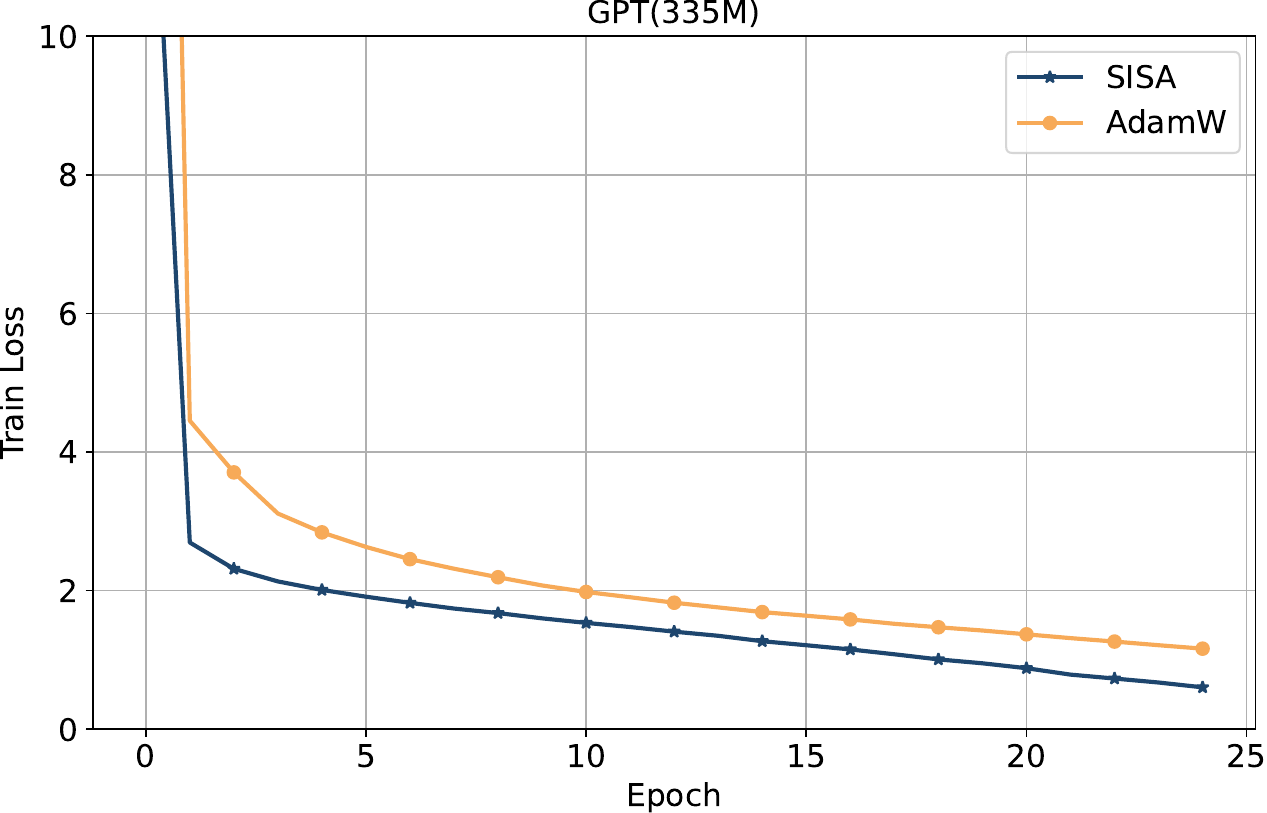} 
\end{subfigure}   
\caption{
Training loss over epochs for fine-tuning GPT2. \label{fig:gpt}}
\end{figure}

\subsection{Training LLMs}\label{app:LLMs}

We compare SISA and AdamW  to train two GPT2 models with 124M and 335M parameters \cite{radford2019language},  following the experimental setup and hyperparameter configurations outlined in \cite{tran2023finetuning}. To prevent overfitting, we terminate SISA at the 25th epoch at which the testing loss begins to rise. To evaluate performance, we report both the training and testing loss, as well as the semantic similarity \cite{zhang2019bertscore}. It noted that lower training and testing loss and higher semantic similarity indicate better performance. Semantic similarity is computed using BERTScore \cite{zhang2019bertscore}, which compares the contextual word embeddings from a pre-trained BERT model to capture the deeper meaning of two texts, rather than relying on exact word matches, and calculates an overall similarity score. By incorporating semantic similarity alongside testing loss, we provide a more comprehensive evaluation of the generated recipes, considering both token-level accuracy and logical coherence. The training process over 25 epochs is shown in Figure \ref{fig:gpt}, where SISA consistently demonstrates faster convergence. Moreover, as illustrated in Table \ref{table:gpt_small}, SISA outperforms AdamW in both testing loss and semantic similarity during the early stages of training.

Table \ref{table:gpt_memory} reports the memory overhead (in MiB) for different optimizers. Note that the batch size for GPT2-Medium is 16, whereas for GPT2-XL it is 4, with four steps of gradient accumulation. This explains why some optimizers exhibit similar memory overhead for GPT2-Medium and GPT2-XL. We also acknowledge that certain algorithms, such as the one proposed in \cite{zhang2024optimal}, are specifically designed to reduce memory usage. However, we do not include them in our comparison due to the lack of publicly available implementations. Despite incurring higher memory overhead than some baseline methods, NSISA consistently achieves lower validation loss in less wall-clock time for both the GPT2-Medium and GPT2-XL configurations, as shown in Figure \ref{fig:gpt_new}.

\begin{table}[!t]
\renewcommand{\arraystretch}{1.0}\addtolength{\tabcolsep}{6.5pt}
\begin{center}
\caption{Testing loss and semantic similarity for GPT2.} 
\label{table:gpt_small}
\begin{tabular}{llccccccc}
\hline
&& \multicolumn{3}{c}{Testing Loss} && \multicolumn{3}{c}{Semantic Similarity(\%)} \\ \cline{3-5}\cline{7-9}
&Epoch      & 2& 5& 15&& 2& 5& 15\\ \hline
\multirow{2}{*}{GPT2(124M)}&AdamW                & 5.20& 3.45& 2.30&& 82.16 & 84.30& 84.73\\ 
&SISA                & 3.69& 2.73& 2.17&& 83.15 & 84.36& 84.82 \\ \hline
\multirow{2}{*}{GPT2(335M)}&AdamW                & 3.62& 2.73& 1.95&& 83.17 & 84.50& 84.19\\  
&SISA                & 2.31& 2.04& 1.93&& 84.76 & 85.24& 84.84 \\ \hline
\end{tabular}
\end{center} 
\end{table}

\begin{table}[!t]
\renewcommand{\arraystretch}{1.0}\addtolength{\tabcolsep}{6.0pt}
\begin{center}
\caption{Memory Overhead (MiB) with Different Optimizers. }
\label{table:gpt_memory}
\begin{tabular}{lcccccc}
\hline 
 & NSISA & Adam & Muon& Shampoo& SOAP & Adam-mini\\\hline 
GPT2-Nano &\textcolor{blue}{22,998} & 22,996 & 22,998 &24,336 & 26,432 &22,996\\
GPT2-Medium &\textcolor{blue}{50,580}& 50,658 & 47,860 &54,794 & 58,812 &47,548\\
GPT2-XL &\textcolor{blue}{55,126} & 53,128 & 50,657 &68,146 & 93,630 &50,658\\
\hline
\end{tabular}
\end{center}
\end{table}
\subsection{Games reward by RLMs}
  In this subsection, we employ SISA to train  RLMs, specifically focusing on two popular algorithms: Proximal Policy Optimization (PPO) and Advantage Actor-Critic (A2C) \cite{duan2016benchmarking}, which are Adam-based approaches. Therefore, we can directly substitute Adam with SISA in PPO and A2C to derive PPO-SISA and A2C-SISA, without requiring additional modification. We refer to the default configurations as PPO-Adam and A2C-Adam. We evaluate these algorithms on four continuous control tasks/games simulated in MuJoCo \cite{todorov2012mujoco}, a widely-used physics engine. The environments for the four tasks are Ant, Humanoid, Half-Cheetah, and Inverted Double Pendulum.

In each test, agents are rewarded at every step based on their actions. Following standard evaluation procedures, we run each task with the same random seeds and assess performance over 10 episodes every 30,000 steps. All experiments are conducted using the Tianshou framework \cite{weng2022tianshou} with its default hyperparameters as the baseline. The results are shown in Figure \ref{fig:rl}, where the solid line represents the average rewards and the shaded region indicates the standard deviation over episodes. Across all four games, SISA outperforms the Adam-based algorithms, achieving higher rewards.

\begin{figure}[!t]
\centering
\begin{subfigure}{.495 \textwidth}
	\centering
	\includegraphics[width=.95\linewidth]{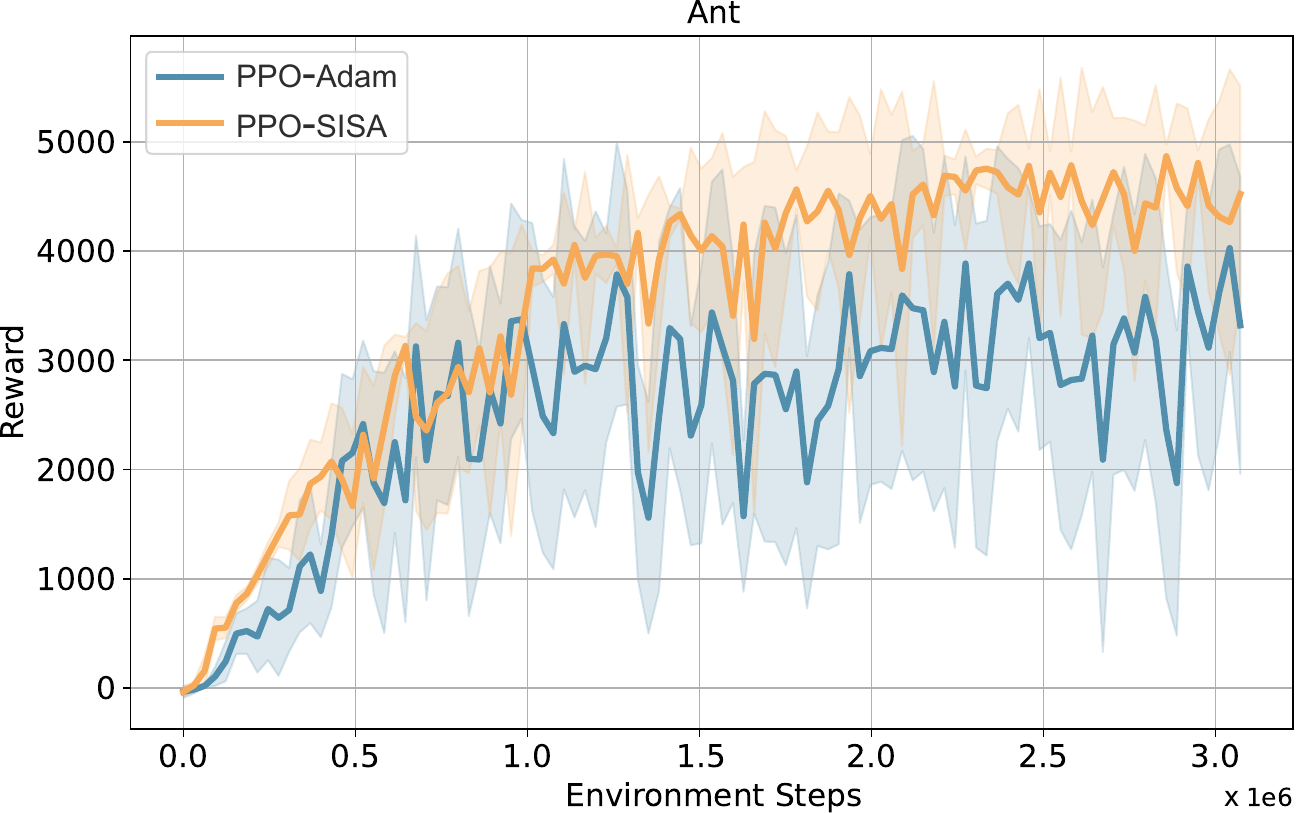} 
\end{subfigure}	 
\begin{subfigure}{.495 \textwidth}
	\centering
	\includegraphics[width=.95\linewidth]{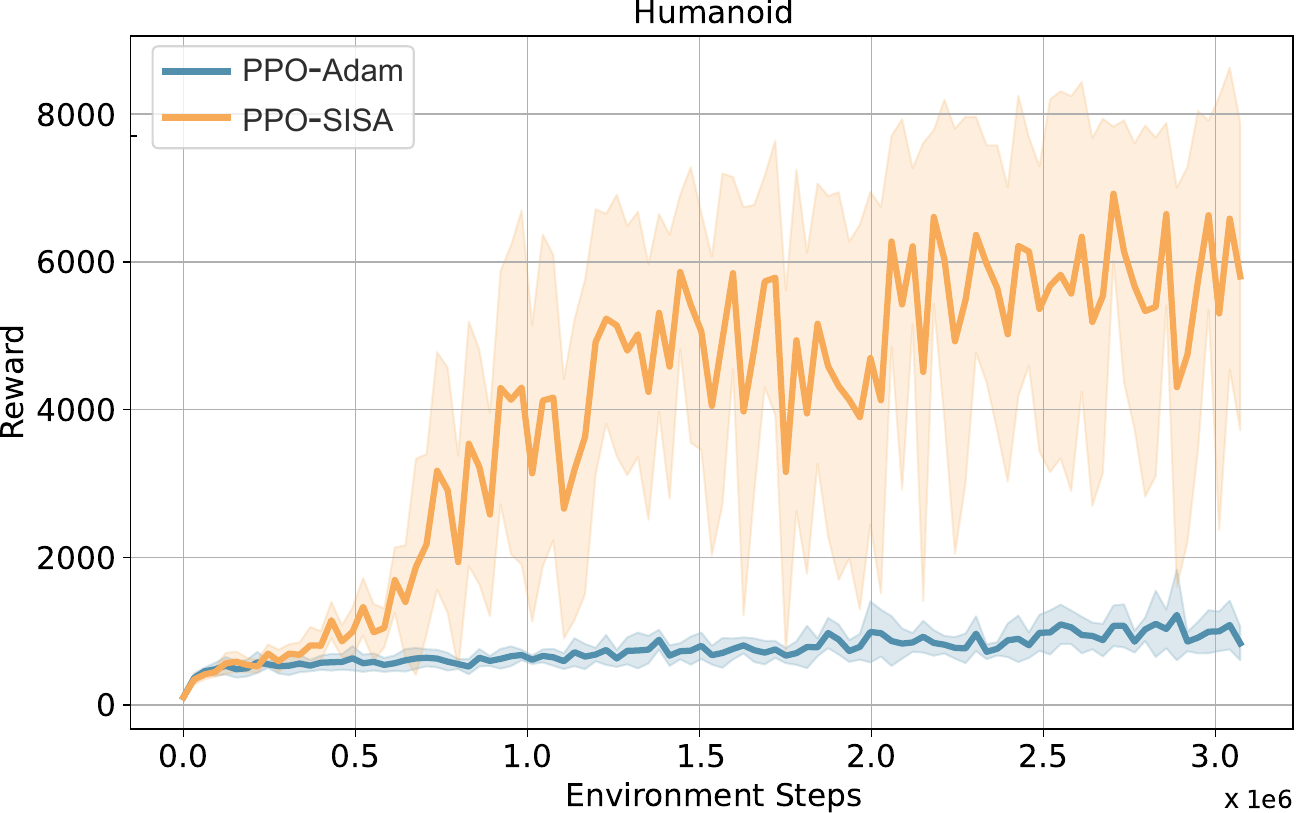} 
\end{subfigure}   
 \\[2ex]
\begin{subfigure}{.495 \textwidth}
	\centering
	\includegraphics[width=.95\linewidth]{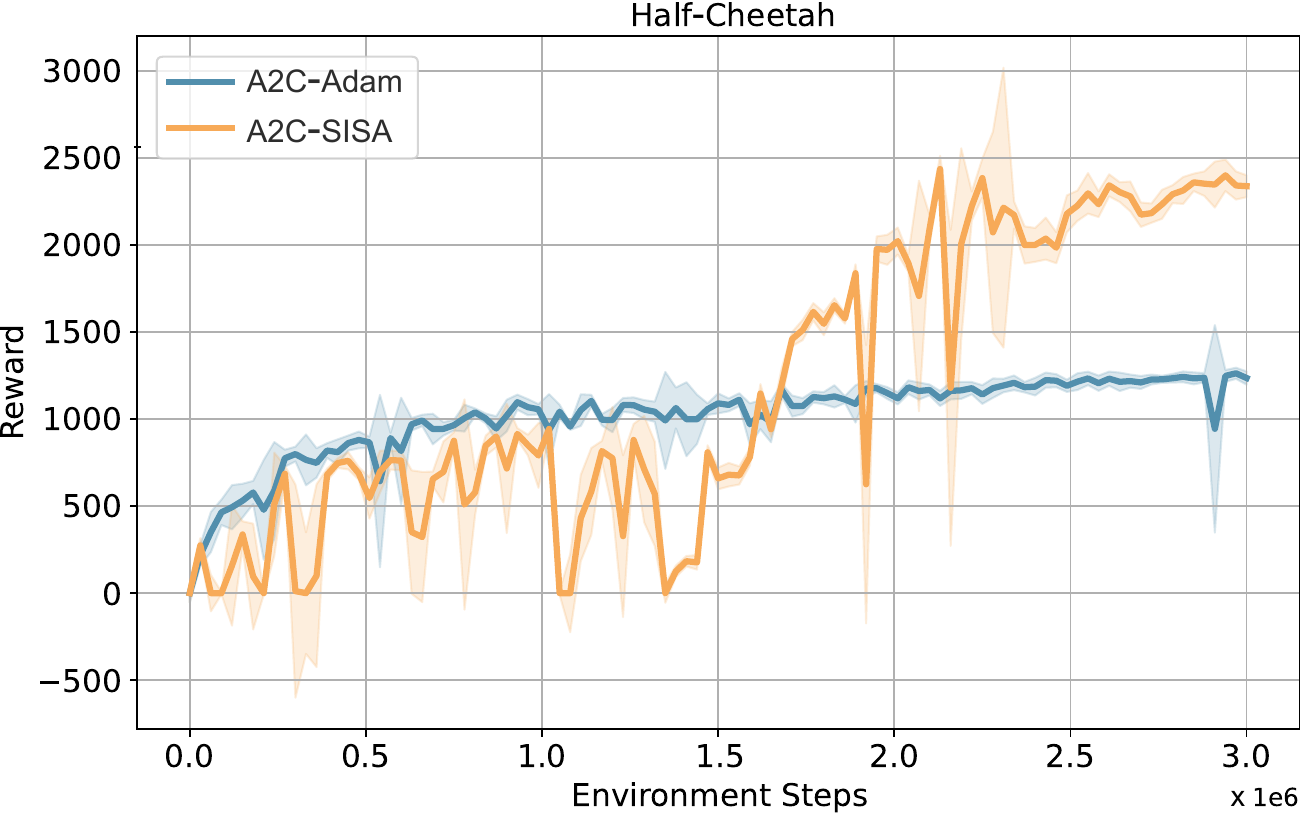} 
\end{subfigure}   
\begin{subfigure}{.495 \textwidth}
	\centering
	\includegraphics[width=.95\linewidth]{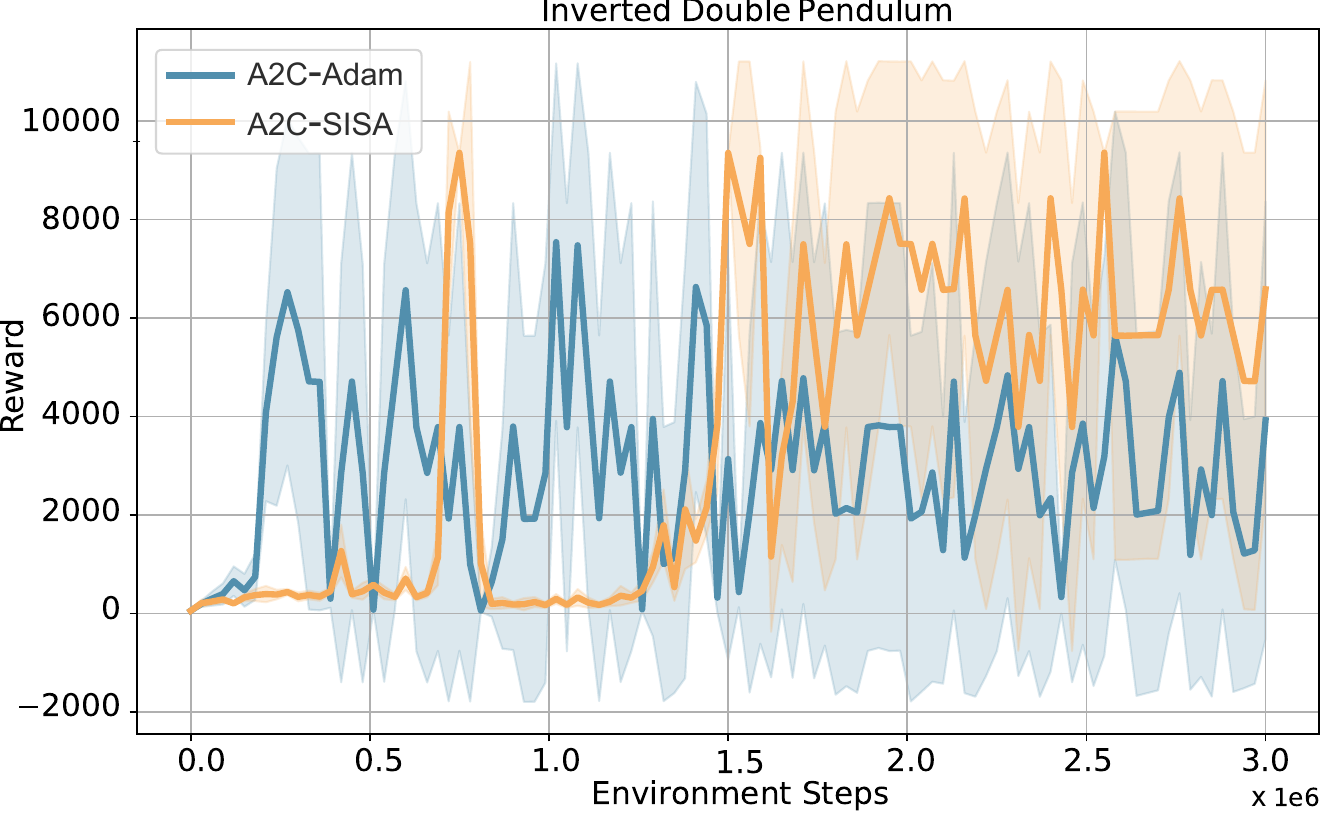} 
\end{subfigure} 
\caption{Rewards over steps for RLMs under four environments.\label{fig:rl}}
\end{figure}

\subsection{Text prediction with RNNs}
We extend our experiments to natural language processing tasks by applying SISA to long short-term memory (LSTM) networks \cite{graves2012long}, one of the most widely used RNNs. We adopt the same experimental setup and default hyperparameter configurations as in \cite{zhuang2020adabelief} since the provided code offers a broader range of baseline comparisons. Model performance is evaluated on the Penn TreeBank dataset \cite{marcus1993building}, with perplexity scores reported on the testing dataset. As shown in Table \ref{table:lstm_sota}, SISA consistently achieves the lowest perplexity across all three LSTM models, demonstrating its effectiveness in improving RNNs performance.

\begin{table}[!th]
\renewcommand{\arraystretch}{1.0}\addtolength{\tabcolsep}{-3.0pt}
\begin{center}
\caption{Testing perplexity on Penn Treebank for 1, 2, 3-layer LSTM. }
\label{table:lstm_sota}
\begin{tabular}{ccccccccccc}
\hline 
 & SISA & SGD-M & AdaBound & Adam & Radam& MSVAG& AdamW & AdaBelief & Yogi &Padam\\ 
  &   & \cite{robbins1951stochastic} & \cite{luo2019adaptive} &  \cite{kingma2014adam} &  \cite{liu2019variance} &  \cite{balles2018dissecting} &  \cite{loshchilov2017decoupled} &  \cite{zhuang2020adabelief} & \cite{zaheer2018adaptive} & \cite{chen2018closing} \\\hline 
1-layer & {84.27}& 85.70 & 84.73 &$85.90^*$&$86.50^*$ &84.85 &88.39 &84.47 & 86.37 & $94.46^\star$\\
2-layer & {66.23}& 67.32 & 67.43 &$67.30^*$&$72.30^*$ &$68.82^*$ &$72.80^*$ &66.70 & 71.64 & $89.77^\star$\\
3-layer & {61.06}& 63.91 & 63.72 &$65.02^*$&$70.00^*$ &$64.32^*$ &$69.90^*$ &61.20 & 67.69 & $95.10^\star$\\
\hline
\end{tabular}
\begin{tablenotes}
$*$ and $\star$ are reported in \cite{Adan24} and \cite{zhuang2020adabelief}. {Lower} values indicate better performance. 
\end{tablenotes}\vspace{-3mm}
\end{center} 
\end{table}

\newpage
\section{Appendix: Supplementary Information}
In this section, we prove the proof of all theorems in the main context, before the main results, we first present some useful facts. 
\subsection{Useful facts}

\begin{itemize}[leftmargin=17pt]
\item For any $\bx, \bz, \bw$, and $\bv$, and $t>0$, 
\begin{eqnarray}\label{triangle-ineq}
\begin{aligned}
  \left\langle \bw, \bv \right\rangle &\leq   \dfrac{t}{2} \left\| \bw  \right\|^2 +   \frac{1}{2t}\left\| \bv \right\|^2\\
\left \| \bw + \bv  \right\|^2&\leq (1+t) \left\| \bw  \right\|^2 + \Big(1+\frac{1}{t}\Big) \left\| \bv  \right\|^2,\\
  \left\langle \bx-\bz, \bw -\bv \right\rangle &\leq   \dfrac{1}{2} \left\|\bx- \bv \right\|^2- \dfrac{1}{2} \left\|\bx- \bw \right\|^2 +    \dfrac{1}{2} \left\|\bz- \bw \right\|^2- \dfrac{1}{2} \left\|\bz- \bv \right\|^2.
\end{aligned}
\end{eqnarray}
\item By update  (\ref{sub-w-mini}), it follows 
\begin{eqnarray}\label{optimality-condition-w}
\begin{aligned}
0=~&\sum_{i=1}^m \alpha_i \Big[   - \bpi_{i}^\ell  -  \sigma_i^{\ell}(\bw_{i}^\ell- \bw^{\ell+1})\Big] +\lambda \bw^{\ell+1}.
\end{aligned}
\end{eqnarray}
\item  By update  (\ref{sub-wbn-mini}), it follows
\begin{eqnarray}\label{optimality-condition}
\begin{aligned}
0=~&\bpi_{i}^{\ell} +  \bg_{i}^{\ell+1}+ \left(\sigma_i^{\ell+1}\I+ \rho_i \Q_i^{\ell+1}\right) \triangle \overline{\bw}_{i}^{\ell+1}   \\[1ex]
\overset{\eqref{sub-pin}}{=} & \bpi_{i}^{\ell+1} +   { \bg_{i}^{\ell+1}} +\rho_i \Q_i^{\ell+1} \triangle \overline{\bw}_{i}^{\ell+1} .
\end{aligned}
\end{eqnarray}
\item The choice, $\eta_i\I\succeq \Q_i^{\ell+1}   \succeq0$, of $\Q_i^{\ell+1}$ indicates
\begin{eqnarray}\label{Q-upper-bd}
\begin{aligned}
\left\| \Q_i^{\ell+1} \right\| \leq \eta_i, ~~\forall i\in[m].
\end{aligned}
\end{eqnarray}
\item   
  We note that the range of $\gamma_i$  is
\begin{eqnarray}\label{range-gamma}
\begin{aligned}
\frac{3}{4}\leq\gamma_i<1, ~~\forall i\in[m].
\end{aligned}
\end{eqnarray}
Together with  \eqref{choice-of-sigma}, one can check the following order,
\begin{eqnarray}
\label{increase-sigma}\sigma_i^{\ell+1}>\sigma_i^{\ell}>\sigma_i^0\geq\sigma^0 \geq 8\max_{i\in[m]}\Big\{\sigma,~\rho_i\eta_i,~ r_i,~\delta^{-2},~  \varepsilon_i(2\delta)\Big\}, ~~\forall i\in[m], \forall \ell\geq 1.
\end{eqnarray} 
Furthermore,   for any $\ell\geq0$, 
\begin{eqnarray} \label{upbd-sigma-ell-1}
\begin{aligned} 
\frac{\varepsilon_i(2\delta )}{\sigma_i^{\ell}} \overset{\eqref{sub-sigma-mini}}{=} \frac{\varepsilon_i(2\delta )\gamma_i^{\ell}}{\sigma_i^0} \overset{\eqref{increase-sigma}}{\leq} \frac{ \gamma_i^{\ell}}{8} = \frac{\gamma_i^{\ell}-\gamma_i^{\ell+1}}{8(1- \gamma_i)}.
\end{aligned}
\end{eqnarray}
\item By letting $\textbf{P}_i^{\ell+1}:=(\sigma_i^{\ell+1}-\sigma_i^{\ell})\I+ \rho_i \Q_i^{\ell+1}$, it follows 
\begin{eqnarray} \label{P-upper-bd}
\begin{aligned} \|\textbf{P}_i^{\ell+1}\| \overset{\eqref{Q-upper-bd}}{\leq}  \sigma_i^{\ell+1}-\sigma_i^{\ell}+\rho_i\eta_i \overset{(\ref{sub-sigma-mini},\ref{increase-sigma})}{\leq}  \left(\frac{1}{\gamma_i} -1+\frac{1}{8}  \right) \sigma_i^{\ell}   \overset{\eqref{range-gamma}}{\leq} \frac{\sigma_i^{\ell}}{2}.
\end{aligned}
\end{eqnarray}
\item Let $(\bw^*,\W^*)$ be the optimal solution to \eqref{opt-prob-distribute}. Since $F^*$ is the optimal function value and is lower bounded,  for any $\bw=\bw_i,i\in[m]$, it holds
\begin{eqnarray}  \label{opt-value}
\begin{aligned}
-\infty < F^*   &=  \sum_{i=1}^m \alpha_i   F_{i}(\bw_i^*) + \frac{\lambda}{2}\|\bw^*\|^2   
 \leq \sum_{i=1}^m \alpha_i   F_{i}(\bw_i) + \frac{\lambda}{2}\|\bw\|^2\\
& = \sum_{i=1}^m \alpha_i   F_{i}(\bw) + \frac{\lambda}{2}\|\bw\|^2 
  = F(\bw)+  \frac{\lambda}{2}\|\bw\|^2.
\end{aligned}
\end{eqnarray} 
\item  To make these notation in \eqref{def-notation} well-defined when ${\ell=0}$, we let $\bw^{-1}=\bw^0,~\bw_{i}^{-1}=\bw_{i}^0,~\bpi_{i}^{-1}=\bpi_{i}^0$, and $\bg_i^{0}=\nabla F_i(\bw^0)$ for any $i\in[m]$. 
\end{itemize}

\subsection{Key lemmas}
\begin{lemma} Under Assumption \ref{assumption}, for any $\bw,\bv\in \N(2\delta)$ and  any $\B_i\subseteq\D_i$,
\begin{eqnarray}\label{Lipschitz-continuity}
\begin{aligned}
  & \left\|\nabla  F_i(\bw;\B_i) - \nabla  F_i(\bv;\B_i) \right\|   \leq \frac{\sigma_i^\ell}{8} \left\| \bw- \bv \right\|,&&\forall i\in[m], ~\forall \ell\geq 0, \\
 & F_{i}(\bw) -F_{i}(\bv) \leq   \left\langle \nabla F_{i}(\bu), \bw-\bv \right\rangle  +\frac{ \sigma_i^{\ell}}{16} \|  \bw-\bv\|^2,&&\forall i\in[m], ~\forall \ell\geq 0, 
\end{aligned}\end{eqnarray} 
where $\bu=\bw$ or $\bu=\bv$.
\end{lemma}
\begin{proof} 
The Lipschitz continuity of $\nabla f(\cdot;\bx_t)$ with  $c(\bx_t)>0$ on $\N(2\delta)$ and $r_i=\max_{\bx_t\in\D_i}c(\bx_t)+ \mu-\lambda$ imply that, for any $\B_i\subseteq\D_i$,
\begin{eqnarray} \label{lip-fi-ri}
\eqspace{2}
\begin{array}{lcl}
   \Big\|\nabla  F_i(\bw;\B_i) - \nabla  F_i(\bv;\B_i) \Big\| 
     &\overset{\eqref{def-F_i-F}}{=}& \Big\|\nabla  H_i(\bw;\B_i) - \nabla  H_i(\bv;\B_i) + (\mu-\lambda)(\bw-\bv)\Big\|\\
 &\leq&  \Big\| \frac{1}{|\B_i|}\sum_{\bx_t\in\B_i} (\nabla  f(\bw;\bx_t) - \nabla  f(\bv;\bx_t) )\Big\|+(\mu-\lambda)\Big\| \bw-\bv  \Big\| \\
  &\leq&  \frac{1}{|\B_i|} \sum_{\bx_t\in\B_i}c(\bx_t)\Big\| \bw-\bv  \Big\|+(\mu-\lambda)\Big\| \bw-\bv  \Big\| \\
  &= &   r_i \Big\| \bw-\bv  \Big\| \overset{\eqref{increase-sigma}}{\leq}\frac{\sigma_i^\ell}{8} \Big\| \bw- \bv \Big\|
\end{array}\end{eqnarray}
It follows from \cite[Eq. (37)]{zhouli23} and the above condition that for $\bu=\bw$ or $\bu=\bv$,
\begin{eqnarray*}
\begin{aligned}  
F_{i}(\bw) -F_{i}(\bv) ~\leq ~ &\left\langle \nabla F_{i}(\bu), \bw-\bv \right\rangle  +\frac{r_i}{2} \|  \bw-\bv\|^2\\
\overset{\eqref{increase-sigma}}{\leq} &\left\langle \nabla F_{i}(\bu), \bw-\bv \right\rangle  +\frac{ \sigma_i^{\ell}}{16} \|  \bw-\bv\|^2,
\end{aligned}
\end{eqnarray*} 
showing the desired result.
\end{proof}

\begin{lemma}\label{descent-lemma} Let $\{(\bw^\ell,\W^\ell,\P^\ell)\}$ be the sequence generated by Algorithm \ref{algorithm-ADMM-mini} with $\bsi^0$ chosen as (\ref{choice-of-sigma}). Under Assumption \ref{assumption}, if ${\bw^\ell, \bw^\ell_i\in\N(\delta)}$ for any ${i\in[m]}$, then  $\bw^{\ell+1}, \bw^{\ell+1}_i\in \N(2\delta)$ and
\begin{eqnarray}  \label{L-lower-bd}
\eqspace{2}
\begin{array}{lll} 
\widetilde{\L}^{\ell+1} 
\geq
  \L^{\ell+1} 
\geq \dsum  \alpha_i \left[  F_{i}(\bw^{\ell+1} )  +  \dfrac{\sigma  }{2}  \|\triangle \overline{\bw}^{\ell+1} _i\|^2  - 1 + \dfrac{\lambda}{2}\|\bw^{\ell+1} \|^2 \right] 
\geq F^* -1.
\end{array}
\end{eqnarray}
\end{lemma}
\begin{proof} By $\bw^\ell, \bw^\ell_i\in\N(\delta)$, we obtain
 \begin{eqnarray*}  
 \begin{aligned}
\rho_i \|\Q_i^{\ell} \triangle \overline{\bw}_{i}^{\ell}\|  &\overset{\eqref{increase-sigma}}{\leq} \frac{\sigma_i^\ell}{8} (\| \bw^{\ell}\| +\| \bw_{i}^{\ell}\|) \leq  \frac{\sigma_i^\ell \delta }{4}, \end{aligned} 
 \end{eqnarray*} 
 and the following bound, \begin{eqnarray*}  
 \begin{aligned}
\frac{ \| { \bg_{i}^{\ell}}\|}{ \sigma_i^\ell}    \overset{\eqref{def-eps}}{\leq}    \frac{  \sqrt{\varepsilon_i(\delta)}}{8\sigma_i^\ell}    \overset{\eqref{def-eps}}{\leq}       \frac{  \sqrt{\varepsilon_i(2\delta)}}{8\sigma_i^\ell}     \overset{\eqref{increase-sigma}}{\leq}      \frac{1}{8\sqrt{8\sigma_i^\ell}} \overset{\eqref{increase-sigma}}{\leq}    \frac{ \delta }{64}. 
 \end{aligned} 
 \end{eqnarray*} 
Let $t_{\ell}:=\sum_{i=1}^m\alpha_{i}  \sigma_i^\ell+\lambda.$  
Based on the above conditions, we obtain
 \begin{eqnarray*}  
 \begin{aligned}
\|\bw^{\ell+1}\|  &~~=~  \frac{1}{t_{\ell}} \Big\|\sum_{i=1}^m\alpha_{i} \left(\sigma_i^\ell  {\bw}_{i}^{\ell}+\bpi_{i}^{\ell}  \right) \Big\|\\
& \overset{\eqref{optimality-condition}}{=}   \frac{1}{t_{\ell}}\Big\|\sum_{i=1}^m\alpha_{i} \left( \sigma_i^\ell {\bw}_{i}^{\ell} -\rho_i \Q_i^{\ell} \triangle \overline{\bw}_{i}^{\ell} -   { \bg_{i}^{\ell}}   \right)\Big\|\\
 &~~=~  \frac{1}{t_{\ell}}\Big\|\sum_{i=1}^m\alpha_{i} \left( (\sigma_i^\ell -\rho_i \Q_i^{\ell}){\bw}_{i}^{\ell} -\rho_i \Q_i^{\ell}  {\bw}^{\ell} -   { \bg_{i}^{\ell}}   \right)\Big\|\\
 &~~\leq~  \frac{1}{t_{\ell}} \sum_{i=1}^m\alpha_{i} \left( \sigma_i^\ell \|{\bw}_i^{\ell}\| + \rho_i \eta_i \|{\bw} ^{\ell}\| + \| { \bg_{i}^{\ell}}\| \right)  \\
 &~~\leq~  \delta  + \frac{ \delta }{8} + \frac{\delta}{64}  \leq \frac{73\delta}{64},  
 \end{aligned} 
 \end{eqnarray*} 
 thereby leading to ${\bw^{\ell+1}\in\N(2\delta)}$  and thus
\begin{eqnarray}\label{fact-0006}
\left\|\triangle\bg_i^{\ell+1}\right\|^2\overset{\eqref{def-eps}}{\leq}   \frac{\varepsilon_i(2\delta )}{64}.
\end{eqnarray} Now we estimate ${\bw^{\ell+1}_i}$ by 
 \begin{eqnarray*} 
 \eqspace{2}
\begin{array}{lcl}
\left\|\bw_{i}^{\ell+1} \right\| &\overset{\eqref{optimality-condition}}{=}& \left\|\left(\sigma_i^{\ell+1}\I+ \rho_i \Q_i^{\ell+1}\right)^{-1} \left(-\bpi_{i}^{\ell} -  \bg_{i}^{\ell+1}\right) +  \bw^{\ell+1}\right\|    \\ 
&\overset{\eqref{optimality-condition}}{=}& \left\|\left(\sigma_i^{\ell+1}\I+ \rho_i \Q_i^{\ell+1}\right)^{-1} \left(\rho_i \Q_i^{\ell}\triangle \overline{\bw}_{i}^{\ell}+ \bg_{i}^{\ell} -  \bg_{i}^{\ell+1}\right) +  \bw^{\ell+1}\right\|    \\ 
&\overset{\eqref{optimality-condition}}{\leq}& \dfrac{1}{\sigma_i^{\ell+1}}\left(  \| \rho_i \Q_i^{\ell}\triangle \overline{\bw}_{i}^{\ell} \| +  \|\bg_{i}^{\ell} \| +   \|\bg_{i}^{\ell+1} \| \right) +   \|\bw^{\ell+1}  \| \\
&\overset{\eqref{optimality-condition}}{\leq}& \dfrac{ {\delta}}{4} + \dfrac{ 2{\delta}}{64} +\dfrac{73{\delta}}{64}  \leq {2\delta}.
\end{array}
\end{eqnarray*} 
 One can verify that 
\begin{eqnarray} \label{fact-06}
 \eqspace{2}
\begin{array}{lcl}
   \left \langle -\bg_i^{\ell+1}, \triangle \overline{\bw}^{\ell+1}_i\right\rangle 
   &\overset{(\ref{optimality-condition})}{=} &  \left\langle \bpi_{i}^{\ell+1} + \rho_i \Q_i^{\ell+1} \triangle \overline{\bw}^{\ell+1} _i, \triangle \overline{\bw}^{\ell+1} _i\right\rangle \\ 
&\overset{\eqref{Q-upper-bd}}{\leq} &  \left\langle \bpi_{i}^{\ell+1} , \triangle \overline{\bw}^{\ell+1} _i\right\rangle +   \rho_i\eta_i \left\| \triangle \overline{\bw}^{\ell+1} _i \right\|^2\\ 
&\overset{\eqref{increase-sigma}}{\leq} &  \left\langle \bpi_{i}^{\ell+1} , \triangle \overline{\bw}^{\ell+1} _i\right\rangle +   \dfrac{\sigma_i^{\ell+1} }{8} \left\| \triangle \overline{\bw}^{\ell+1} _i \right\|^2,\\
 \left\langle \triangle \bg_i^{\ell+1} , \triangle \overline{\bw}^{\ell+1} _i\right\rangle 
&\overset{(\ref{triangle-ineq})}{\leq}&  \dfrac{\sigma_i^{\ell+1} }{16}\left\| \triangle \overline{\bw}^{\ell+1} _i\right\|^2+\dfrac{4}{\sigma_i^{\ell+1} } \left\|\triangle \bg_i^{\ell+1} \right\|^2.
\end{array}
\end{eqnarray}
These two conditions and $\bw^{\ell+1} , \bw^{\ell+1} _i\in\N(2\delta)$ lead to
\begin{eqnarray} \label{fact-6-w-wi}
 \eqspace{2}
\begin{array}{lcl}
F_{i}(\bw^{\ell+1} ) 
&\overset{(\ref{Lipschitz-continuity})}{\leq}&F_{i}(\bw_{i}^{\ell+1} )+\left\langle -\nabla F_{i}(\bw ^{\ell+1} ), \triangle \overline{\bw}^{\ell+1} _i\right\rangle + \dfrac{\sigma_i^{\ell+1} }{16}\left\|\triangle \overline{\bw}^{\ell+1} _i\right\|^2\\
&~\overset{(\ref{def-notation})}{=}&F_{i}(\bw_{i}^{\ell+1} )+\left\langle \triangle \bg_i^{\ell+1} -\bg_i^{\ell+1} , \triangle \overline{\bw}^{\ell+1} _i\right\rangle + \dfrac{\sigma_i^{\ell+1} }{16}\left\|\triangle \overline{\bw}^{\ell+1} _i\right\|^2\\
&\leq & F_{i}(\bw_{i}^{\ell+1} )+ \left\langle \bpi_{i}^{\ell+1}, \triangle \overline{\bw}^{\ell+1} _i\right\rangle  +\dfrac{\sigma_i^{\ell+1} }{4}\left\|\triangle \overline{\bw}^{\ell+1} _i\right\|^2+\dfrac{4}{\sigma_i^{\ell+1} } \left\|\triangle \bg_i^{\ell+1} \right\|^2,
\end{array}
\end{eqnarray}
thereby bringing out
\begin{eqnarray*}
\eqspace{2}
\begin{array}{lcl}
\widetilde{\L}^{\ell+1} 
\overset{(\ref{def-tilde-L})}{\geq }
  \L^{\ell+1} 
&\overset{(\ref{opt-prob-distribute-Lag})}{=}&  \dsum \alpha_i \left[  F_{i}(\bw_{i}^{\ell+1} )+\langle \bpi_{i}^{\ell+1} , \triangle \overline{\bw}^{\ell+1} _i\rangle +   \dfrac{\sigma_i^{\ell+1} }{2}  \|\triangle \overline{\bw}^{\ell+1} _i\|^2  + \dfrac{\lambda}{2}\|\bw^{\ell+1} \|^2  \right] \\ 
&\overset{(\ref{fact-6-w-wi})}{\geq }& \dsum  \alpha_i \left[  F_{i}(\bw^{\ell+1} )  +  \dfrac{\sigma_i^{\ell+1} }{4}  \|\triangle \overline{\bw}^{\ell+1} _i\|^2  - \dfrac{4}{\sigma_i^{\ell+1}}\left\|\triangle\bg_i^{\ell+1} \right\|^2 + \dfrac{\lambda}{2}\|\bw^{\ell+1} \|^2  \right]\\ 
&\overset{(\ref{fact-0006})}{\geq}&  \dsum  \alpha_i \left[  F_{i}(\bw^{\ell+1} )  +  \dfrac{\sigma_i^0 }{4}  \|\triangle \overline{\bw}^{\ell+1} _i\|^2  - \dfrac{  \varepsilon_i(2\delta ) }{16\sigma_i^{\ell+1}} + \dfrac{\lambda}{2}\|\bw^{\ell+1} \|^2 \right]\\
&\overset{(\ref{increase-sigma})}{\geq}& \dsum  \alpha_i \left[  F_{i}(\bw^{\ell+1} )  +  \dfrac{\sigma  }{2}  \|\triangle \overline{\bw}^{\ell+1} _i\|^2  - 1 + \dfrac{\lambda}{2}\|\bw^{\ell+1} \|^2 \right]\\
&\overset{(\ref{opt-value})}{\geq}&  F^* -1, 
\end{array}
\end{eqnarray*}
showing the desired result. 
\end{proof} 

\begin{lemma}\label{descent-lemma} Let $\{(\bw^\ell,\W^\ell,\P^\ell)\}$ be the sequence generated by Algorithm \ref{algorithm-ADMM-mini} with $\bsi^0$  chosen as (\ref{choice-of-sigma}). Under Assumption \ref{assumption}, for any $i\in[m]$ and any $\ell\geq0$, if $\bw^\ell, \bw_i^\ell \in\N(\delta)$ then 
 \begin{eqnarray}
 \label{gap-pi-k}
     \left\|\triangle \bpi_{i}^{\ell+1}\right\|^2  \leq     \frac{\varepsilon_i(2\delta)}{6} + \dfrac{16\varphi_i^{\ell+1}}{3},
\end{eqnarray}
where $\varphi_i^{\ell+1}$ is defined by
 \begin{eqnarray}
 \label{def-varphi}\varphi_i^{\ell+1}:=   \left\|\rho_i \Q_i^{\ell} \triangle\overline{\bw}^{\ell}_i\right\|^2-\left\|\rho_i \Q_i^{\ell+1} \triangle\overline{\bw}^{\ell+1}_i\right\|^2 .\end{eqnarray} 
\end{lemma}
\begin{proof} 
  First,  it follows from $\bw^\ell, \bw_i^\ell \in\N(\delta)$ and \eqref{L-lower-bd} that $\bw^{\ell+1}, \bw^{\ell+1}_i\in \N(2\delta)$, resulting in  
 \begin{eqnarray} \label{fact-33}
\begin{aligned}
\left\|\bg_{i}^{\ell+1}-\bg_{i}^{\ell}\right\|^2 \leq 2\left\|\bg_{i}^{\ell+1} \right\|^2+2\left\| \bg_{i}^{\ell}\right\|^2
\leq \dfrac{\varepsilon_i(2\delta)}{16}.
\end{aligned}
\end{eqnarray}
Moreover, one can check that
\begin{eqnarray*}
\eqspace{2}
\begin{array}{lcl}
 \left\|\rho_i \Q_i^{\ell+1} \triangle\overline{\bw}^{\ell+1}_i\right\|^2+ \left\| \rho_i\Q_i^{\ell}\triangle\overline{\bw}^{\ell}_i  \right\|^2
&=& \varphi_i^{\ell+1} + 2\left\|\rho_i \Q_i^{\ell+1} \triangle\overline{\bw}^{\ell+1}_i\right\|^2 \\
   &\overset{\eqref{Q-upper-bd}}{\leq}  &    \varphi_i^{\ell+1} +  \dfrac{ 2\rho_i^2\eta_i^2}{(\sigma_i^{\ell+1})^2} \left\| \triangle {\bpi}^{\ell+1}_i\right\|^2 \\ 
  &\overset{\eqref{increase-sigma}}{\leq}  &   \varphi_i^{\ell+1} +   \dfrac{ 1}{32} \left\| \triangle {\bpi}^{\ell+1}_i\right\|^2,
\end{array}
\end{eqnarray*}
which further results in
\begin{eqnarray*} 
\eqspace{2}
\begin{array}{lcl}
 \left\|\triangle \bpi_{i}^{\ell+1}\right\|^2  
 &\overset{\eqref{optimality-condition}}{=} & \left\|\bg_{i}^{\ell+1}-\bg_{i}^{\ell}+\rho_i \Q_i^{\ell+1} \triangle\overline{\bw}^{\ell+1}_i - \rho_i \Q_i^{\ell} \triangle\overline{\bw}^{\ell}_i \right\|^2\\
&\overset{\eqref{triangle-ineq}}{\leq}  &   2\left\|\bg_{i}^{\ell+1}-\bg_{i}^{\ell}\right\|^2 + 4   \left\|\rho_i \Q_i^{\ell+1} \triangle\overline{\bw}^{\ell+1}_i\right\|^2+ 4 \Big\| \rho_i\Q_i^{\ell}\triangle\overline{\bw}^{\ell}_i  \Big\|^2 \\
 &\overset{\eqref{fact-33}}{\leq} & \dfrac{\varepsilon_i(2\delta)}{8} +  \dfrac{1}{8} \left\| \triangle {\bpi}^{\ell+1}_i\right\|^2+ 4\varphi_i^{\ell+1} .
\end{array}
\end{eqnarray*}
 This enables showing \eqref{gap-pi-k}.  \end{proof}

\subsection{Proof of Lemma \ref{descent-lemma-L}}\label{proof-descent-lemms} \label{app:proof-lemma32}
\begin{proof}  
We decompose 
\begin{eqnarray*}  
  \L^{\ell+1}- \L^\ell =G_0+G_1+G_2+G_3,
\end{eqnarray*}
where $G_0, G_1,G_2,$ and $G_3$ are defined by
\begin{eqnarray} 
\begin{aligned} 
G_0 &:=\L\left(\bw^{\ell+1},\W^{\ell+1},\P^{\ell+1};\bsi^{\ell+1}\right)- \L\left(\bw^{\ell+1},\W^{\ell+1},\P^{\ell+1};\bsi^{\ell}\right),\\
G_1 &:=\L\left(\bw^{\ell+1},\W^{\ell+1},\P^{\ell+1};\bsi^{\ell}\right)- \L\left(\bw^{\ell+1},\W^{\ell+1},\P^\ell;\bsi^{\ell}\right),\\
G_2 &:=\L\left(\bw^{\ell+1},\W^{\ell+1},\P^{\ell};\bsi^{\ell}\right)- \L\left(\bw^{\ell+1},\W^{\ell},\P^{\ell};\bsi^{\ell}\right),\\
G_3 &:= \L\left(\bw^{\ell+1},\W^{\ell},\P^{\ell};\bsi^{\ell}\right)- \L\left(\bw^\ell,\W^\ell,\P^\ell;\bsi^{\ell}\right).
\end{aligned}
\end{eqnarray}
We  prove the results by induction. 

\noindent \textbf{Part 1: $\ell=0$.}  Since $\bw_i^0=\bw^0=0$ for any $i\in[m]$, we have that
\begin{eqnarray} \label{w0-in-N-delta}
\begin{aligned} 
\bw^0\in\N(\delta ),~~\bw_i^0\in \N(\delta), ~\forall~i\in[m].
\end{aligned}
\end{eqnarray}
 For $G_0$, it follows that
\begin{eqnarray*} 
\begin{aligned} 
G_0 =\sum_{i=1}^m  \frac{\alpha_i(\sigma_i^{1}-\sigma_i^{0})}{2} \| \triangle \overline{\bw}^{1}_i\|^2 \overset{(\ref{sub-sigma-mini},\ref{sub-pin})}{=}\sum_{i=1}^m  \alpha_i \left[\frac{1}{2\sigma_i^{1}}-\frac{ \gamma_i^2 }{2\sigma_i^{0}} \right]\|\triangle \bpi_{i}^{1}\|^2.
\end{aligned}
\end{eqnarray*}
For $G_1$, we have 
\begin{eqnarray*} 
\begin{aligned} 
G_1 
 = \sum_{i=1}^m \alpha_i  \langle \bpi_{i}^{1}- \bpi_{i}^{0}, \bw_{i}^{1}- \bw^{1}\rangle   \overset{(\ref{sub-pin})}{=} \sum_{i=1}^m   \frac{\alpha_i}{\sigma_i^{1}}\left\|\triangle \bpi_{i}^{1} \right\|^2. 
\end{aligned}
\end{eqnarray*}
For $G_2$,  it follows from $\bw^0, \bw_i^0 \in\N(\delta)$ and \eqref{L-lower-bd} that $\bw^{1}, \bw^{1}_i\in \N(2\delta)$, resulting in  
\begin{eqnarray} \label{fact-02}
\eqspace{2}
\begin{array}{lcl}
 \left\langle   \nabla F_{i}(\bw_{i}^{1
 })-\bg_i^{1}, \triangle \bw_{i}^{1}\right\rangle  
 &\overset{(\ref{triangle-ineq})}{\leq } & \dfrac{4 }{ \sigma_i^{0}}  \left\|\nabla F_{i}(\bw_{i}^{1})-\bg_i^{1} \right\|^2 + \dfrac{\sigma_i^{0}}{16}\left\|\triangle \bw_{i}^{1}\right\|^2  \\
 &\overset{(\ref{def-eps})}{\leq }& \dfrac{ \varepsilon_i(2\delta )}{4\sigma_i^{0}}  + \dfrac{\sigma_i^{0}}{16}\left\|\triangle \bw_{i}^{1}\right\|^2.  
\end{array}
\end{eqnarray}
This further contributes to 
\begin{eqnarray} \label{fact-222}
\eqspace{2}
\begin{array}{lcl}
F_{i}(\bw_{i}^{1}) - F_{i}(\bw_{i}^{0})
&\overset{(\ref{Lipschitz-continuity})}{\leq }  & \left\langle \nabla F_{i}(\bw_{i}^{1}), \triangle \bw_{i}^{0} \right\rangle+ \dfrac{\sigma_i^{0}}{16}\left\|\triangle \bw_{i}^{1}\right\|^2\\ 
&=& \left\langle \nabla F_{i}(\bw_{i}^{1})-\bg_i^{1}, \triangle \bw_{i}^{1} \right\rangle +\left\langle  \bg_i^{1}, \triangle \bw_{i}^{1} \right\rangle + \dfrac{\sigma_i^{0}}{16}\left\|\triangle \bw_{i}^{1}\right\|^2 \\  
&\overset{(\ref{fact-02})}{\leq } &  \left\langle \bg_i^{1}, \triangle \bw_{i}^{1} \right\rangle+\dfrac{\sigma_i^{0}}{8}\left\|\triangle \bw_{i}^{1}\right\|^2 +\dfrac{\varepsilon_i(2\delta )}{4\sigma_i^{0}} .
\end{array}
\end{eqnarray} 
In addition, one can check that
\begin{eqnarray} \label{fact-22}
\eqspace{2}
\begin{array}{lcl} 
 -   \left \langle \textbf{P}_i^{1} \triangle\overline{\bw}^{1}_i, \triangle \bw_{i}^{1}\right\rangle 
&\overset{\eqref{triangle-ineq}}{\leq} &  \dfrac{\sigma_i^{0}}{8}  \left\| \triangle \bw_{i}^{1}\right\|^2+ \dfrac{2}{ \sigma_i^{0}}\left\|  \textbf{P}_i^{1} \triangle\overline{\bw}^{1}_i\right\|^2\\
&\overset{\eqref{P-upper-bd}}{\leq} ~&  \dfrac{\sigma_i^{0}}{8}  \left\| \triangle \bw_{i}^{1}\right\|^2+\dfrac{ \sigma_i^{0}}{2}\left\|  \triangle\overline{\bw}^{1}_i\right\|^2\\
&\overset{\eqref{sub-pin}}{\leq} ~&  \dfrac{\sigma_i^{0}}{8}  \left\|\triangle \bw_{i}^{1}\right\|^2+\dfrac{\gamma_i^2}{2\sigma_i^{0}}\left\|\triangle\bpi^{1}_i\right\|^2.
\end{array}
\end{eqnarray}
The above facts enable us to derive that
\begin{eqnarray*} 
\eqspace{2}
\begin{array}{lcl} 
p_i~  
&:=& L_{i}(\bw^{1},\bw_i^{1}, \bpi_i^{0};\sigma_i^{0}) - L_{i}(\bw^{1},\bw_i^{0}, \bpi_i^{0};\sigma_i^{0})  \\
&\overset{(\ref{opt-prob-distribute-Lag})}{=} &
  F_{i}(\bw_{i}^{1}) - F_{i}(\bw_{i}^{0})   +   \left\langle  \bpi_{i}^{0} +\sigma_i^{0}  \triangle\overline{\bw}^{1}_i, \triangle \bw_{i}^{1}\right\rangle -  \dfrac{\sigma_i^{0}}{2}\left\|\triangle \bw_{i}^{1}\right\|^2\\ 
& \overset{(\ref{fact-222})}{\leq } &
\dfrac{\varepsilon_i(2\delta )}{4\sigma_i^{0}}     -  \dfrac{3\sigma_i^{0}}{8} \left\|\triangle \bw_{i}^{1}\right\|^2  + \left\langle \bg_i^{1}+\bpi_{i}^{0}+\sigma_i^{0}  \triangle\overline{\bw}^{1}_i, \triangle \bw_{i}^{1}\right\rangle \\ 
&\overset{(\ref{optimality-condition})}{=} & \dfrac{\varepsilon_i(2\delta )}{4\sigma_i^{0}}   -  \dfrac{3\sigma_i^{0}}{8} \left\|\triangle \bw_{i}^{1}\right\|^2  -  \left\langle  \textbf{P}_i^{1} \triangle\overline{\bw}^{1}_i, \triangle \bw_{i}^{1}\right\rangle   \\ 
&\overset{(\ref{fact-22})}{\leq }&\dfrac{\varepsilon_i(2\delta )}{ 4\sigma_i^{0}}   -  \dfrac{\sigma_i^{0}}{4} \left\|\triangle \bw_{i}^{1}\right\|^2 +\dfrac{\gamma_i^2}{2\sigma_i^{0} } \left\| \triangle \bpi_{i}^{1}\right\|^2,
\end{array}
\end{eqnarray*}
which immediately leads to 
\begin{eqnarray*} 
\eqspace{2}
\begin{array}{lcl}
G_2  
= \dsum  \alpha_i p_i  
\leq   \dsum  \alpha_i \left[\dfrac{\varepsilon_i(2\delta )}{4\sigma_i^{ 1}} -  \dfrac{\sigma_i^{0}}{4} \left\|\triangle \bw_{i}^{1}\right\|^2 +\dfrac{\gamma_i^2}{2\sigma_i^{0} } \left\| \triangle \bpi_{i}^{1}\right\|^2 \right].
\end{array}
\end{eqnarray*}
For $G_3$,  direct verification leads to 
\begin{eqnarray*} 
\eqspace{2}
\begin{array}{lcl}
G_3  &=&\dsum  \alpha_i \left[ \left\langle \bpi_{i}^{0},  - \triangle \bw^{1}\right\rangle + \dfrac{\sigma_i^{0}}{2} \left\| \bw_{i}^{0}- \bw^{1}\right\|^2 -\dfrac{\sigma_i^{0}}{2} \left\| \bw_{i}^{0}- \bw^{0}\right\|^2 \right] +  \dfrac{\lambda}{2}\left\|\bw^{1}\right\|^2-  \dfrac{\lambda}{2}\left\|\bw^{0}\right\|^2 \\ 
&=&\dsum  \alpha_i \left[ \left\langle -\bpi_{i}^{0}+\sigma_i^{0}( \bw^{1}-\bw_{i}^{0}) + \lambda \bw^{1},   \triangle \bw^{1}\right\rangle    -  \dfrac{\sigma_i^{0}+\lambda}{2}\left\|  \triangle\bw^{1} \right\|^2 \right] \\
&\overset{(\ref{optimality-condition-w})}{=}& \dsum \alpha_i \left[- \dfrac{\sigma_i^{0}+\lambda}{2}\left \|  \triangle\bw^{1} \right\|^2\right].
\end{array}
\end{eqnarray*}
Overall, the upper bounds of $G_0, G_1, G_2,$ and $G_3$ yield that 
\begin{eqnarray} \label{gap-LL-sum} 
\eqspace{2}
\begin{array}{lcl}
 \L^{1}- \L^0 
&\leq&\dsum \alpha_i  \left[\dfrac{\varepsilon_i(2\delta )}{4\sigma_i^{0}}     -  \dfrac{\sigma_i^{0}+2\lambda}{2}\left\|  \triangle\bw^{1} \right\|^2-  \dfrac{\sigma_i^{0}}{4} \left\|\triangle \bw_{i}^{1}\right\|^2 +\dfrac{3}{2\sigma_i^{0} } \left\| \triangle \bpi_{i}^{1}\right\|^2\right]\\
&\overset{(\ref{gap-pi-k})}{\leq} & \dsum  \alpha_i \left[\dfrac{  \varepsilon_i(2\delta )}{2\sigma_i^{0}}   + \dfrac{8\varphi_i^1 }{ \sigma_i^{0}}   -  \dfrac{\sigma_i^{0}+2\lambda}{2}\left\|  \triangle\bw^{1} \right\|^2-  \dfrac{\sigma_i^{0}}{4} \left\|\triangle \bw_{i}^{1}\right\|^2  \right]\\
&\overset{(\ref{upbd-sigma-ell-1})}{\leq} & \dsum  \alpha_i \left[ \frac{\gamma_i^0-\gamma_i^1}{16(1- \gamma_i)}   +   \dfrac{8}{\sigma_i^{0}}  \left\|\rho_i \Q_i^{0} \triangle\overline{\bw}^{0}_i\right\|^2 - \dfrac{8}{\sigma_i^{1}} \left\|\rho_i \Q_i^{1} \triangle\overline{\bw}^{1}_i\right\|^2 \right]\\
&-& \dsum  \alpha_i \left[  \dfrac{\sigma_i^{0}+2\lambda}{4}\left\|  \triangle\bw^{1} \right\|^2+  \dfrac{\sigma_i^{0}}{4} \left\|\triangle \bw_{i}^{1}\right\|^2  \right],
\end{array}
\end{eqnarray}  
sufficing to the following condition, i.e., condition \eqref{descent-L} with $\ell=0$, 
\begin{equation} \label{decreasing-ell-0}
\begin{aligned} 
 &\sum_{i=1}^m \alpha_i \left[   \frac{\sigma_i^{0}+2\lambda}{4}\left\|  \triangle\bw^{1} \right\|^2+  \frac{\sigma_i^{0}}{4} \left\|\triangle \bw_{i}^{1}\right\|^2 \right] \leq \widetilde{\L}^{0}- \widetilde{\L}^{1}.
\end{aligned}
\end{equation} 
The initialization in Algorithm \ref{algorithm-ADMM-mini} implies that $ \triangle \overline{\bw}_{i}^{0} =0$, thereby
\begin{equation*} 
\begin{aligned} 
\widetilde{\L}^0  = {\L}^0 + \sum_{i=1}^m \frac{ \alpha_i\gamma_i^0 }{16(1- \gamma_i)}  \leq  F(\bw^0)+ \frac{1 }{16(1-\gamma)}  \overset{(\ref{range-gamma})}{\leq} F(\bw^0)+ \frac{\gamma  }{1-\gamma } ,
\end{aligned}
\end{equation*} 
which results in
\begin{eqnarray} \label{w0-bound-in-omega} 
\eqspace{2}
\begin{array}{lll} 
 \dsum  \alpha_i \left[  F_{i}(\bw^{1} )  +  \dfrac{\sigma  }{2}  \|\triangle \overline{\bw}^{1} _i\|^2   + \dfrac{\lambda}{2}\|\bw^{1} \|^2 \right]  \overset{\eqref{L-lower-bd}}{\leq} \widetilde{\L}^{1} + 1 \overset{\eqref{decreasing-ell-0}}{\leq}  \widetilde{\L}^{0} + 1\leq  F(\bw^0)+ \frac{1 }{1-\gamma }.
\end{array}
\end{eqnarray}
This above condition implies that $(\bw^1,\W^1)\in\Omega$ by \eqref{compact-set},  namely, 
 \begin{eqnarray*} 
\begin{aligned} 
\bw^{1}\in\N(\delta ),~~\bw_i^{1}\in\N(\delta ), ~\forall~i\in[m],
\end{aligned}
\end{eqnarray*} 
\noindent \textbf{Part 2: $\ell=1$.} Using the above condition and the similar reasoning to show \eqref{decreasing-ell-0} enable us to derive \eqref{descent-L} for $\ell=1$. Then we can show $\bw^{1}\in\N(\delta )$ and $\bw_i^{1}\in\N(\delta )$ for any $i\in[m]$ by using the similar reasoning to prove \eqref{w0-bound-in-omega}. 

\noindent \textbf{Part 3: $\ell\geq2$.} Repeating the above process for all $\ell=2,3,\ldots$ allows us to prove the result.
\end{proof}

\subsection{Proof of Theorem \ref{main-convergence}}\label{app:proof-theorem31}
\begin{proof}  By \eqref{descent-L} and \eqref{L-lower-bd}, sequence $\{\widetilde{\L}^\ell\}$ is non-increasing and lower bounded, so it  converges. Taking the limit of both sides of \eqref{descent-L} yields  $\lim_{\ell\to\infty}\sigma_i^{\ell}\left\| \triangle  {\bw}^{\ell}\right\|^2=\lim_{\ell\to\infty}\sigma_i^{\ell}\left\| \triangle  {\bw}^{\ell}_i\right\|^2=0$, thereby $\lim_{\ell\to\infty} \left\| \triangle  {\bw}^{\ell}\right\|=\lim_{\ell\to\infty} \left\| \triangle  {\bw}^{\ell}_i\right\|=0$, and hence $\lim_{\ell\to\infty}(\widetilde{\L}^\ell-\L^\ell)=0$ by   \eqref{def-tilde-L}. 

2) Hereafter, we define
$$r=\max_{i\in[m]}r_i,~~\rho:=\max_{i\in[m]}\rho_i,~~\eta:= \max_{i\in[m]}\eta_i, ~~\varepsilon  := \max_{i\in[m]}\varepsilon_i(\delta).$$
It follows from Lemma \ref{descent-lemma-L} that $\bw^\ell,\bw_i^\ell\in\N(\delta)$ for any $i\in[m]$ and $\ell\geq0$, thereby
\begin{eqnarray} \label{bd-pi-ell}
\eqspace{1.5}
\begin{array}{lcl}
\| \bpi_{i}^{\ell}\|^2
  &\overset{\eqref{optimality-condition}}{=} &\|  \bg_{i}^\ell+\rho_i \Q_i^{\ell}  \triangle \overline{\bw}_{i}^{\ell}\|^2 \\
&\overset{\eqref{triangle-ineq}}{\leq} & 5\|\bg_i^{\ell}\|^2+ (5/4)\| \rho_i \Q_i^{\ell}   \triangle \overline{\bw}_{i}^{\ell}\|^2 \\ 
&\overset{\eqref{def-eps}}{\leq} &   5\varepsilon /64  + 5(\rho\eta\delta)^2.
\end{array}
\end{eqnarray}
Hence sequence $\{(\bw^\ell,\W^\ell,\P^\ell)\}$ is bounded.  We note from Lemma \ref{descent-lemma-L}  that for any $i$ and $\ell$, it follows ${\bw^\ell, \bw^\ell_i\in\N(\delta)}$, thereby 
\begin{eqnarray} \label{gap-sigma-wi-w-0}
\eqspace{2}
\begin{array}{lcl}
\left\| \sigma_i^{\ell+1}  \triangle \overline{\bw}_{i}^{\ell+1}  \right\|^2 &\leq & 
\left\|\left(\sigma_i^{\ell+1} \I+ \rho_i \Q_i^{\ell+1}\right) \triangle \overline{\bw}_{i}^{\ell+1}  \right\|^2\\
 &\overset{\eqref{optimality-condition}}{=}& \left\| \bpi_{i}^{\ell} +  \bg_{i}^{\ell+1} \right\|^2  \\ 
& \overset{\eqref{optimality-condition}}{\leq} & \dfrac{6}{5}\left\| \bpi_{i}^{\ell}   \right\|^2 + 6\left\| \bg_{i}^{\ell+1} \right\|^2  \\
& \overset{\eqref{bd-pi-ell}}{\leq}  & 6(\rho\eta\delta)^2  + \dfrac{\varepsilon}{4}=: \varpi^2.
\end{array}
\end{eqnarray}
Moreover, by letting $t_\ell:= \sum_{i=1}^m \alpha_i  \sigma_i^{\ell}$,   it follows
\begin{eqnarray*} 
\eqspace{2} 
\begin{array}{lcl} 
 t_\ell  \triangle\bw^{\ell+1}  &\overset{\eqref{sub-w-mini}}{=} &  \dsum \alpha_i   \left[\sigma_i^{\ell}  \triangle \overline{\bw}_{i}^{\ell} +  \bpi_i^{\ell}    -  \lambda\bw^{\ell+1}\right]\\
 &\overset{\eqref{optimality-condition}}{=} &  \dsum \alpha_i   \left[\left(\sigma_i^{\ell} -\rho_i \Q_i^{\ell}\right) \triangle \overline{\bw}_{i}^{\ell} -  \bg_i^{\ell}    -  \lambda\bw^{\ell+1}\right],
\end{array}
\end{eqnarray*}
which together with $\sum_{i=1}^m \alpha_i=1$   yields
\begin{eqnarray*} 
\eqspace{2}
\begin{array}{lcl}
 t_\ell ^2\left\| \triangle\bw^{\ell+1}\right\|^2  &\leq &  \dsum \alpha_i  \left[  3\left\|  \sigma_i^{\ell}\triangle \overline{\bw}_{i}^{\ell}\right\|^2 + 3  \|\bg_i^\ell\|^2 +  3\lambda^2 \|\bw^{\ell+1}\|^2\right]\\
 &\overset{\eqref{gap-sigma-wi-w-0}}{\leq} &   \dsum \alpha_i \Big[ \Big(18(\rho\eta)^2 +  3\lambda^2\Big)\delta^2  +   \varepsilon    \Big]=:c^2.
\end{array}
\end{eqnarray*}
Since $\gamma=\max_{i\in[m]}\gamma_i$ and  $\sigma_i^0>\sigma$ for any $i\in[m]$ from \eqref{choice-of-sigma}, we obtain
\begin{eqnarray}  \label{gap-w-gamma-ell}
\eqspace{2}
\begin{array}{lcl}
\left\|  \triangle {\bw}^{\ell+1}  \right\|  \leq \dfrac{c }{t_\ell }  =  c \left(\dsum  \dfrac{\alpha_i \sigma_i^0}{\gamma_i^\ell}  \right)^{-1}\leq  \dfrac{c\gamma^\ell}{\sigma}. 
\end{array}
\end{eqnarray}
 For every given $\epsilon>0$, by letting $K\geq \log_\gamma((1-\gamma)\sigma\epsilon/c)$, for any $s>k>K$, there is
\begin{eqnarray*}  
\eqspace{2}
\begin{array}{lcl}
 \left\| \bw^{s}-\bw^{k}\right\|  \leq  \displaystyle \sum_{\ell=k}^s \left\| \triangle\bw^{\ell+1}\right\| \leq   \sum_{\ell=k}^s  \dfrac{c\gamma^\ell}{\sigma}
 \leq  \dfrac{c  \gamma^{k} }{\sigma(1-\gamma)} \leq \epsilon,
\end{array}
\end{eqnarray*}
which means  $\{\bw^{\ell}\}$ is a Cauchy sequence. Therefore, sequence $\{\bw^{\ell}\}$ converges. By \eqref{gap-sigma-wi-w-0}, we can prove  ${\lim_{\ell\to\infty}  (\bw^{\ell}_i -\bw^{\ell}) =0}$ and thus sequence $\{\bw^{\ell}_i\}$ converges for each $i\in[m]$. Let $\bw^\infty$ be the limit of  $\{\bw^{\ell}\}$. 
We note that $\bw^{\ell}$ is independent to $\B_i^{\ell}$ for each $i\in[m]$ and thus  
\begin{eqnarray}\label{E-g-F-w-ell}
 \E    \bg_{i}^{\ell} = \E \nabla F_i(\bw^{\ell}; \B_i^{\ell} )= \nabla F_i(\bw^{\ell} ).
\end{eqnarray}
 Recalling \eqref{optimality-condition}, it follows
\begin{eqnarray*} 
0&=&  \lim_{\ell\to \infty}  \E \left[ \bpi_{i}^{\ell} +     \bg_{i}^{\ell}  + \rho_i \Q_i^{\ell} \triangle \overline{\bw}_{i}^{\ell}\right] \\
&= &\lim_{\ell\to \infty} \left[\E\bpi_{i}^{\ell}   +  \nabla F_i(\bw^\ell)\right]\\
&= &\lim_{\ell\to \infty} \left[\E \bpi_{i}^{\ell}   + \nabla F_i(\bw^\infty)\right].   
\end{eqnarray*} 
Therefore, sequence $\{\E \bpi_{i}^{\ell} \}$ converges. 
\end{proof}

\subsection{Proof of Theorem \ref{main-convergence-rate-eps}}\label{app:proof-theorem32}
\begin{proof}
 It follows from \eqref{gap-w-gamma-ell} that
\begin{eqnarray*}  
\left\|  {\bw}^{\ell} - \bw^\infty  \right\| \leq \sum_{k=\ell}^\infty \left\|  \triangle {\bw}^{k+1}  \right\|   \leq  \sum_{k=\ell}^\infty  \dfrac{c\gamma^k}{\sigma} = \dfrac{c\gamma^\ell}{\sigma(1-\gamma)}. 
\end{eqnarray*}
This condition and $\sigma_i^\ell=\sigma_i^0/\gamma_i^\ell$ further lead  to
\begin{eqnarray*}   
\eqspace{2}
\begin{array}{lcl}
\left\|  {\bw}^{\ell}_i - \bw^\infty  \right\| &\leq&  \left\| \triangle \overline{\bw}^{\ell}_i  \right\| + \left\|  {\bw}^{\ell} - \bw^\infty  \right\|  \\
&\overset{\eqref{gap-sigma-wi-w-0}}{\leq}&     \dfrac{\varpi\gamma_i^{\ell}}{\sigma_i^0} +\dfrac{c\gamma^\ell}{\sigma(1-\gamma)}\\
& \leq& \dfrac{\varpi\gamma^{\ell}}{\sigma} +\dfrac{c\gamma^\ell}{\sigma(1-\gamma)}.
\end{array} 
\end{eqnarray*}
In addition,   one can verify that
\begin{eqnarray*} 
\eqspace{2}
\begin{array}{lcl}
\left\|\E \bpi_{i}^{\ell} +    \nabla F_i(\bw^{\ell} )  \right\|&\overset{\eqref{E-g-F-w-ell}}{=}&   \left\|\E\left( \bpi_{i}^{\ell} +     \bg_{i}^{\ell} \right) \right\|  \overset{\eqref{optimality-condition}}{\leq}  \left\| \E  \left(\rho_i \Q_i^{\ell} \triangle \overline{\bw}_{i}^{\ell}  \right)\right\| \\
&\overset{\eqref{Q-upper-bd}}{\leq}& \rho_i \eta_i  \E \left\|  \triangle \overline{\bw}_{i}^{\ell}  \right\|\overset{\eqref{gap-sigma-wi-w-0}}{\leq} \dfrac{\varpi\rho_i \eta_i \gamma^{\ell}_i}{\sigma_i^0}\overset{\eqref{increase-sigma}}{\leq} \dfrac{\varpi \gamma^{\ell}}{8}.
\end{array} 
\end{eqnarray*}
Using this condition, it follows
\begin{eqnarray*} 
\eqspace{2}
\begin{array}{lcl}
\left\|\E  \bpi_{i}^{\ell} +    \nabla F_i(\bw^\infty)  \right\| &\leq& \left\|\E  \bpi_{i}^{\ell} +    \nabla F_i(\bw^{\ell} )  \right\|+
\left\| \nabla F_i(\bw^{\ell} )- \nabla F_i(\bw^\infty)       \right\|\\
&\overset{\eqref{lip-fi-ri}}{\leq}&  \dfrac{\varpi \gamma^{\ell}}{8} + r \left\| \bw^{\ell}  - \bw^\infty \right\|\leq \dfrac{\varpi \gamma^{\ell}}{8} + \dfrac{c r\gamma^\ell}{\sigma(1-\gamma)}.
\end{array}
\end{eqnarray*}
The above facts enable us to show result \eqref{rate-w} by letting
  $$C_1:=\max\left\{\dfrac{\varpi}{\sigma} +\dfrac{c}{\sigma(1-\gamma)},~\dfrac{\varpi}{8} + \dfrac{c r}{\sigma(1-\gamma)}  \right\}.$$
Now we prove \eqref{rate-w}. It follows from $\bw^\ell\in\N(\delta)$ that 
\begin{eqnarray} \label{gap-Fw-Fw*}
\eqspace{1.75}
\begin{array}{lcl}
&&F_{\lambda}(\bw^\ell) -F_{\lambda}(\bw^\infty) \\&=& 
\dsum \alpha_i \left[F_{i}(\bw^\ell)+\dfrac{\lambda}{2} \left\|  \bw^\ell\right\|^2  - F_{i}(\bw^\infty)-\dfrac{\lambda}{2} \left\|  \bw \right\|^2 \right]  \\
&\overset{\eqref{Lipschitz-continuity}}{\leq}&  \dsum \alpha_i \left[\left\langle \nabla F_{i}(\bw^\ell) + \lambda  \bw^\ell, \bw^\ell-\bw^\infty\right\rangle  +\dfrac{ r_i+\lambda}{2} \left\|  \bw^\ell-\bw^\infty\right\|^2 \right]  \\ 
 &\leq&   \dsum  \alpha_i \left[  \left(\left\|\nabla F_{i}(\bw^\ell)\right\| + \lambda \left\| \bw^\ell \right\|\right) \left\| \bw^\ell-\bw^\infty\right\| +\dfrac{ r+\lambda}{2} \left\|  \bw^\ell-\bw^\infty\right\|^2 \right] \\ 
&\overset{\eqref{def-eps}}{\leq}&  \dsum  \alpha_i \left[  \dfrac{\sqrt{\varepsilon(\delta)}+8\lambda\delta}{8} \left\| \bw^\ell-\bw^\infty\right\| +\dfrac{ r+\lambda}{2} \left\|  \bw^\ell-\bw^\infty\right\|^2\right]\\ 
  &\leq&  C_2 \gamma^\ell,  
\end{array}\end{eqnarray} 
for any $i\in[m]$ and $ \ell\geq 0$, where $C_2$ is a constant satisfying
$$C_2\geq \dfrac{ \sqrt{\varepsilon(\delta)}C_1+8\lambda\delta+4 (r +\lambda )C_1^2   }{8}.
$$
From  \eqref{gap-sigma-wi-w-0}, one can verify 
\begin{eqnarray*} 
\left\|  \triangle \overline{\bw}_{i}^{\ell+1}  \right\|^2  \leq  \dfrac{\varpi^2}{(\sigma_i^{\ell+1})^2} \leq \dfrac{\varpi^2\gamma^{\ell+1}}{\sigma^2}.
\end{eqnarray*}
Then using the above fact, the boundedness of $\{\P^\ell\}$, and \eqref{gap-Fw-Fw*} can also show that $\L^{\ell}  - F_\lambda( {\bw}^{\infty}) $ and $ \widetilde{\L}^{\ell}  - F_\lambda( {\bw}^{\infty})$ is bounded by $C_2\gamma^\ell$. The proof is finished.
\end{proof}

\subsection{Proof of Theorem \ref{main-convergence-rate-gradient}}\label{app:proof-theorem33}
	\begin{proof}
		It follows from \eqref{optimality-condition} that  
		\[
		\bg_i^{\ell+1}
		=
		-\bpi_i^\ell
		-
		\left(
		\sigma_i^{\ell+1}\I+\rho_i\Q_i^{\ell+1}
		\right)
		\Delta\overline{\bw}_i^{\ell+1},
		\]
and from \eqref{optimality-condition-w} that 
		\[
		0=
		\sum_{i=1}^m\alpha_i
		\left[
		-\bpi_i^\ell
		-
		\sigma_i^\ell(\bw_i^\ell-\bw^{\ell+1})
		\right]
		+
		\lambda\bw^{\ell+1}.
		\]
		Therefore,
		\[
		-\sum_{i=1}^m\alpha_i\bpi_i^\ell
		+
		\lambda\bw^{\ell+1}
		=
		\sum_{i=1}^m\alpha_i\sigma_i^\ell
		(\bw_i^\ell-\bw^{\ell+1}).
		\]
		Combining the above two relations, we obtain
		\[
		\begin{aligned}
			\sum_{i=1}^m\alpha_i \bg_i^{\ell+1}
			+
			\lambda\bw^{\ell+1}
			&=
			\sum_{i=1}^m\alpha_i\sigma_i^\ell
			(\bw_i^\ell-\bw^{\ell+1})
			-
			\sum_{i=1}^m\alpha_i
			\left(
			\sigma_i^{\ell+1}I+\rho_iQ_i^{\ell+1}
			\right)
			\Delta\overline{\bw}_i^{\ell+1}\\
			&=
			-t_\ell\Delta\bw^{\ell+1}
			+
			\sum_{i=1}^m\alpha_i\sigma_i^\ell
			\Delta\overline{\bw}_i^\ell
			-
			\sum_{i=1}^m\alpha_i
			\left(
			\sigma_i^{\ell+1}\I+\rho_i\Q_i^{\ell+1}
			\right)
			\Delta\overline{\bw}_i^{\ell+1},
		\end{aligned}
		\]
		where
		\begin{equation}\label{def-t-ell}
		t_\ell:=\sum_{i=1}^m\alpha_i\sigma_i^\ell=\sum_{i=1}^m \frac{\alpha_i\sigma_i^0}{\gamma_i^\ell} =\sum_{i=1}^m\frac{\alpha_i\sigma_i^0}{\gamma^\ell}.
	\end{equation}
		Taking conditional expectation with respect to the history before drawing
		\(\mathcal B_i^{\ell+1}\), and using the unbiasedness of the stochastic
		gradient, we have
		\[
		\nabla F_\lambda(\bw^{\ell+1})
		=
		\mathbb E\left[
		\sum_{i=1}^m\alpha_i \bg_i^{\ell+1}
		+
		\lambda\bw^{\ell+1}
		\,\middle|\,\mathcal F_\ell
		\right].
		\]
		Hence, by Jensen's inequality and the preceding decomposition,
		\[
		\begin{aligned}
			\frac1{t_\ell}
			\mathbb E\|\nabla F_\lambda(\bw^{\ell+1})\|^2
			& \le 
			3t_\ell
			\mathbb E\|\Delta\bw^{\ell+1}\|^2                                      +
			\frac{3}{t_\ell}
			\mathbb E
			\left\|
			\sum_{i=1}^m
			\alpha_i\sigma_i^\ell
			\Delta\overline{\bw}_i^\ell
			\right\|^2 \\
			&   +
			\frac{3}{t_\ell}
			\mathbb E
			\left\|
			\sum_{i=1}^m
			\alpha_i
			\left(
			\sigma_i^{\ell+1}\I+\rho_i\Q_i^{\ell+1}
			\right)
			\Delta\overline{\bw}_i^{\ell+1}
			\right\|^2 .
		\end{aligned}
		\]		
		We next estimate the three terms on the right-hand side separately. Again by \eqref{descent-L}, 
		\begin{equation}
			\label{eq:first-gradient-term-bound-with-3}
			3t_\ell
			\mathbb E\|\Delta\bw^{\ell+1}\|^2
			\le
			12\mathbb E
			\left(
			  \widetilde{\L}^{\ell}- \widetilde{\L}^{\ell+1}
			\right). 
		\end{equation}		
		For the second term, by the weighted Jensen inequality, we have
		\begin{align*} 
			\frac{1}{t_\ell} \left\|
			\sum_{i=1}^m
			\alpha_i\sigma_i^\ell
			\Delta\overline{\bw}_i^\ell
			\right\|^2
			&=
			t_\ell 
			\left\|
			\sum_{i=1}^m
			\frac{\alpha_i\sigma_i^\ell}{t_\ell}
			\Delta\overline{\bw}_i^\ell
			\right\|^2                                           \\
			&\le
			t_\ell 
			\sum_{i=1}^m
			\frac{\alpha_i\sigma_i^\ell}{t_\ell}
			\|\Delta\overline{\bw}_i^\ell\|^2                         \\
			&=
			\sum_{i=1}^m
			\alpha_i\sigma_i^\ell
			\|\Delta\overline{\bw}_i^\ell\|^2 .
		\end{align*}
		Moreover, since $\gamma_i\in[3/4,1),$  
		we have
		\[
		\sigma_i^\ell= \frac{\sigma_i^{\ell-1}}{\gamma_i}
		\le
		\frac{4\sigma_i^{\ell-1}}{3},
		\]
which by \eqref{descent-L}  implies	
		\begin{align}\label{bound-E-delta-wi}
			\sum_{i=1}^m
			\alpha_i\sigma_i^\ell
			\mathbb E\|\Delta\overline{\bw}_i^\ell\|^2
			 \le
			\frac{4}{3}
			\sum_{i=1}^m
			\alpha_i\sigma_i^{\ell-1}
			\mathbb E\|\Delta\overline{\bw}_i^\ell\|^2                  \le
			\frac{16}{3}
			\mathbb E
			\left(
			\widetilde \L^{\ell-1}-\widetilde \L^\ell
			\right). 
		\end{align}		
 Therefore,
		\begin{equation}
			\label{eq:second-gradient-term-bound}
			\frac{3}{t_\ell}
			\mathbb E
			\left\|
			\sum_{i=1}^m
			\alpha_i\sigma_i^\ell
			\Delta\overline{\bw}_i^\ell
			\right\|^2
			\le 16
			\mathbb E
			\left(
			\widetilde \L^{\ell-1}-\widetilde \L^\ell
			\right).
		\end{equation}		
		For the third term, using \(\|\Q_i^{\ell+1}\|\le \eta_i\) and
		\(\rho_i\eta_i\le \sigma_i^\ell/8\), we have
		\[
		\begin{aligned}
			\left\|
			\left(
			\sigma_i^{\ell+1}\I+\rho_i\Q_i^{\ell+1}
			\right)\bx
			\right\|
			&\le
			\left(
			\sigma_i^{\ell+1}+\rho_i\eta_i
			\right)\|\bx\|                                      \\
			&=
			\left(
			\frac{\sigma_i^\ell}{\gamma_i}+\rho_i\eta_i
			\right)\|\bx\|                                      \\
			&\le
			\left(
			\frac{4}{3}+\frac18
			\right)
			\sigma_i^\ell\|\bx\|                                \\
			&\le
			 \frac{3 \sigma_i^\ell}{2}\|x\|.
		\end{aligned}
		\]
		Therefore, again by the weighted Jensen inequality,
		\begin{align}
			\frac{3}{t_\ell}\left\|
			\sum_{i=1}^m
			\alpha_i
			\left(
			\sigma_i^{\ell+1}\I+\rho_i\Q_i^{\ell+1}
			\right)
			\Delta\overline{\bw}_i^{\ell+1}
			\right\|^2                                             & =
			3t_\ell
			\left\|
			\sum_{i=1}^m
			\frac{\alpha_i\sigma_i^\ell}{t_\ell}
			\frac{
				\left(
				\sigma_i^{\ell+1}\I+\rho_i\Q_i^{\ell+1}
				\right)
			}{
				\sigma_i^\ell
			}
			\Delta\overline{\bw}_i^{\ell+1}
			\right\|^2                                                     \nonumber\\
			&\le
			3t_\ell
			\sum_{i=1}^m
			\frac{\alpha_i\sigma_i^\ell}{t_\ell}
			\left\|
			\frac{
				\left(
				\sigma_i^{\ell+1}\I+\rho_i\Q_i^{\ell+1}
				\right)
			}{
				\sigma_i^\ell
			}
			\Delta\overline{\bw}_i^{\ell+1}
			\right\|^2                                                     \nonumber\\
			&\le
			  \frac{27}{4}
			\sum_{i=1}^m
			\alpha_i\sigma_i^\ell
			\|\Delta\overline{\bw}_i^{\ell+1}\|^2\nonumber\\			
			&\le 36 
			\mathbb E
			\left(
			\widetilde \L^{\ell}-\widetilde \L^{\ell+1}
			\right), \label{eq:third-gradient-term-bound}
		\end{align}
		where the last inequity is from \eqref{bound-E-delta-wi}.
		Combining \eqref{eq:first-gradient-term-bound-with-3},
		\eqref{eq:second-gradient-term-bound}, and
		\eqref{eq:third-gradient-term-bound}, we obtain
		\begin{equation}
			\label{eq:gradient-stability-one-step}
		\frac1{t_\ell}
			\mathbb E\|\nabla F_\lambda(\bw^{\ell+1})\|^2
			\le 16 \mathbb E
			\left(
			\widetilde \L^{\ell-1}-\widetilde \L^{\ell}
			\right) + 48 \mathbb E
			\left(
			\widetilde \L^{\ell}-\widetilde \L^{\ell+1}
			\right).
		\end{equation}	
		Summing \eqref{eq:gradient-stability-one-step} over
		\(\ell=1,\ldots,T\), we obtain
		\[
		\begin{aligned}
			\sum_{\ell=1}^{T}
			\frac1{t_\ell}
			\mathbb E\|\nabla F_\lambda(\bw^{\ell+1})\|^2
			&\le
			16
			\sum_{\ell=1}^{T}
			\mathbb E
			\left(
			\widetilde \L^{\ell-1}-\widetilde \L^\ell
			\right)
			+
			48
			\sum_{\ell=1}^{T}
			\mathbb E
			\left(
			\widetilde \L^\ell-\widetilde \L^{\ell+1}
			\right)\\
			& = 16 \mathbb E
			\left(
			\widetilde \L^{0}-\widetilde \L^T
			\right)
			+
			48 
			\mathbb E
			\left(
			\widetilde \L^1-\widetilde \L^{T+1}
			\right)\\[1ex]
			&\leq  64  
			\left(
			\widetilde \L^{0}- F^*+1
			\right) =: C_4, 
		\end{aligned}
		\]
		where the last inequality is from \eqref{L-lower-bd} and $\widetilde \L^1\leq \widetilde \L^0.$ Let $C:=C_4\sum_{i=1}^m\alpha_i\sigma_i^0 $ and define
		\[
		S_T:=\sum_{\ell=1}^{T}\frac{1}{t_\ell} = \sum_{\ell=1}^{T}\frac{\gamma^\ell}{\sum_{i=1}^m\alpha_i\sigma_i^0} = \frac{C_4\gamma(1-\gamma^T)}{C(1-\gamma)}.
		\]
		where the second equation is from \eqref{def-t-ell}. 
		Let \(R_T\) be sampled from \(\{2,\ldots,T+1\}\) according to
		\[
		\mathbb P(R_T=\ell+1)
		=
		\frac{1/t_\ell}{S_T},
		\qquad
		\ell=1,\ldots,T.
		\]
		Then it follows from  \cite{Ghadimi13} and  \eqref{condition-gamma-K} that
		\begin{equation}
		\lim_{T\to\infty}	\mathbb E\|\nabla F_\lambda(\bw^{R_T})\|^2
			=
			\lim_{T\to\infty}	 \frac{
				\sum_{\ell=1}^{T}
				\frac{1}{t_\ell}
				\mathbb E\|\nabla F_\lambda(\bw^{\ell+1})\|^2
			}{
				S_T
			}
			\le
			\lim_{T\to\infty}	 \frac{C(1-\gamma)}{\gamma(1-\gamma^T)}
			 = 0.
		\end{equation}
		Therefore, there exists a subsequence
		\(\{\bw^{\ell_j}\}\) with \(\ell_j\to\infty\) such that
		\begin{equation}
			\mathbb E\|\nabla F_\lambda(\bw^{\ell_j})\|^2\to0.
		\end{equation}
		
		By the convergence of PISA, \(\bw^\ell\to\bw^\infty\). Since
		\(\nabla F_\lambda\) is Lipschitz continuous on the bounded region containing
		the generated sequence, there exists \(L_\lambda>0\) such that
		\[
		\|\nabla F_\lambda(\bu)-\nabla F_\lambda(\bv)\|
		\le
		L_\lambda\|\bu-\bv\|.
		\]
		Thus,
		\[
		\begin{aligned}
			\mathbb E\|\nabla F_\lambda(\bw^\infty)\|^2
			&\le
			2\mathbb E\|\nabla F_\lambda(\bw^{\ell_j})\|^2
			+
			2\mathbb E
			\|\nabla F_\lambda(\bw^{\ell_j})
			-
			\nabla F_\lambda(\bw^\infty)\|^2     \\
			&\le
			2\mathbb E\|\nabla F_\lambda(\bw^{\ell_j})\|^2
			+
			2L_\lambda^2
			\mathbb E\|\bw^{\ell_j}-\bw^\infty\|^2 .
		\end{aligned}
		\]
		Letting \(j\to\infty\), we obtain
		\begin{equation}
			\mathbb E\|\nabla F_\lambda(\bw^\infty)\|^2=0.
		\end{equation}
		
		Finally, by Lipschitz continuity again and  \eqref{rate-w}, 
		\[
		\begin{aligned}
			\mathbb E\|\nabla F_\lambda(\bw^T)\|^2
			 =
			\mathbb E
			\|\nabla F_\lambda(\bw^T)-\nabla F_\lambda(\bw^\infty)\|^2  \le
			L_\lambda^2
			\mathbb E\|\bw^T-\bw^\infty\|^2 = O(\gamma^T).
		\end{aligned}
		\]
		This completes the proof.
	\end{proof}

\end{document}